\documentclass[final]{article}

\usepackage[nonatbib,final]{neurips_2021}

\usepackage[utf8]{inputenc} % allow utf-8 input
\usepackage[T1]{fontenc}    % use 8-bit T1 fonts
\usepackage{hyperref}       % hyperlinks
\usepackage{url}            % simple URL typesetting
\usepackage{booktabs}       % professional-quality tables
\usepackage{amsfonts}       % blackboard math symbols
\usepackage{nicefrac}       % compact symbols for 1/2, etc.
\usepackage{microtype}      % microtypography
\usepackage{xcolor}         % colors

\usepackage{preamble}
\title{Lossy Compression for Lossless Prediction %\ using Invariances
}

\author{%
  Yann Dubois\\
  Vector Institute \\
  \texttt{yanndubois@cs.toronto.edu} \\
  \And  
  Benjamin Bloem-Reddy\\
  The University of British Columbia \\
  \texttt{benbr@stat.ubc.ca} \\
  \AND
  Karen Ullrich \\
  Facebook AI Research \\
  \texttt{karenu@fb.com} \\
  \And
  Chris J.~Maddison \\
  University of Toronto, Vector Institute \\
  \texttt{cmaddis@cs.toronto.edu} \\
}

\begin{document}
\doparttoc % Tell to minitoc to generate a toc for the parts
\faketableofcontents % Run a fake tableofcontents command for the partocs
%\part{} % Start the document part
% %\parttoc % Insert the document TOC

\maketitle

\begin{abstract}
Most data is automatically collected and only ever ``seen'' by algorithms.
Yet, data compressors preserve perceptual fidelity rather than just the information needed by algorithms performing downstream tasks.
In this paper, we characterize the bit-rate required to ensure high performance on all predictive tasks that are invariant under a set of transformations, such as data augmentations.
Based on our theory, we design unsupervised objectives for training neural compressors.
Using these objectives, we train a generic image compressor that achieves substantial rate savings (more than $1000\times$ on ImageNet) compared to JPEG on 8 datasets, without decreasing downstream classification performance.
\end{abstract}

\section{Introduction}
\label{sec:introduction}
%\input{figures/banana/intro}

% \ydnote{Is the flow always understandable ? I feel like there's sometimes a jump in the reasoning (e.g. why this is an interesting setting) because this is obvious for us, but but might not for others  }

%\input{figures/intro_vertical}

Progress in important areas requires processing huge amounts of data. 
For climate prediction, models are still data-limited \cite{rolnick_tackling_2019}, despite the Natl. Center for Computational Sciences storing 32 million gigabytes (GB) of climate data \cite{zgurovsky_big_2020}. For autonomous driving, capturing a realistic range of rare events with current methods requires around 3 trillion GB of data.%
\footnote{\citet{kalra_driving_2016} estimated that autonomous vehicles would have to drive hundreds of billions of miles to demonstrate reliability in rare events. At 1 TB / hour \citep{yeong_sensor_2021} and 30 miles / hour, this is $3\scip{12}$ GB.}
%Processing and storing this amount of data poses an urgent challenge.
At these scales, data are only processed by task-specific algorithms, and storing data in human-readable formats can be prohibitive.
We need compressors that retain only the information needed for algorithmic execution of downstream tasks.
%Scaling up current data storage and processing pipelines to such levels is simply not sustainable,
%Compression methods that retain only information necessary for \textit{algorithmic execution of downstream tasks} are potentially transformative in their reduced storage.
% So, we need compressors that retain only the information necessary for performing downstream tasks.

\begin{wrapfigure}{r}{0.45\textwidth}
\vspace{-\baselineskip}
\centering
\captionsetup[subfigure]{labelformat=empty}

\begin{minipage}{0.45\textwidth}
\centering
\begin{subfigure}[h]{0.32\columnwidth}
 \centering
 \includegraphics[width=\textwidth]{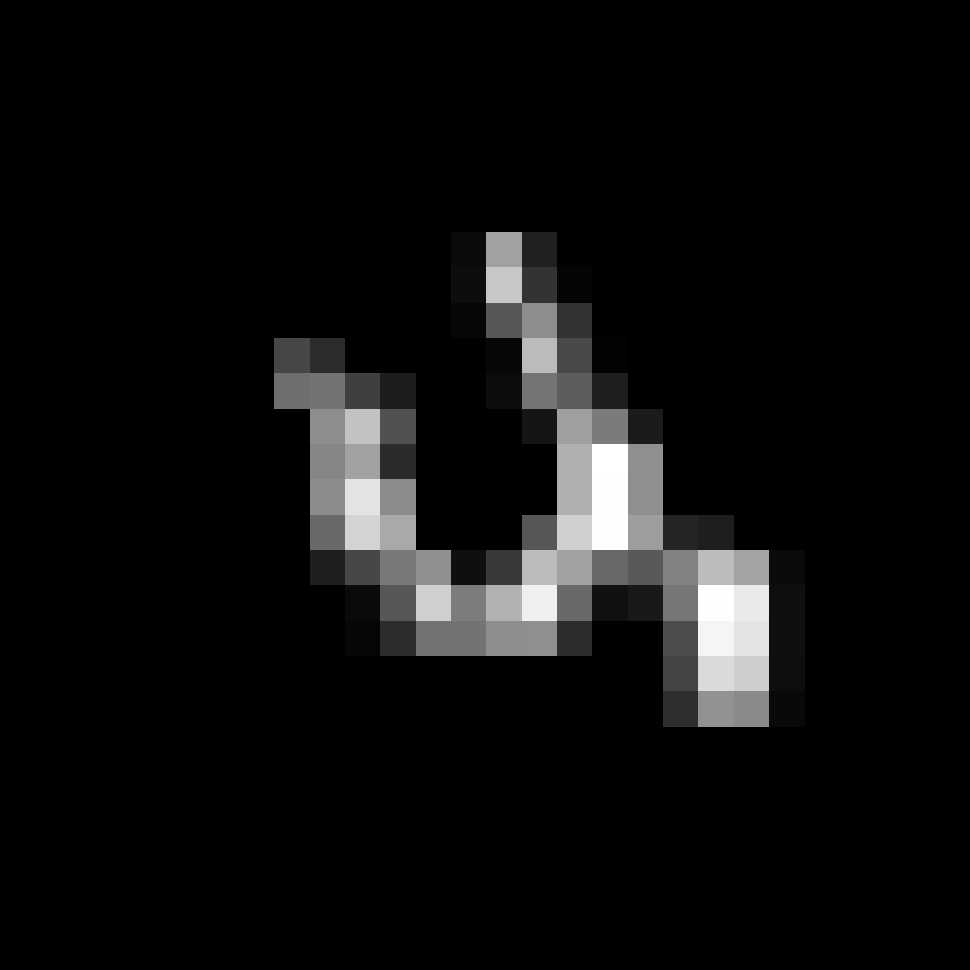}
 %\vspace*{-1em}
 \caption{Source}
\end{subfigure}
\begin{subfigure}[h]{0.32\columnwidth}
 \centering
 \includegraphics[width=\textwidth]{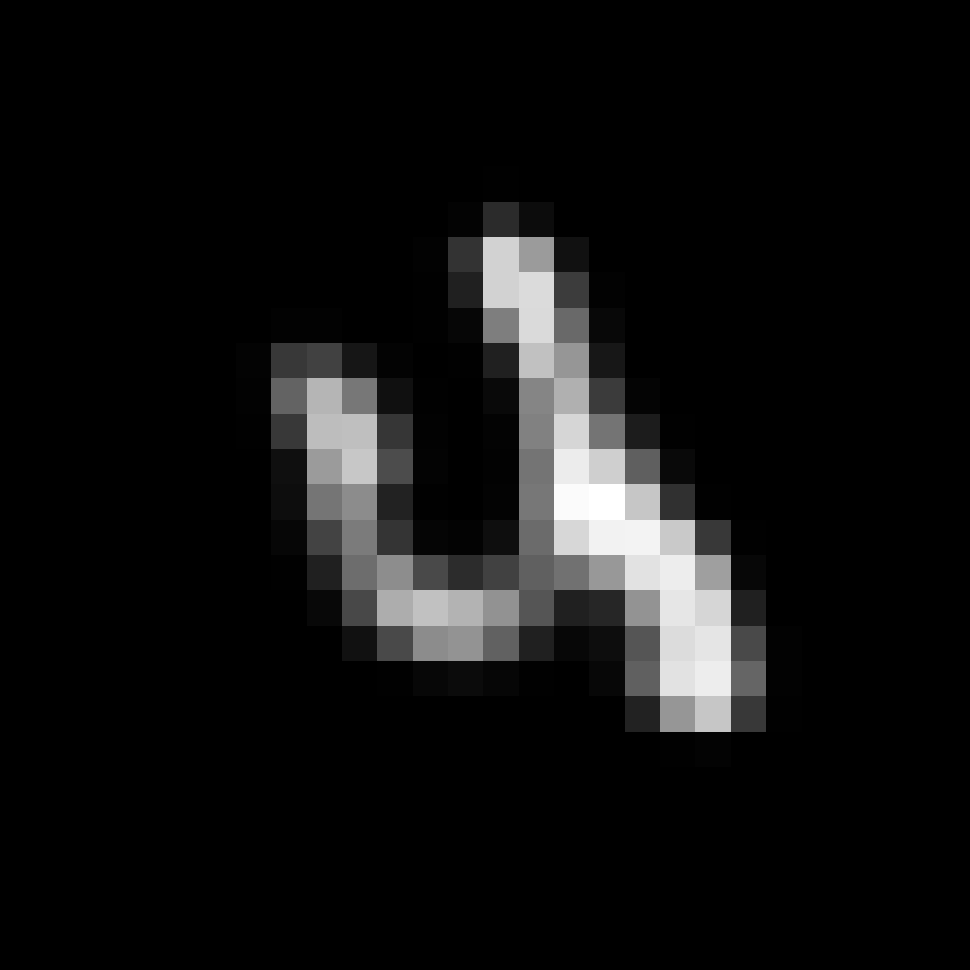}
 %\vspace*{-1em}
 \caption{Standard rec.}
\end{subfigure}
\begin{subfigure}[h]{0.32\columnwidth}
 \centering
 \includegraphics[width=\textwidth]{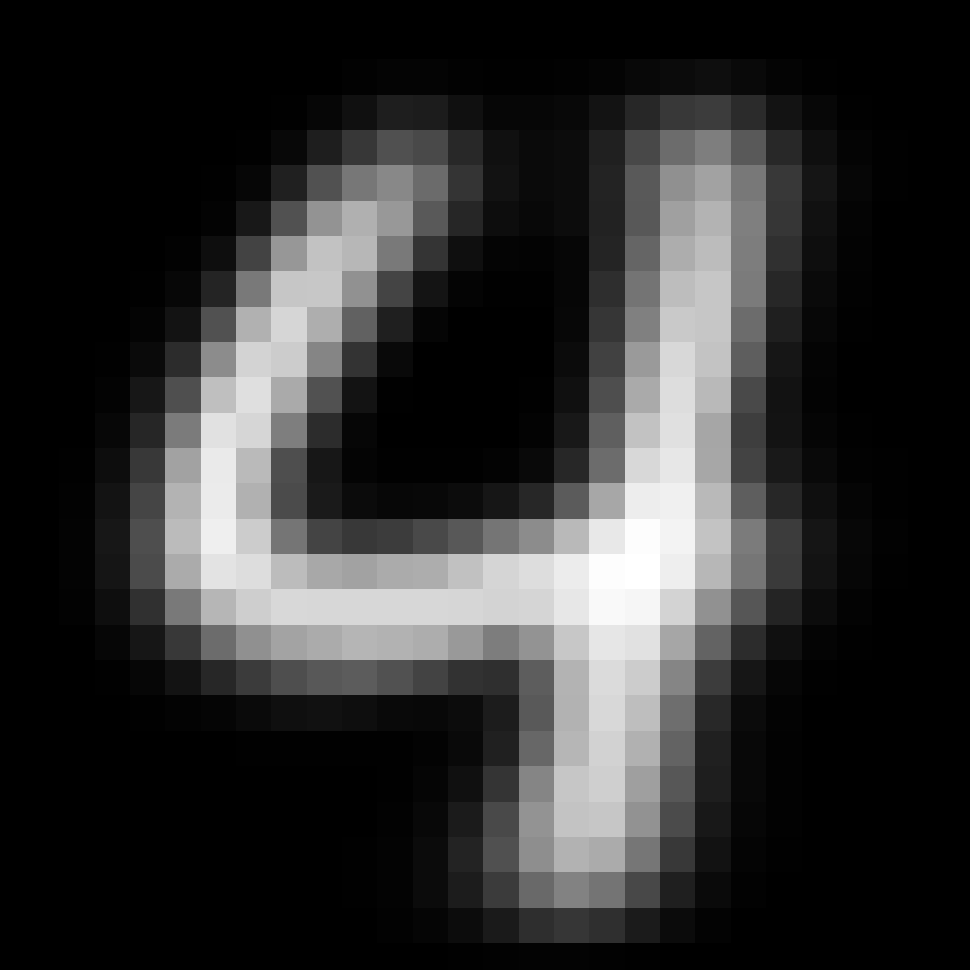}
 %\vspace*{-1em}
 \caption{Our rec.}
\end{subfigure}
 \vspace*{-0.5em}
\addtocounter{figure}{-1} % for some reason adding 1 too much
 \captionof{figure}{
Our unsupervised coder improves compression by only keeping information necessary for typical tasks.
(left) source augmented MNIST digit; (center) a neural perceptual compressor achieves 130 bit-rate; (right)
our invariant compressor achieves 48 bit-rate.
 \label{fig:mnist_intro}
}
\end{minipage}%
\vspace{-2\baselineskip}

\end{wrapfigure}

Existing lossy compressors are not up to the challenge, because they aim to reconstruct the data for human perception \cite{johnston_transform_1988,heeger_model_1995,lee_perceptual_2012,blau_rethinking_2019,pan_digital_1993,mentzer_high-fidelity_2020}. 
However, much of perceptual information is not needed to perform the tasks that we care about.
Consider classifying images, which can require about 1 MB to store. 
Classification is typically invariant under small image transformations, such as rescalings or rotations, and could instead be performed using a representation that discards such information (see \cref{fig:mnist_intro}). 
The amount of unnecessary perceptual information is likely substantial, as illustrated by the fact that typical image classification can be performed using a detailed caption, which requires only about 1 kB to store ($1000\times$ fewer bits).
% \ydnote{I'm thinking of maybe droppign "invariance" here or expanding it more. I'm not sure it's very clear.}
% \cmnote{I think if we remove it, then invariance shows up out of nowhere later on. So, I vote expand it.}
% \ydnote{ok we should expand on language+humans and invariances}
% This suggests that standard compressors are storing $1000\times$ more information than is needed to perform typical image classification.

% \ydnote{I think that this paragraph is a little light as to why this is really the setting that we should care about / the most realistic }
% The minimum bit-rate required for high performance on a supervised task effectively corresponds to compressing the labels.
% Achieving this rate requires access to the labels, and caring only about a single task.
% Instead, we are interested in compression for a collection of tasks that may not fully determined at compression time.

Our goal is to quantify the bit-rate needed to ensure high performance on a collection of prediction tasks.
In the simple case of a \emph{single} supervised task, the minimum bit-rate is achieved by compressing predicted labels, and essentially corresponds to the Information Bottleneck (IB; \cite{tishby_information_2000}).
Our challenge, instead, is to ensure good performance on \emph{any} future tasks of interest, which will rarely be completely known at compression time, or might be too large to enumerate.

We overcome this challenge by focusing on sets of tasks that are \textit{invariant} under user-defined transformations (\eg, translation, brightness, cropping), as is the case for many tasks of interest to humans \cite{heaton_ian_2018,shorten_survey_2019}.
This structure allows us to characterize a worst-case invariant task, which bounds the relative predictive performance on all invariant tasks.
As a result, the bit-rate required to perform well on \textit{all} invariant tasks is exactly the rate to compress the worst-case labels.
At a high level, the worst-case task is to recognize which examples are transformed versions of one another, and rate savings come from discarding information from those transformations. 

We also provide two unsupervised neural compressors to target the optimal rates.
One is similar to a variational autoencoder \cite{kingma_auto-encoding_2014} that reconstructs canonical examples (\cref{fig:mnist_intro}).
%Because many digits are mapped to the same prototype, we improve compression rates from 130 bits to 48 bits.
Our second is a simple modification of contrastive self-supervised learning (SSL; \cite{oord_representation_2019}), which allows us to convert pre-trained SSL models into powerful, generic compressors.
Our contributions are:
%can be summarized as:
\begin{itemize}[noitemsep,leftmargin=*]
\item We formalize the notion of compression for downstream predictive tasks.
\item We characterize the bits needed for high performance on any task invariant to augmentations.
\item We provide unsupervised objectives to train compressors that approximate the optimal rates.
\item We show that our compressor outperforms JPEG by orders of magnitude on 8 datasets on which it was never trained (\ie, zero-shot).
E.g., on ImageNet \cite{deng_imagenet_2009}, it decreases the bit-rate by $1000\times$.
\end{itemize}

\cmnote{we need to hammer home the idea that all human goals are specified in language -> every thing we could possibly be itnerested in predicting can likely be compiled to a detailed caption -> massive compression gains}
\bbnote{For classification this might be true, but not for more sophisticated tasks. Climate prediction?}

\section{Rate-distortion theory background}
\label{sec:background}

The goal of lossy compression theory is to find the number of bits (\emph{bit-rate}) required to store outcomes $x$ of a random variable (r.v.) $\rv X$, so that it can be reconstructed within a certain tolerance.
This is accomplished in \citepos{shannon_coding_1959} rate-distortion (RD) theory by mapping $\rv X$ into a r.v.\ $\rv Z$ with low mutual information $\op{I}{\rv X; \rv Z}$.
Specifically, given a distortion measure $\op{D}{\rv X, \rv Z}$, the RD theory characterizes the minimal achievable bit-rate for a distortion threshold $\delta$ by
\begin{equation}\label{eq:rate_distortion}
Rate(\delta) = \min_{p(\rv Z| \rv X)}   \op{I}{\rv X; \rv Z} \quad \text{such that} \quad  \op{D}{\rv X, \rv Z} \leq \delta \,.
\end{equation}
%
%where $\op{I}{\cdot; \cdot}$ is the mutual information. 
In lossy compression, $\rv Z$ is usually a reconstruction of $\rv X$, \ie,  it aims to faithfully approximate $\rv X$.
As a result, typical distortions, \eg, the mean squared error (MSE), assume that the sample spaces $\mathcal{X},\mathcal{Z}$ of both r.v.s are the same.
This assumption is not required.
Indeed, any distortion $d: \mathcal{X} \times \mathcal{Z} \to \R_{\geq 0}$ of the form $\op{D}{\rv X, \rv Z} = \E{p(\rv X, \rv Z)}{d(\rv X, \rv Z)}$, where there exists a $z \in \rv Z$ such that $\op{D}{\rv X, z}$ is finite, is a valid choice \citep{berger_rate_1968}.
This shows that RD theory can be used outside of reconstructions.
In the following we refer to $\rv Z$ as a compressed \textit{representation} of $\rv X$ to distinguish it from a reconstruction.

\section{Minimal bit-rate for high predictive performance}
\label{sec:theory}
% In this section, we characterize the number of bits needed to store $X$ such that we can still achieve high performance on a set of downstream tasks. We follow these steps: we define a distortion that controls the worst-case performance on a set of tasks when using a compressed $\rv Z$ instead of $\rv X$; we simplify this distortion in the case of invariant tasks by showing that it reduces to the Bayes risk of a certain worst-case task; and we apply RD theory with the simplified distortion to derive the rate needed to ensure high performance on invariant tasks.

%\subsection{High level idea}
In this section, we characterize the bit-rate needed to represent $X$ to ensure high performance on downstream tasks.
% First, we define a distortion that quantifies the decrease in downstream performance when predicting from a compressed $\rv Z$ instead of the source $\rv X$.
% We then simplify and validate this distortion when tasks of interests are invariant to some transformations.
% , by showing that it reduces to how well one can recognize inputs that are invariant to one another using $\rv Z$.
% Finally, we apply RD theory to derive the rate needed to ensure high performance on invariant tasks. 
Our argument has three high-level steps:
\begin{inlinelist}
    \item define a distortion that controls downstream performance when predicting from $\rv Z$ instead of $\rv X$;
    \item simplify and validate this distortion when desired tasks satisfy an invariance condition;
    \item apply RD theory with the valid distortion.
\end{inlinelist}
For simplicity, our presentation is relatively informal; formal proofs are in \cref{appx:preliminaries,appx:proofs}.

\subsection{A distortion for worst-case predictive performance}

Suppose $X$ is an image. Potential downstream tasks might include $Y_{\text{dog}}$, whether the image displays a dog; or $Y_{\text{hd}}$, whether the image is hand-drawn. Formally, these and other downstream tasks are expressed as $\tasks = \set{\rv Y_{\text{dog}}, \rv Y_{\text{hd}}, \ldots}$, a set of random variables that are jointly distributed with $X$. 
Let $\Risk{Y}{X}$ denote the Bayes (best possible) risk when predicting $\rv Y$ from $\rv X$.
For ease of presentation in the main paper, we consider only classification tasks $\tasks$ and Bayes risk of the standard log loss $\Risk{Y}{X} \defeq \inf_q \E{ p(\rv X, \rv Y) }{-\log q(\rv Y | \rv X)}$.
We deal with MSE and regression in \cref{appx:theorem_mse}.

In this setting, a meaningful distortion $\distst{}$ quantifies the difference between predicting any $Y \in \tasks$ from the compressed $Z$, as opposed to using $\rv X$. This is the worst-case excess risk,
\begin{equation}\label{eqn:distortion_def}
\distst{} \defeq  \sup_{\rv Y \in \tasks{}} \quad  \Risk{Y}{Z} - \Risk{Y}{X}.
\end{equation}
If $\distst{}=0$, it is possible to achieve \textit{lossless prediction}: performing as well using $\rv Z$ as using $\rv X$.
More generally, bounding \disttext{} by $\delta$ ensures that $\Risk{Y}{Z} - \Risk{Y}{X} \leq \delta$ for all tasks in $\tasks$. 
However, there are two issues that need to be addressed before \cref{eqn:distortion_def} can be used. 
First, it is not clear whether \disttext{} is a valid distortion for RD theory.
Second, the worst excess-risk \disttext{} assumes access to all downstream tasks of interest $\tasks{}$ during compression, which is unrealistic in general. 
% Without further assumptions on $\tasks$, neither issue is possible to resolve.

\subsection{Invariant tasks}

The tasks that we care about are not arbitrary, and often share structure. 
One such structure is invariance to certain pre-specified transformations of input data. For example, computer vision tasks are often invariant to mild transformations such as brightness changes.
Such invariance structure is common in realistic tasks, as seen by the wide-spread use  of data augmentations \cite{shorten_survey_2019} in machine learning (ML), which encourage predictions to be the same for an unaugmented $x$ and an augmented $x^+$.
Motivated by this we focus on sets of invariant tasks  $\tasks{}$.
\begin{figure}[t]
\centering
 \begin{subfigure}{0.30\linewidth}
\includegraphics[width=\linewidth]{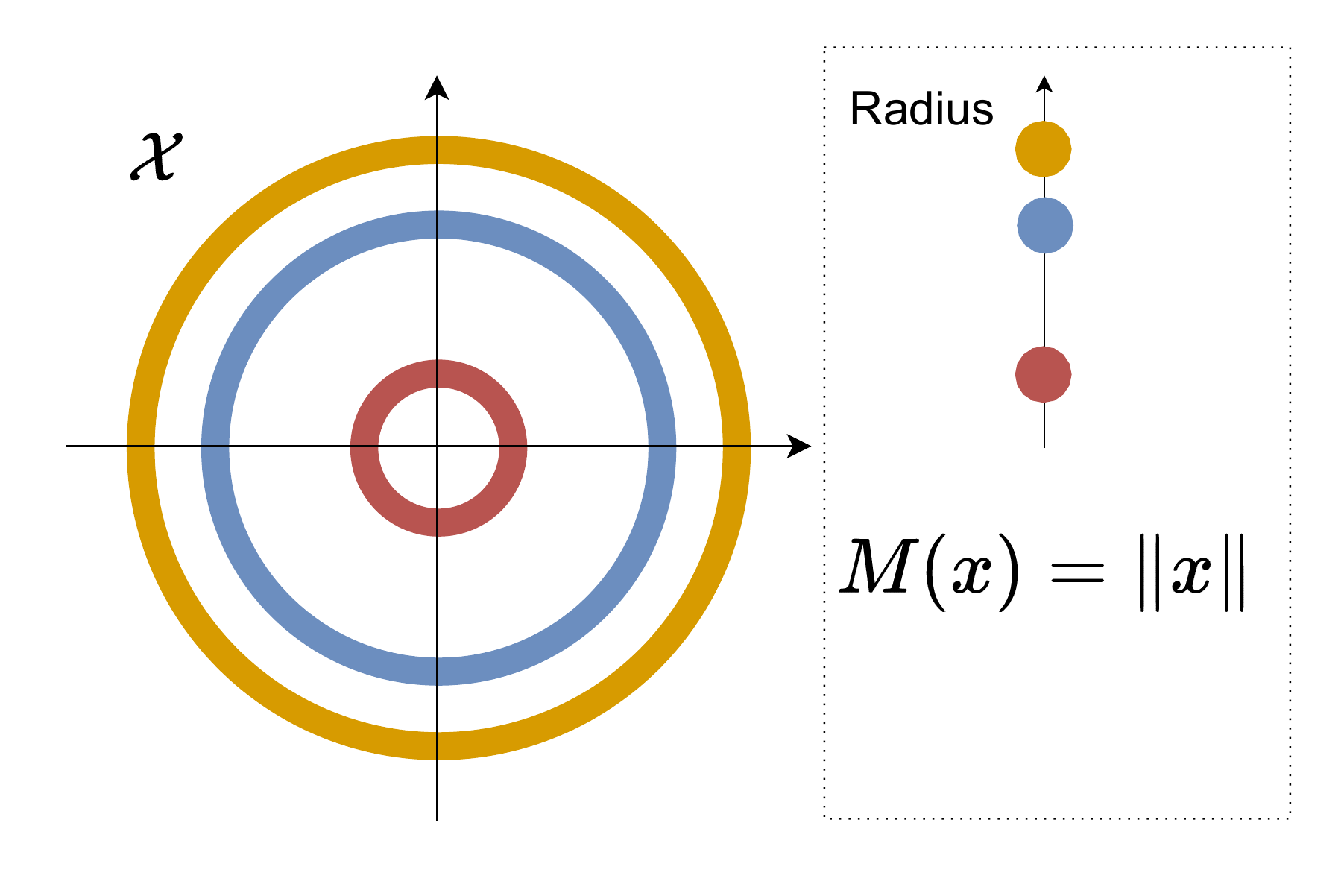}
\vspace*{-1.5em}
  \caption{Rotation}
 \label{fig:Mx_rot}
 \end{subfigure}
 \quad
  \begin{subfigure}{0.30\linewidth}
\includegraphics[width=\linewidth]{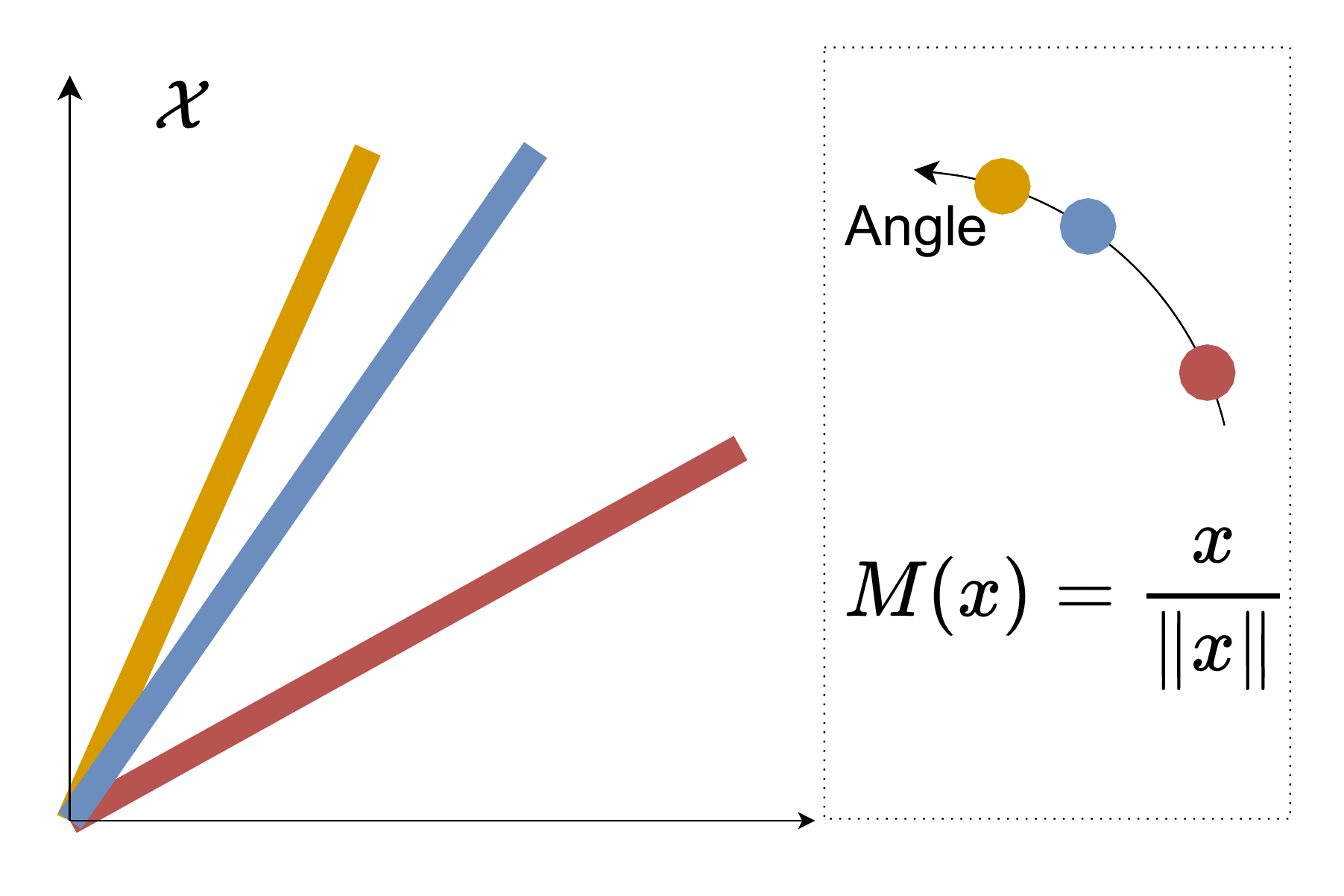}
\vspace*{-1.5em}
  \caption{Scaling}
  \label{fig:Mx_scaling}
 \end{subfigure}
%   \begin{subfigure}{0.30\linewidth}
% \includegraphics[width=\linewidth]{figures/max_inv/max_invariants-Rounding.pdf}
%   \caption{Change of decimals}
%   \label{fig:Mx_dec}
%  \end{subfigure}
  %
 \quad
 \begin{subfigure}{0.30\linewidth}
\includegraphics[width=\linewidth]{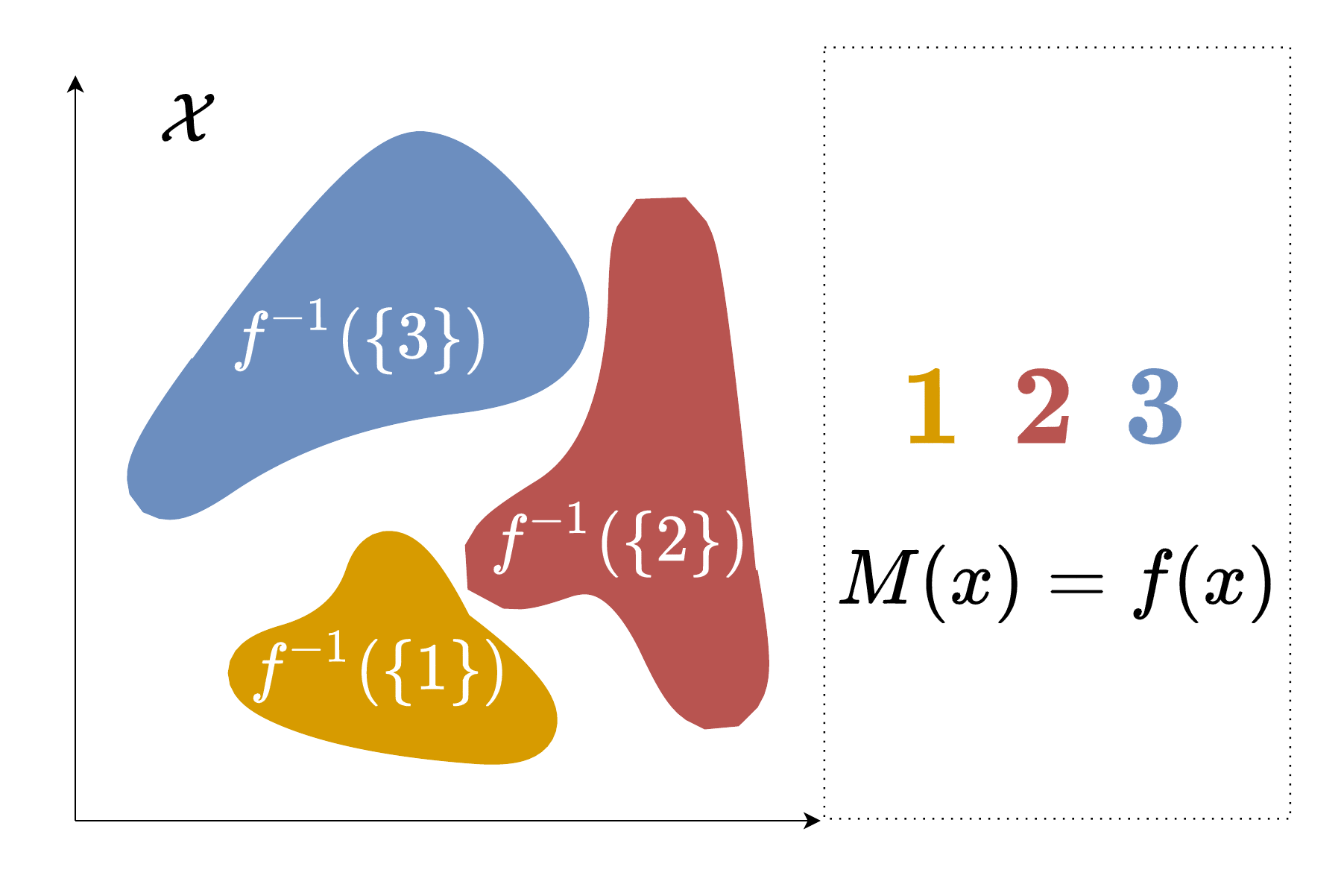}
\vspace*{-1.5em}
  \caption{Any transformation $f$}
  \label{fig:Mx_f}
 \end{subfigure}
 \hfill
  \begin{subfigure}{0.30\linewidth}
\includegraphics[width=\linewidth]{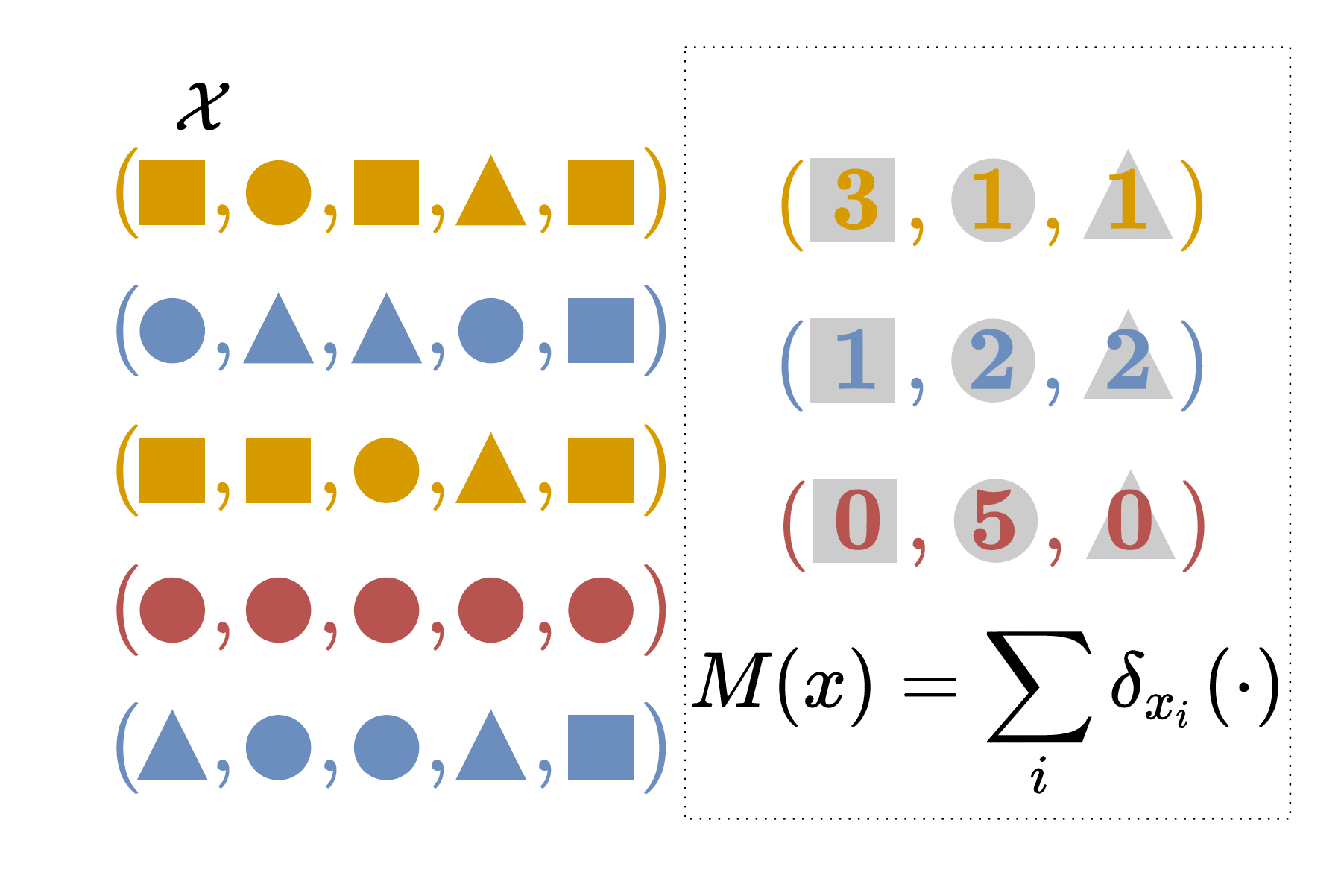}
\vspace*{-1.5em}
  \caption{Permutation}
 \label{fig:Mx_perm}
 \end{subfigure}
 \quad
   \begin{subfigure}{0.30\linewidth}
\includegraphics[width=\linewidth]{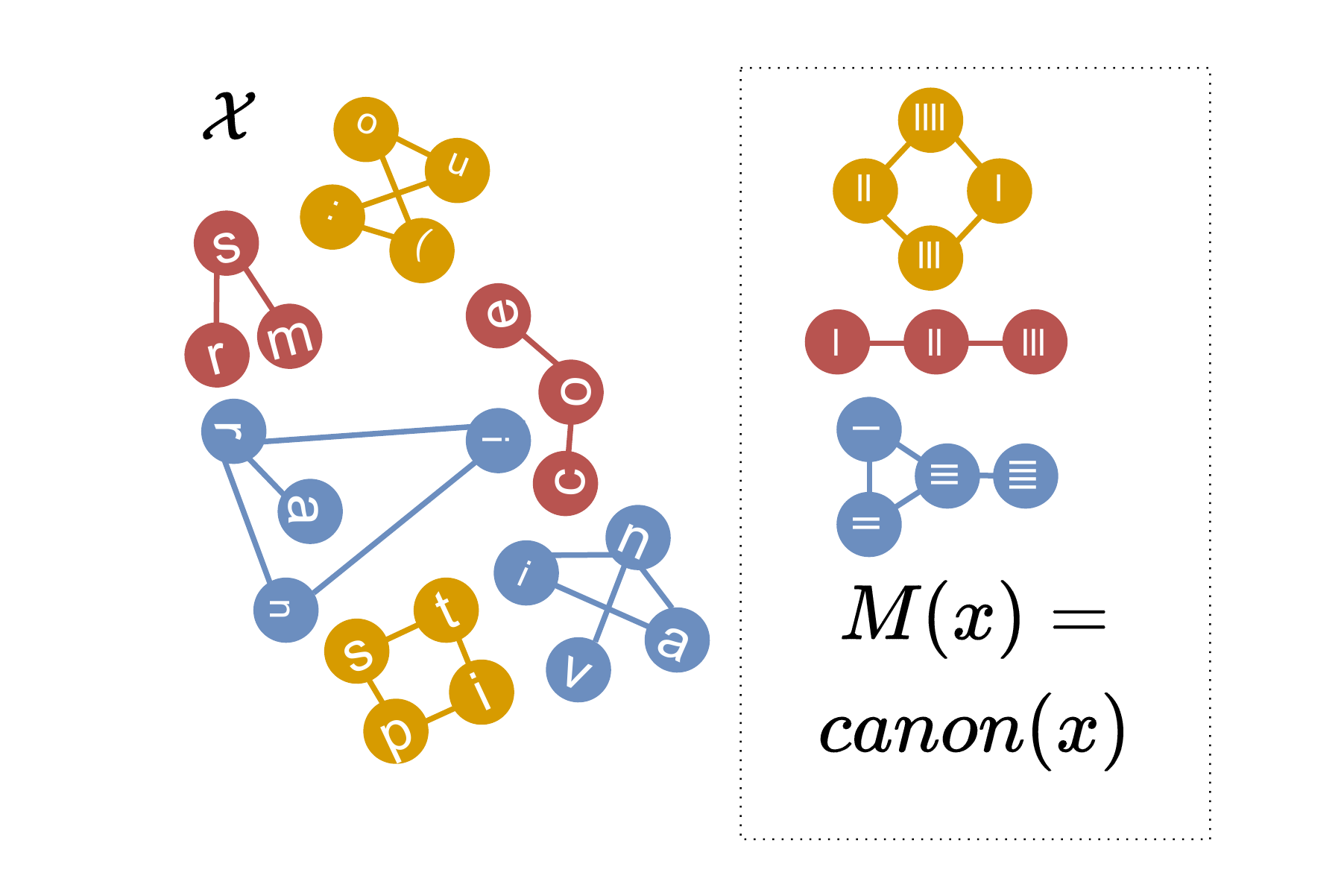}
\vspace*{-1.5em}
  \caption{Graph isomorphism }
  \label{fig:Mx_graph}
  \end{subfigure}
 \quad
   \begin{subfigure}{0.30\linewidth}
\includegraphics[width=\linewidth]{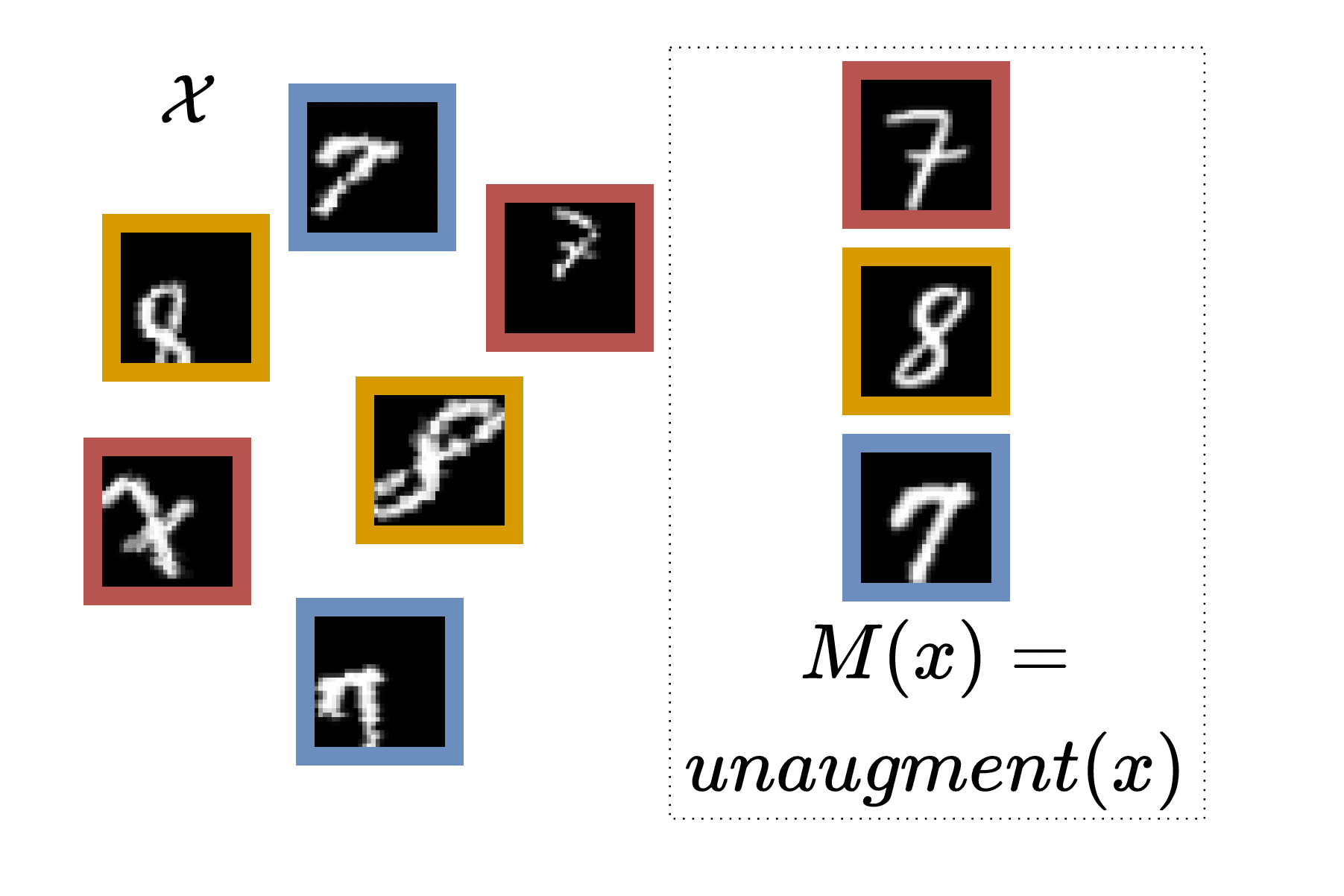}
\vspace*{-1.5em}
  \caption{Any data augmentation}
  \label{fig:Mx_aug}
  \end{subfigure}
\caption{Maximal invariants $M(X)$ are representatives of equivalence classes.
%contain the minimal information needed to perform any invariant task.
Example $M$s include the: (a) Euclidean norm for rotations; (b)  unit vector for scaling; (c) $f$ when equivalence classes are pre-images by $f$; (d) empirical measure for permutations; (e) canonical graph for graph isomorphisms; (f) unaugmented input for data augmentations.}
  \label{fig:max_inv}
\end{figure}

We consider a general notion of invariance, namely invariance specified by an equivalence relation $\sim$ on $\mathcal{X}$.%
\footnote{As a reminder, $\sim$ is an equivalence relation iff for all $x,x',x'' \in \mathcal{X}$: (reflexivity) $x \sim x$, (symmetry) $x \sim x' \iff x' \sim x$, and (transitivity) $x \sim x'$ and $x' \sim x''  \implies   x \sim x''$.
%Invariances to $\sim$ subsumes all invariances to ( groups, semigroups, functions).
}
 The equivalence induces a partition of $\calX$ into disjoint \textit{equivalence classes}, and we are interested in  tasks whose conditional distributions are constant within these classes.
\begin{definition} \label{def:invariant_tasks_interest:main}
The set of \textit{invariant tasks of interest} with respect to an  equivalence $(\mathcal X, \sim)$, denoted $\tasksinv{}$, is all random variables $\rv Y$ such that $x \sim x^+ \implies p(\rv Y \cond x ) = p(\rv Y \cond x^+)$ for any $x,x^+ \in \mathcal{X}$. 
\end{definition}

\subsection{Rate-distortion theory for invariant task prediction} 

The key to simplifying \disttextinv{} is the existence of a (non-unique) worst-case invariant task, denoted $M(X)$.
Such task contains all and only information to which tasks $Y\in\tasksinv$ are not invariant; we call them \textit{maximal invariants}.
A maximal invariant $M(\argdot)$ with respect to $\sim$ is any function satisfying%
\footnote{
This extends the definition of maximal invariants \cite{eaton_group_1989} beyond invariances to group actions.
}
\begin{equation}\label{eqn:max_invariant}
x \sim x^+ \iff M(x) = M(x^+) \quad \text{for any  $x,x^+ \in \mathcal{X}$.}
\end{equation}
A maximal invariant removes all information that tasks are invariant to, as it maps equivalent
inputs to the same output, \ie, $M(x)=M(x^+)$.
Yet, it retains the minimal information needed to perform invariant tasks, by mapping non-equivalent inputs $x \not \sim x^-$ to different outputs $M(x) \neq M(x^-)$.
In other words, $M(x)$ indexes the equivalence classes. 
For example, the Euclidean norm is a maximal invariant for rotation invariance, as all vectors that are rotated versions of one another can be characterized by their radial coordinate. 
For data augmentations, the canonical (unaugmented) version of the input is a maximal invariant. Other examples are shown in  \cref{fig:max_inv}.

We prove in \cref{appx:invariant_distortion} that under weak regularity conditions, maximal invariant tasks exist in $\tasksinv{}$, and that they achieve the supremum in \cref{eqn:distortion_def}.
This allows us to show that \disttextinv{} reduces to the Bayes risk of predicting $M(X)$ from $\rv Z$ and that it is a valid distortion measure. Crucially, this allows us to quantify downstream performance without enumerating invariant tasks.
% \ydnote{not sure how detailed we want this paragraph to be... i.e. whether or not to give an intuition about the proof. Is that already too much?}
%
\begin{restatable}%[Invariance Distortion for Log Loss]
{proposition}{distrisk}
\label{prop:nicer_dist}
Let $(\calX,\sim)$ be an equivalence relation and $M$ a maximal invariant that takes at most countably many values, with $\H{\rv M(X)} < \infty$.
Then \disttextinv{} \eqref{eqn:distortion_def} with log loss is a valid distortion and
\begin{equation}\label{eqn:max_invariance_distortion}
\diststinv{}{} = \Risk{M(\rv X)}{Z} \;.
\end{equation}
\end{restatable}

Here we used $\Risk{M(X)}{X}=0$, as $M$ is a deterministic function.  
Also, note that the countable requirement holds when tasks are invariant to some rounding of the input, as is typically the case due to floating-point storage.
We accommodate the uncountable case for squared-error loss in \cref{appx:theorem_mse}.

With a valid distortion in hand, we invoke the RD theorem with \disttextinv{} to obtain our ``Rate-Invariance'' (RI) theorem.
The RI theorem characterizes the bit-rate needed to store $X$ while ensuring small log-loss on invariant tasks.
% For the following result, we assume that the (differential) entropy $\H{\rv X}$ is finite and $\sim$ induces countably many equivalence classes.
We obtain analogous results for squared-error loss.
\begin{restatable}[Rate-Invariance]{theorem}{rateinvdist}\label{thm:rate_invariance_distortion}
Assume the conditions of \cref{prop:nicer_dist}. 
Let $\delta \geq 0$, and $Rate(\delta)$ denote the minimum achievable bit-rate for transmitting $\rv Z$ such that for any $\rv Y \in \tasksinv$ we have $\Risk{Y}{Z} - \Risk{Y}{X} \leq \delta$.
Then $Rate(\delta)=0$ if $\delta \geq \op{H}{M(\rv X)}$ and otherwise it is finite and
\begin{align}\label{eq:rate_inv}
Rate(\delta) %&= \min_{p(\rv Z| \rv X)}   \op{I}{\rv X; \rv Z} +  \beta(\delta) \cdot \op{H}{M(\rv X) \cond \rv Z}    \label{eq:rate_inv_dist_IZX_unconstrained}\\
\quad = \quad
\underbrace{\ \vphantom{\sum}\H{M(X)} }_{\substack{\text{information needed} \\ \text{to predict } \tasksinv{}}}
\quad
\underbrace{-\ \vphantom{\sum}\delta }_{\substack{\text{acceptable decrease} \\ \text{in predictive loss}}}
\quad = \quad
\underbrace{\vphantom{\sum} \op{H}{\rv X} }_{\substack{\text{standard} \\ \text{compression}}}
\quad
\underbrace{-\ \vphantom{\sum} \op{H}{\rv X \cond M(\rv X)} }_{\substack{\text{gains due to} \\ \text{invarainces}}}
\quad
-\ \delta.
\end{align}
\end{restatable}
\begin{wrapfigure}{r}{0.4\textwidth}
\vspace{-1\baselineskip}
 \centering
 \includegraphics[width=0.98\linewidth]{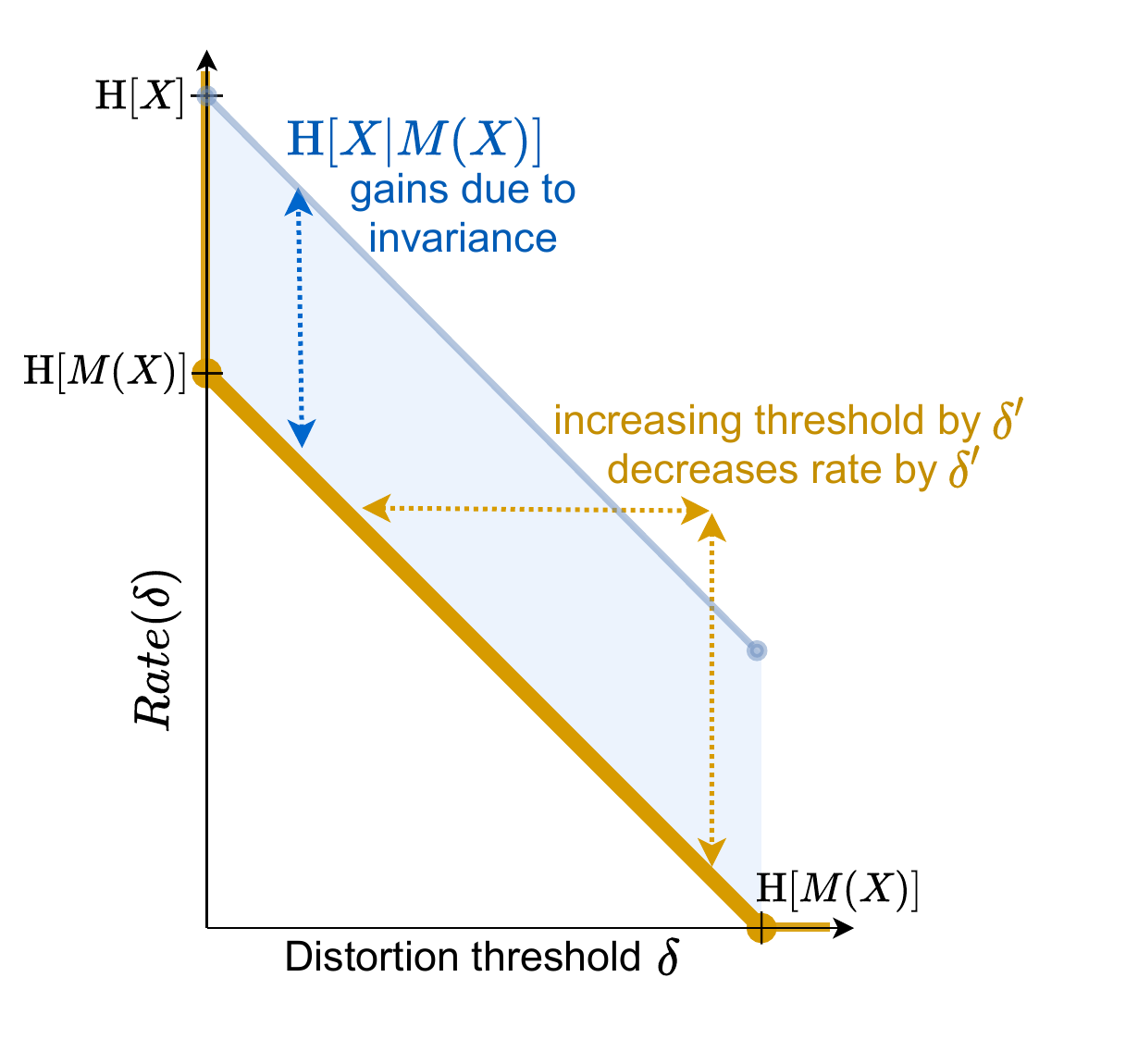}
 \vspace{-1\baselineskip}
 \caption{Rate-Invariance function.}
\label{fig:schema_RD}
\vspace{-1\baselineskip}
\end{wrapfigure}

To ensure lossless prediction, i.e., $Rate(0)$, our theorem states that we require a bit-rate of $\H{M(\rv X)}$.
Intuitively, this is because $M(\rv X)$ contains the minimal information needed to predict losslessly any $\rv Y \in \tasksinv{}$.\footnote{
We prove in \cref{appx:theorem_lossless} that $Rate(0)=\H{M(\rv X)}$ for any losses used in ML.
}
Furthermore, the theorem relates compression and prediction by showing that allowing a $\delta$ decrease in log-loss performance on all tasks can save \textit{exactly} $\delta$ bits. 
Intuitively, this is a linear relationship, because expected log-loss is measured in bits.
On the right of \cref{eq:rate_inv} we further decompose $\op{H}{M(\rv X)}$ into two terms to provide another interpretation:
(i) $\op{H}{\rv X}$, which, for discrete $X$, is the bit-rate required to losslessly compress $\rv X$, and (ii) $\op{H}{\rv X \cond M(\rv X)}$, which quantifies the information removed due to the invariance of desired tasks.
% First, there is $\op{H}{\rv X}$, which, for discrete $X$, is the bit-rate required to losslessly compress $\rv X$ .
% \ydnote{THis explanation holds for discrete $X$. In the case of continuous $X$, explanation only holds if $H[X]$ is the limit of discrete entropy rather than differential entropy (in which case it's infinite, which is ok because $\op{H}{\rv X \cond M(\rv X)}$ will also be infinite) and difference well defined.}
% The second term relates predictions and compression.
% It shows that allowing a $\delta$ decrease in log-loss performance can save \textit{exactly} $\delta$ bits.
% Second, $\op{H}{\rv X \cond M(\rv X)}$ quantifies how much information can be removed due to the invariance structure of desired tasks.
Importantly, removing this information does not impact the best possible predictive performance.
See \cref{fig:schema_RD}.
% Such gain is still be achieved in lossless prediction, \ie, when $\delta=0$.

The bit-rate gains can be substantial, depending on the invariances.
Consider compressing a sequence of $n$ \iid fair coin flips.
 Suppose one is only interested in predicting permutation invariant labels.
 Then instead of compressing the entire sequence in $\op{H}{\rv X^n}=n\op{H}{\rv X}= n$ bits, one could compress the number of heads, which is a maximal invariant for permutation invariance, in $\mathcal{O}(\log n)$ bits.%
 \footnote{The number of heads $\rv K$ follow a binomial distribution so $\op{H}{\rv K} \in \mathcal{O}(\log n)$.
 Here $\rv K$ is also a minimal sufficient statistic for $\rv X^n$.
 More generally, if $P(\rv X')$ is invariant to $\sim$ and $\sim$ is the coarsest such relation, then minimal sufficiency coincides with maximal invariance.
In practice, however, $P(\rv X')$ will rarely be $\sim$ invariant.
 }
As more interesting examples, we recover in \cref{appx:recovering} results from \begin{inlinelist}
\item unlabeled graph compression \cite{rashevsky_life_1955};
\item multiset compression \cite{varshney_benefiting_2007};
\item single task compression (IB; \cite{tishby_information_2000}).
\end{inlinelist}
The equivalence $\sim$ can be induced by any transformations, such as transforming an image to its caption. 
We use this idea in \cref{sec:clip_experiments} to obtain >$1000\times$ compression on ImageNet without sacrificing predictive performance. 

\section{Unsupervised training of invariant neural compressors}
\label{sec:learning}
In this section, we design practical, invariant neural compressors that bound optimal rates. 
Derivations are in \cref{appx:objectives}.
%Specifically, we would like to learn an encoder $p(\rv Z| \rv X)$ that gives rise to a predictive and compressed $\rv Z$ that can easily be coded in practice.
In particular, we are interested in the $\argmin$ encoders $p(\rv Z | \rv X)$ of the RD function (\cref{eq:rate_distortion}) under the invariance distortion \disttextinv{}.
To accomplish this, we can optimize the following equivalent (\ie, it induces the same RI function) Lagrangian, where $\beta$ takes the role of $\delta$,\footnote{
As the RI function is not strictly convex (\cref{fig:schema_RD}), it should be beneficial to use $\Risk{M(\rv X)}{\rv Z}^2$ to ensure that sweeping over $\beta$ is equivalent to sweeping over $\delta$ \cite{kolchinsky_caveats_2019}.
We did not see any difference in practice.
}
% \footnote{
% As the RI function is not strictly convex (\cref{fig:schema_RD}), it should be beneficial to use $\Risk{M(\rv X)}{\rv Z}^2$ to ensure that sweeping over $\beta$ is equivalent to sweeping over $\delta$ \cite{kolchinsky_caveats_2019}.
% We did not see any difference in practice.
% }
%
\begin{equation}\label{eq:unconstrained}
\argmin_{p(\rv Z| \rv X)} \quad \op{I}{\rv X; \rv Z}  \; + \; \beta  \cdot \Risk{M(\rv X)}{\rv Z}.
\end{equation}
%The correspondence between $\beta$ and $\delta$ is discussed in \cref{appx:objectives}, and in practice we use a tuned, constant $\beta$.
%
%The optimal $\beta(\delta)$ is given in \toappx{}, and in practice we use a tuned constant.
%In practice we use a parameterized encoder $p_\varphi(\rv Z| \rv X)$.
In ML, the maximal invariant $M$ is often not available.
Instead, invariances are implicitly specified by sampling a random augmentation from $A$, applying it to a datapoint $X$, and asking that the model's prediction be invariant between $X$ and $A(X)$. 
For example, invariance to cropping can be enforced by randomly cropping images while retaining the original label. 
We show in \cref{appx:objectives}, that in such case, we can treat the augmented $A(X)$ as the new source, $\rv Z$ as the representation of $A(X)$, and the unaugmented $\rv X$ as the maximal invariant task $M(A(X))$.
Indeed, $\Risk{M(A(X))}{Z}$ is equal to $\Risk{X}{Z}$ up to a constant, so we can rewrite \cref{eq:unconstrained} as the following equivalent objective,
\ydnote{here the equivalence is for any beta (i.e. for a given beta it induces the same segment of RI function.
But really we only care about the fact that it's same RI function, so I use "equivalence" which was already defined above.
}
\begin{equation}\label{eq:unconstrained_no_Mx}
\argmin_{p(\rv Z| A(\rv X))} \quad \op{I}{A(\rv X); \rv Z}  \; + \; \beta  \cdot \Risk{\rv X}{\rv Z}.
\end{equation}
Such reformulation is possible if random augmentations retain the invariance structure $X \sim A(X)$ but ``erase'' enough information about equivalent inputs, specifically, if $\rv X \indep A(X) \cond M(X)$.
We discuss the second requirement in \cref{appx:objectives}  but note that it will likely not be a practical issue if the dataset is small compared to the support $|\mathcal{D}| \ll |\mathcal{X}|$.
With this, we have an objective whose r.v.s. are easy to sample from. 
However, both terms in \cref{eq:unconstrained_no_Mx} are still challenging to estimate. 

In the following, we develop two practical variational bounds to \cref{eq:unconstrained_no_Mx}, which can be optimized by stochastic gradient descent \cite{bottou_large-scale_2010} over the encoder's parameters. Both approximations use the standard lossy neural compression bound
% bound on $\op{I}{A(\rv X); \rv Z}$. 
% \footnote{Outside lossy compression, $\op{I}{\rv X; \rv Z}$ is often estimated by the KL divergence between prior and encoder \cite{kingma_auto-encoding_2014,alemi_deep_2017}.
% But no efficient algorithms can achieve this rate \cite{agustsson_universally_2020}.
% See \cite{flamich_compressing_2020,schulman_sending_2020} for an $\mathcal{O}(\exp(\MI{Z}{X}))$ algorithm.
% }
% So, for simplicity, we use the standard upper-bound 
$ \MI{Z}{A(X)}  \leq \H{Z} \leq \min_{\theta}  \E{p(\rv Z)} {-\log q_\theta(\rv z)} $ where $q_\theta(\rv Z)$ is called an entropy model (or a prior) \cite{balle_end--end_2017,theis_lossy_2017}. This has the advantage that the learned $q_\theta$ can be used for entropy coding $Z$ \cite{rissanen_generalized_1976,duda_asymmetric_2009}. See \citet{balle_variational_2018} for possible entropy models. Our two approximations differ in how they upper bound $\Risk{\rv X}{\rv Z}$.
% % \ydnote{It sounds like we only provide R[X|Z] but we provide entire loss}
% In the following, we present two upper bounds on $\Risk{\rv X}{\rv Z}$.
The first uses a reconstruction loss, which attempts to reconstruct the unaugmented input $x \in \mathcal{D}$ from $A(x)$.
The second uses a discrimination loss, which attempts to recognize which examples are augmented versions of the input.

\begin{figure}
\centering
\includegraphics[width=\linewidth]{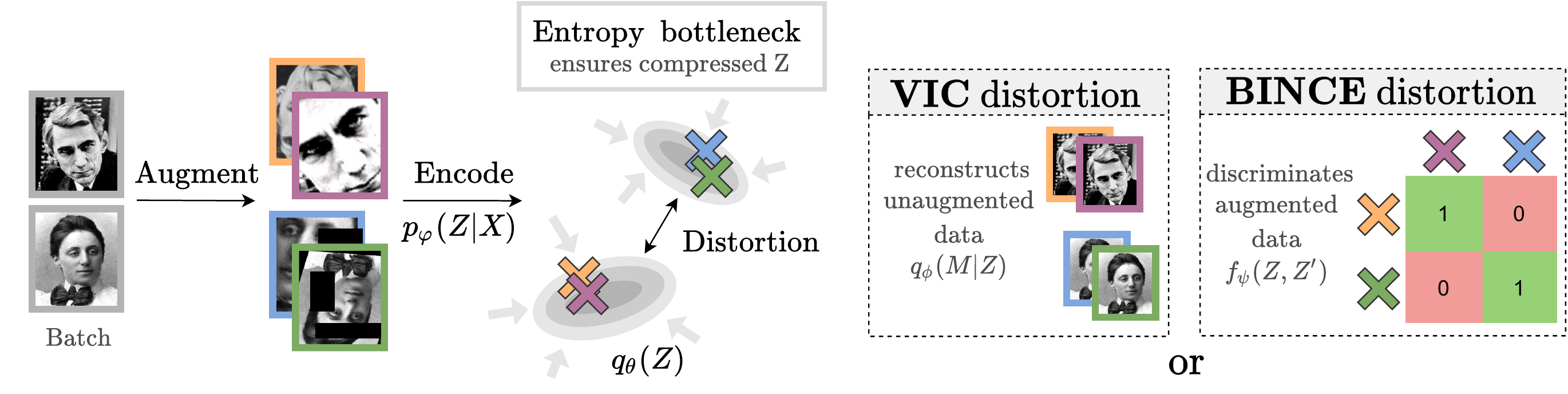}
\vspace*{-1.5em}
\caption{
Our unsupervised objectives for invariant image compression under data augmentation use the same encoder, but differ in their approximation to the invariance distortion.
Both models encode the augmented data, pass the representation through an entropy bottleneck which ensures that they are compressed, and use a distortion to retain the information about the identity of the original data.
The models differ in how they retain that information: 
(VIC) by reconstructing unaugmented inputs;
(BINCE) by recognizing which inputs come from the same original data. 
}\label{fig:objectives}
\vspace{-1\baselineskip}
\end{figure}

\subsection{Variational Invariant Compressor (VIC)}

Our first model is a modified neural compressor in which inputs are augmented but target reconstructions are not.
We refer to it as a \textit{variational invariant compressor} (VIC).
See \cref{fig:objectives} for an illustration.
% and an algorithm to train VIC.
VIC has an encoder $p_\varphi(\rv Z| A(\rv X))$, an entropy model $q_\theta(\rv Z)$, and a decoder $q_\phi(\rv X | \rv Z)$.
Given a data sample $x \in \mathcal{D}$, we apply a random augmentation $A(x)$, and encode it to get a representation $Z$.
The decoder then attempts to reconstruct the unaugmented $x$ from $Z$. This leads to the objective,
\begin{equation}\label{eq:invariant_vae}
\Livae{}(\phi, \theta, \varphi)
\defeq
-
\sum_{x \in \mathcal{D}}
\E{p(\rv A)p_\varphi(\rv Z|A(x))}{
\log q_\theta(\rv Z) + \beta \cdot \log q_\phi(x \cond \rv Z)}.
\end{equation}
The term $\log q_\theta(\rv Z)$ is an entropy bottleneck, which bounds the rate $\op{I}{A(\rv X); \rv Z}$ and ensures that unnecessary information is removed.
The term $\log q_\phi(x \cond \rv Z)$ bounds the distortion $\Risk{X}{Z} \leq  \E{p(\rv X, \rv Z)} {-\log q_\phi(X \cond \rv Z)} $ and ensures that VIC preserves the information needed for invariant tasks.
% It arises from the natural bound $\Risk{X}{Z} \leq \inf_{\phi} \E{p(\rv X, \rv Z)} {-\log q_\phi(X \cond \rv Z)} $.

% \ydnote{Not sure whether this paragraph about high dimensional prediction should be the end of VIC section or Begining of BINCE}

\subsection{Bottleneck InfoNCE (BINCE)}
\label{sec:BINCE}

Our second compressor retains all predictive information without reconstructing the data. It has two components: an entropy bottleneck and an InfoNCE \cite{oord_representation_2019} objective, which is the standard in contrastive SSL. We refer to this as the \textit{bottleneck InfoNCE} (BINCE), see \cref{fig:objectives}. BINCE has an advantage over VIC in that it avoids the problem of reconstructing possibly  high dimensional data.

% \footnote{The bottleneck arises from our desire to increase compression, but this SSL objective can be interested in its own right as such bottleneck provably improve generalization of downstream predictors \citep{dubois_learning_2020}.}

\begin{wrapfigure}{r}{0.45\textwidth}
\vspace{-2.1\baselineskip} % -1 when deal with intro
\begin{minipage}{0.45\textwidth}
 \centering
\newcommand{\vspacing}{$\vphantom{\sum_{i}^t}$} %$\vphantom{\sum_{i}^t}$
\begin{algorithm}[H]
\small
\caption{BINCE's forward pass for $x$}
\label{alg:BINCE}
    \begin{algorithmic}[1]
    \Require \vspacing $p_{\varphi}$,  $ q_{\theta}$, $f_{\psi}$, $\mathcal{D}$, $\rv A$, $\beta$, $n$, $x$ 
    %\State \vspacing $x \leftarrow  \mathrm{select}(\mathcal{D})$  
    \State \vspacing $\tilde{x} \leftarrow \text{sample}(A(x))$ \Comment{Augment}
    \State \vspacing $z \leftarrow \mathrm{sample}(p_{\varphi}(\rv Z | \tilde{x} ))$ 
    \Comment{Encode}
    \State \vspacing $ \mathrm{rate\_loss}  \leftarrow - \log q_{\theta}(z)$ 
    \State \vspacing $\{ x_i^- \}_{i=1}^n  \leftarrow  \mathrm{select}(\mathcal{D} \setminus \set{x})$ $n$ times  
    \State \vspacing $\tilde{\mathbf{x}} \leftarrow  \text{sample}([A(x), A(x_1^-), \dots, A(x_n^-)])$ 
    \State \vspacing $\mathbf{z}  \leftarrow \mathrm{sample}(p_{\varphi}(\rv Z | \tilde{\mathbf{x}}  ))$ 
    \State \vspacing $z^+  \leftarrow \mathbf{z}[0]$ 
    \State \vspacing $\mathrm{softmax} \leftarrow \frac{\exp f_{\psi}( z^{+}, z) }{(\sum_{z' \in \mathbf{z}} \exp f_{\psi}( z',  z))}$
    \State \vspacing $\mathrm{distortion\_loss}  \leftarrow - \log (\mathrm{softmax}) $  \\
    \Return \vspacing $ \mathrm{rate\_loss} + \beta \cdot \mathrm{distortion\_loss}$
\end{algorithmic}
\end{algorithm}
\end{minipage}
\vspace{-1.7\baselineskip}
\end{wrapfigure}

\Cref{alg:BINCE} shows how to train BINCE, where each call to $A$ returns an independent augmentation of its input.
As with VIC, for every datapoint $x \in \mathcal{D}$, we obtain a representation $Z$ by applying an augmentation $A(x)$ and passing it through the encoder $p_\varphi(\rv Z \cond A(\rv X))$.
%Contrary to VIC, we need a sequence of
%We now describe how to sample a sequence $\boldsymbol{z} = (z_{+}, z_1, \ldots, z_n)$
We then sample a ``positive'' example $Z^+$ by encoding a different augmented version of the same underlying datapoint $x$.
Finally, we sample $n$ ``negative'' examples $Z_i^-$ by encoding augmentations $A(x_i^-)$ of datapoints $x_i^- \in \mathcal{D}$ that are different from $x$.
This results in a sequence $\boldsymbol{Z} = (Z^{+}, Z_1^-, \ldots, Z_n^-)$.
%of a single positive and $n$ negative examples with respect to the current $z$.
For conciseness we will denote the above sampling procedure as $p_\varphi(Z,\boldsymbol{Z} \cond A, \mathcal{D}, x)$.
The final loss uses a discriminator $f_\psi$ that is optimized to score the equivalence of two representation,
% Our second loss is based on InfoNCE \citep{oord_representation_2019}. For every $X$, we sample a sequence of random points $\boldsymbol{X} = (X_{+}, X_1, \ldots, X_n)$, where $X_{+} = A(X)$ is an augmentation of $X$ (``positive'') and each $X_i$ are non equivalent examples $X_i \not\sim X$ (``negatives'').
% Let $\boldsymbol{Z} = (Z_{+}, Z_1, \ldots, Z_n)$, be the corresponding representations given by $p_\varphi(\rv Z \cond \rv X)$. Let $f_\psi$ be a discriminator that is optimized to score the equivalence of two representation. The final loss is:
%
\begin{equation}\label{eq:bottlenecked_simclr}
\Lbince{}(\varphi, \theta, \psi) \defeq
-
\sum_{x \in \mathcal{D}}  \E{p(A) p_{\varphi}(Z, \boldsymbol{\rv Z}| A,\mathcal{D},x)}{ \log q_\theta(\rv Z)
+ \beta \cdot \log \frac{\exp f_\psi(\rv Z^{+}, \rv Z)}{\sum_{\rv Z' \in \mathbf{Z}} \exp f_\psi(\rv Z', \rv Z) }}.
\end{equation}
BINCE retains the necessary information by classifying (as seen by the softmax) which $\rv Z$ is associated with an equivalent example $X$.
Both VIC and BINCE give rise to efficient compressors by passing $\rv X$ through $p_{\varphi}(\rv Z |X)$ and entropy coding using $q_\theta(\rv Z)$.
In theory they can both recover the optimal rate for lossless predictions, \ie, $\H{M(X)}$, in the limit of infinite samples ($|\mathcal D|$,$n$) and unconstrained variational families.
In practice, BINCE  has the advantage over VIC of 
\begin{inlinelist}
\item not requiring a high dimensional decoder; and
\item  giving (for suitable $f_{\psi}$) representations that are approximately linearly separable \cite{saunshi_theoretical_2019,tosh_contrastive_2021,lee_predicting_2020} and thus easy to predict from \cite{chen_simple_2020,oord_representation_2019}.
\end{inlinelist}
The disadvantages of BINCE are that it
\begin{inlinelist}
\item does not provide to reconstructions diminishes interpretability; and
\item has a high bias, unless the number of negative samples $n$ is large \cite{poole_variational_2019,song_multi-label_2020}, which is computationally intensive.
\end{inlinelist}

\section{Experiments}
\label{sec:experiments}
We evaluated our framework focusing on two questions:
\begin{inlinelist}
\item What compression rates can our framework achieve at what cost?
\item Can we train a general purpose predictive image compressor?
\end{inlinelist}
For all experiments, we train the compressors, freeze them, train the downstream predictors, and finally evaluate both on a test set.
For classical compressors, standard neural compressors (VC) and our VIC, we used either reconstructions $\Tilde{X}$ as inputs to the predictors or representations $\rv Z$.
As BINCE does not provide reconstructions, we predicted from the compressed $\rv Z$ using a multi-layer perceptron (MLP).
We used ResNet18 \cite{he_deep_2016} for encoders and image predictors.
For entropy models we used \citepos{balle_variational_2018} hyperprior, which uses uniform quantization.
%To produce RD curves we swept over $\beta$.
We optimized hyper-parameters on validation using random search.
%All reported results are on the test set. 
For classification tasks, we report classification error instead of log-loss.
The former is more standard and gave similar results (see \cref{appx:mnist}).
For experimental details see \cref{appx:reproducability}. 
For additional results see \cref{appx:results}.
Code is at \codeurl{}.

\subsection{Building intuition with toy experiments}
\label{sec:toy_experiments}

\begin{figure}
     \centering
     \begin{minipage}{.36\linewidth}
     \begin{subfigure}[h]{\linewidth}
         \centering
         \includegraphics[width=\linewidth]{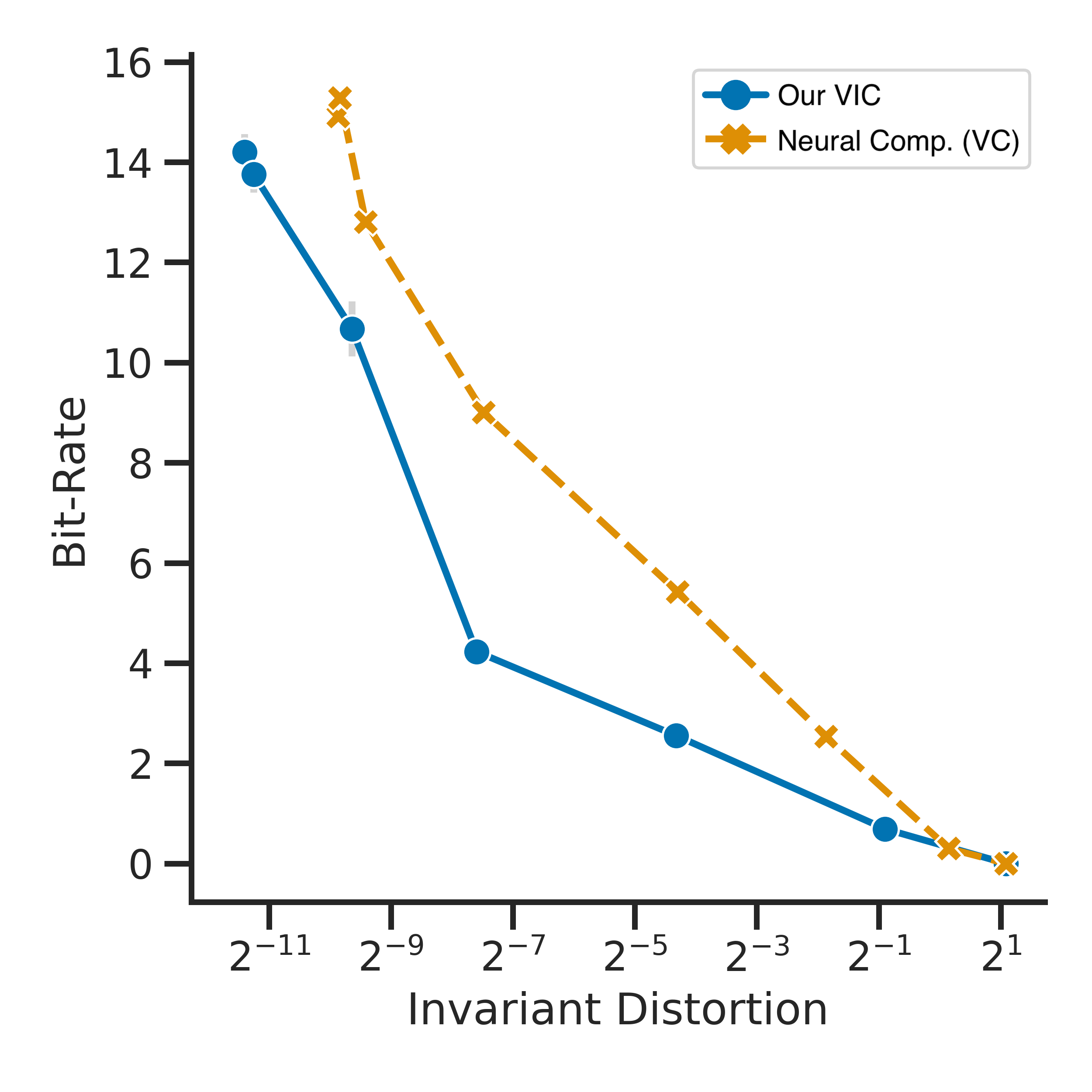}
         \vspace{-2em}
         \caption{Rate-Invariance curves}
         \label{fig:bananas_RI}
     \end{subfigure}
     \end{minipage}
     \qquad
     \begin{minipage}{.48\linewidth}
     %\captionsetup[subfloat]{captionskip=40pt}
     \begin{subfigure}{0.32\columnwidth}
         \centering
         \includegraphics[width=\textwidth]{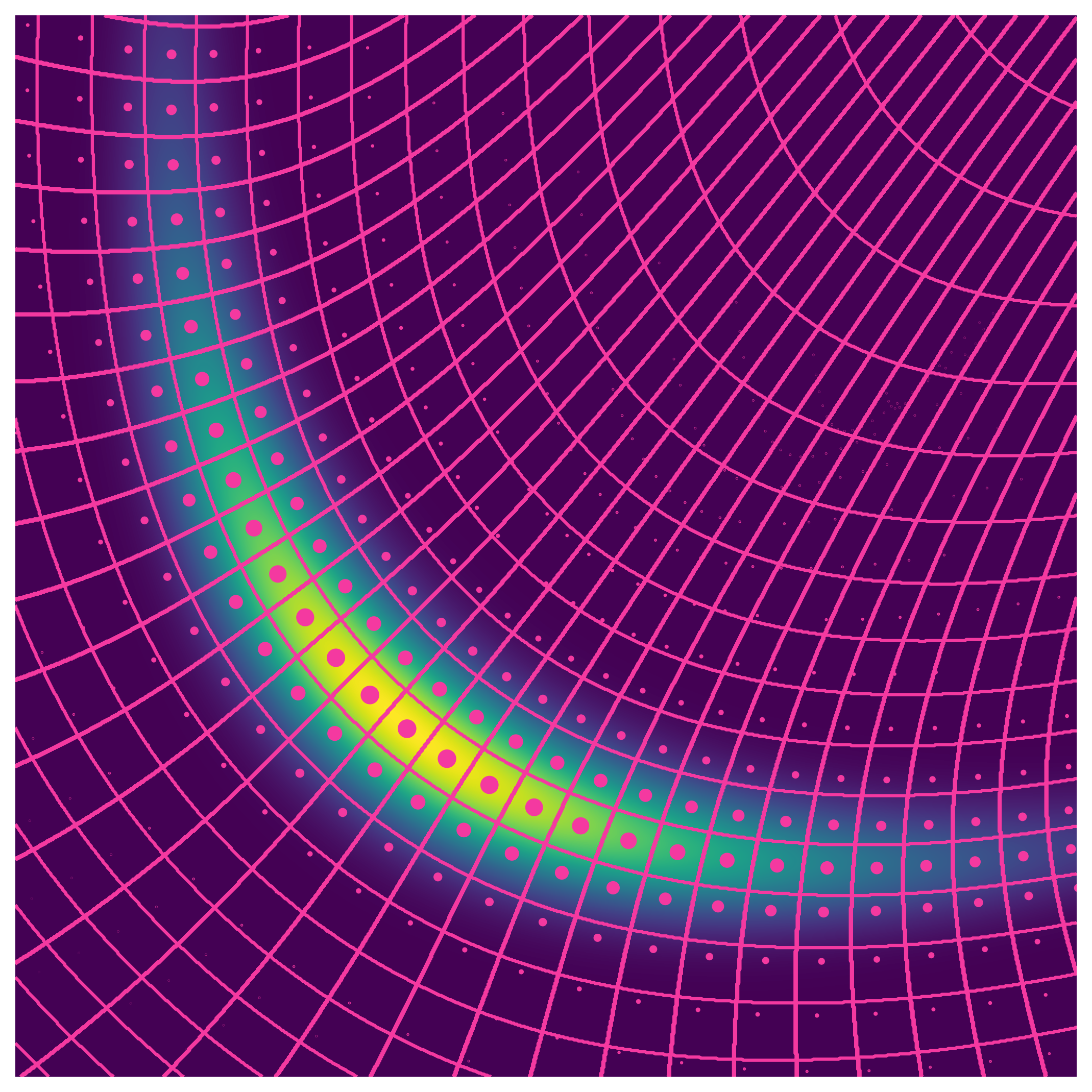}
         \vspace{-1.3em}
         \caption{VC high rate}
         \label{fig:bananas_sweepvae_100}
     \end{subfigure}
     \begin{subfigure}{0.32\columnwidth}
         \centering
         \includegraphics[width=\textwidth]{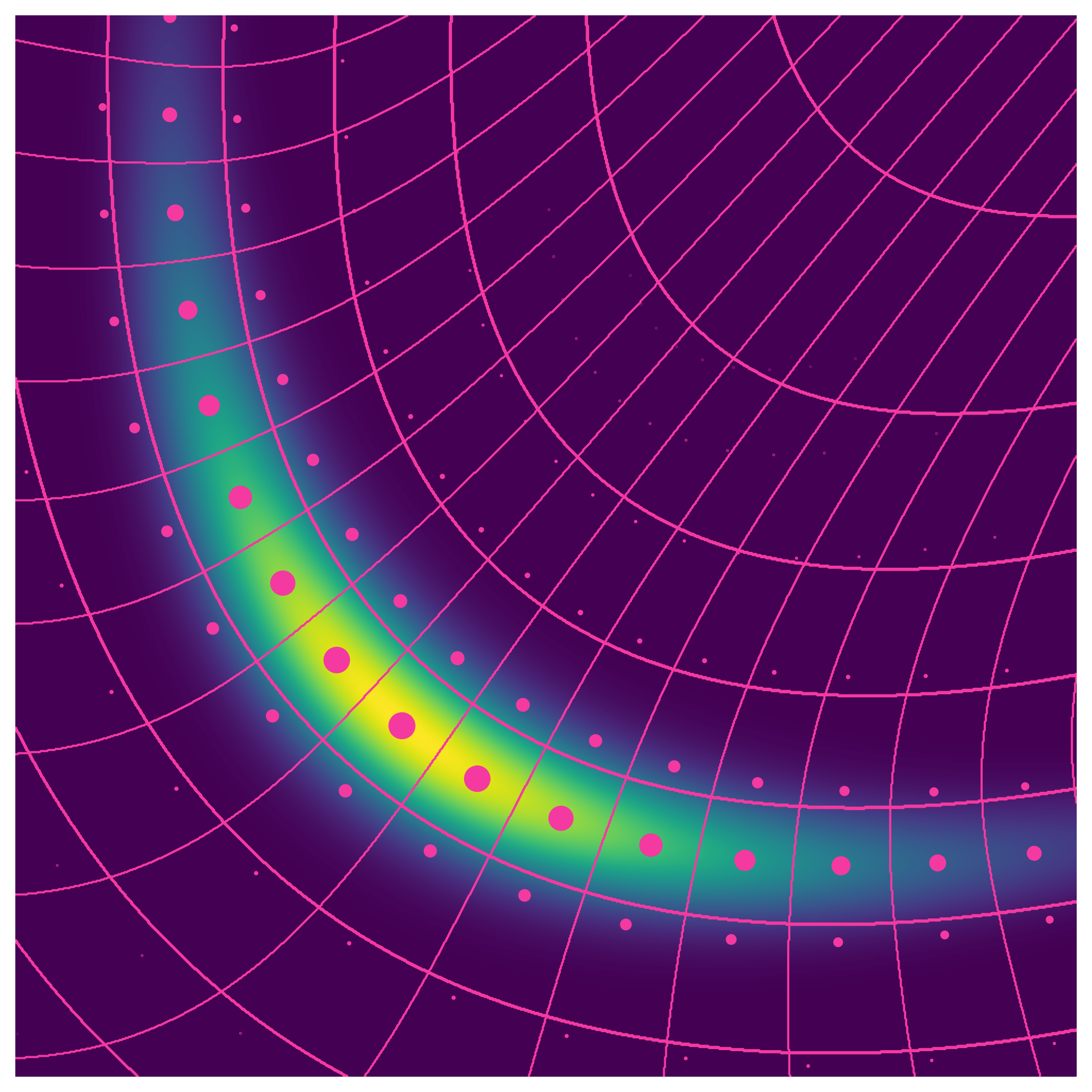}
         \vspace{-1.3em}
         \caption{VC}
         \label{fig:bananas_sweepvae_10}
     \end{subfigure}
     \begin{subfigure}{0.32\columnwidth}
         \centering
         \includegraphics[width=\textwidth]{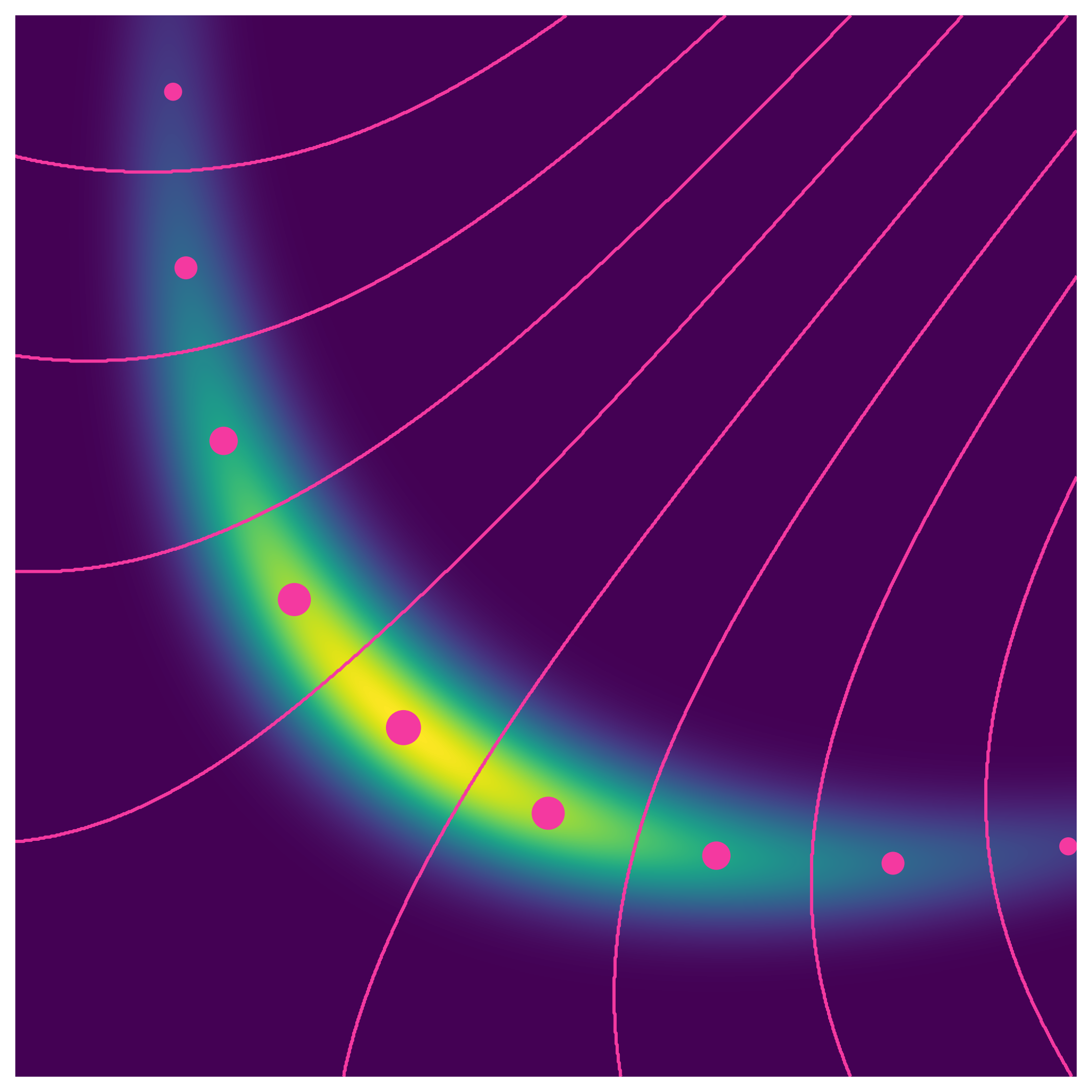}
         \vspace{-1.3em}
         \caption{VC low rate}
         \label{fig:bananas_sweepvae_1}
     \end{subfigure}

     \begin{subfigure}{0.32\columnwidth}
         \centering
         \includegraphics[width=\textwidth]{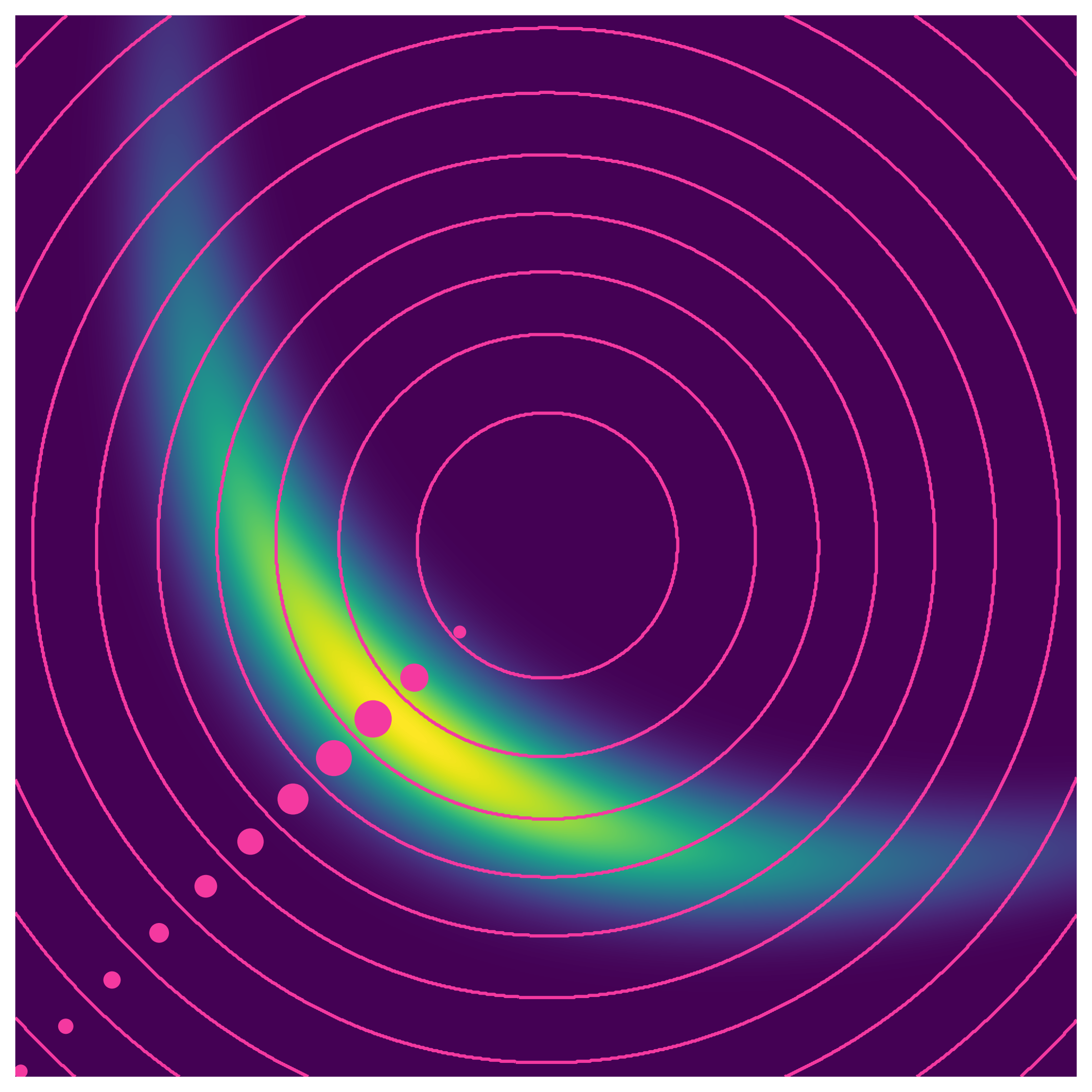}
         \vspace{-1.3em}
         \caption{VIC high rate}
         \label{fig:bananas_sweepivae_100}
     \end{subfigure}
     \begin{subfigure}{0.32\columnwidth}
         \centering
         \includegraphics[width=\textwidth]{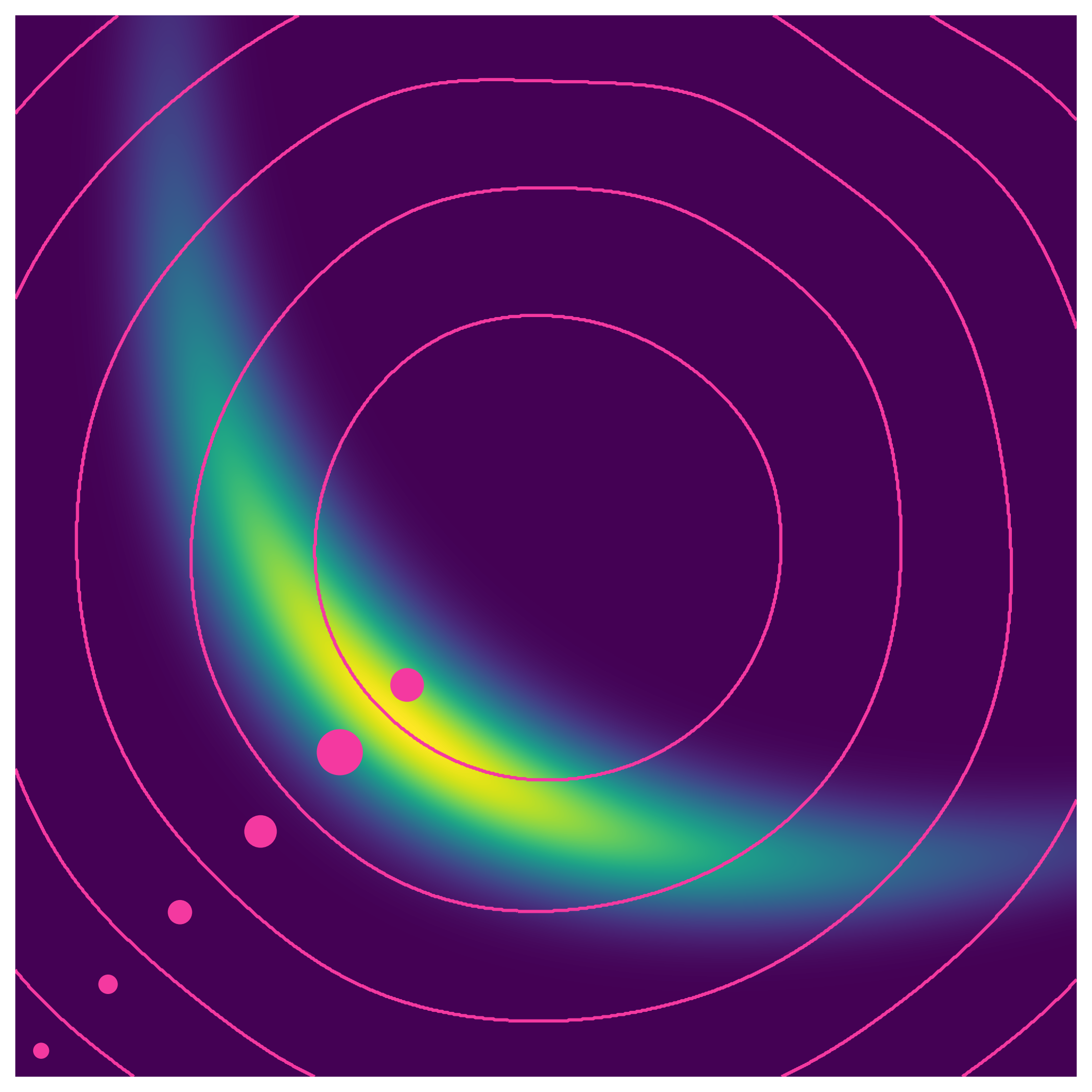}
         \vspace{-1.3em}
         \caption{VIC }
         \label{fig:bananas_sweepivae_10}
     \end{subfigure}
     \begin{subfigure}{0.32\columnwidth}
         \centering
         \includegraphics[width=\textwidth]{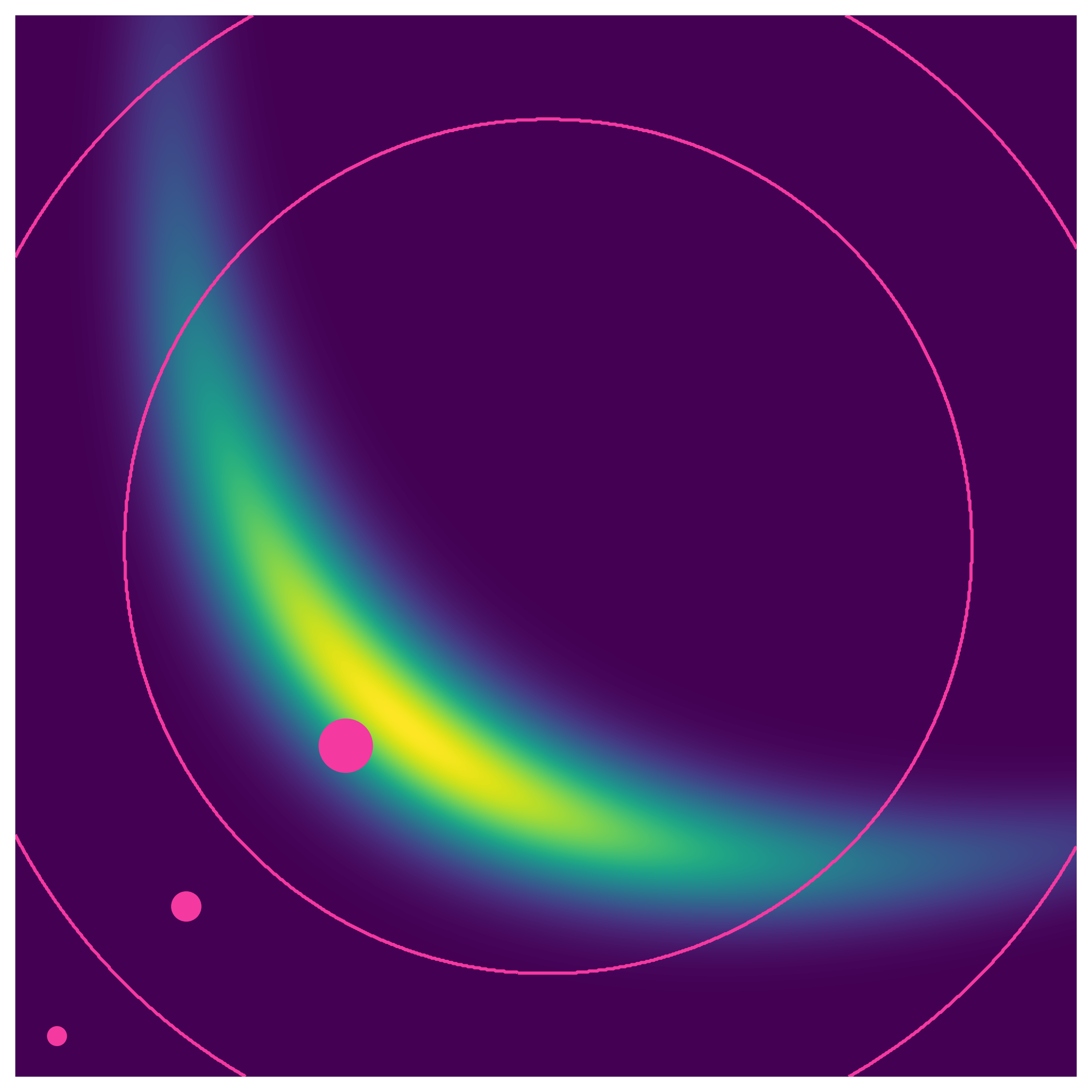}
         \vspace{-1.3em}
         \caption{VIC low rate }
         \label{fig:bananas_sweepivae_1}
     \end{subfigure}
     \end{minipage}
\caption{
Compression rates of a Banana source \cite{balle_nonlinear_2020} can be decreased when downstream tasks are rotation invariant. 
(Left) Our invariant compressor (VIC, blue) outperforms neural compressors (VC, orange). 5 runs with standard errors in gray.
(Right) VIC quantizes the space using disks to remove unnecessary angular information.
Pink lines are quantization boundaries, dots are code vectors with size proportional to learned probabilities.
Low rates correspond to low $\beta$ in \cref{eq:unconstrained_no_Mx}.
}
\label{fig:bananas_sweeps}
\vspace{-1\baselineskip}
\end{figure}
\ydnote{TODO: make gray line in RD curve more visible}

To build an visual intuition, we compressed samples from a 2D banana source distribution \cite{balle_nonlinear_2020}, assuming rotation invariant tasks, \eg, classifying whether points are in the unit circle.
We also compressed MNIST digits as in \cref{fig:mnist_intro}.
Digits are augmented (rotations, translations, shearing, scaling) both at train and test time to ensure that our invariance assumption still holds.

\paragraphQ{Where do our rate gains come from}
%Both models optimize squared-error loss.
For rotation invariant tasks, our method (VIC) discards unnecessary angular information by learning disk-shaped quantizations (\cref{fig:bananas_sweeps}, bottom right).
Specifically, VIC retains only radial information 
% ($M(x) = \| x \|_2$)
by mapping all randomly rotated points (disks) back to maximal invariants (pink dots).
In contrast, standard neural compressors (VC) attempt to reconstruct all information, which requires a finer partition (\cref{fig:bananas_sweeps}, top right).
As a result (\Cref{fig:bananas_RI}), VIC needs a smaller bit-rate ($y$-axis) for the same desired performance (\disttextinv{}, $x$-axis).
The area under the RD curve (AURD) for VIC is $35.8 {\scriptsize \pm 4.2}$ against $48.1 {\scriptsize \pm 0.3}$ for VC, \ie, expected bit-rate gains are around $70\%$.
Similar gains are achieved for augmented MNIST in  \cref{fig:augmnist++} by reconstructing canonical digits.

\begin{figure}[h]
     \centering
     \begin{subfigure}[h]{0.31\columnwidth}
         \centering
         \includegraphics[width=\textwidth]{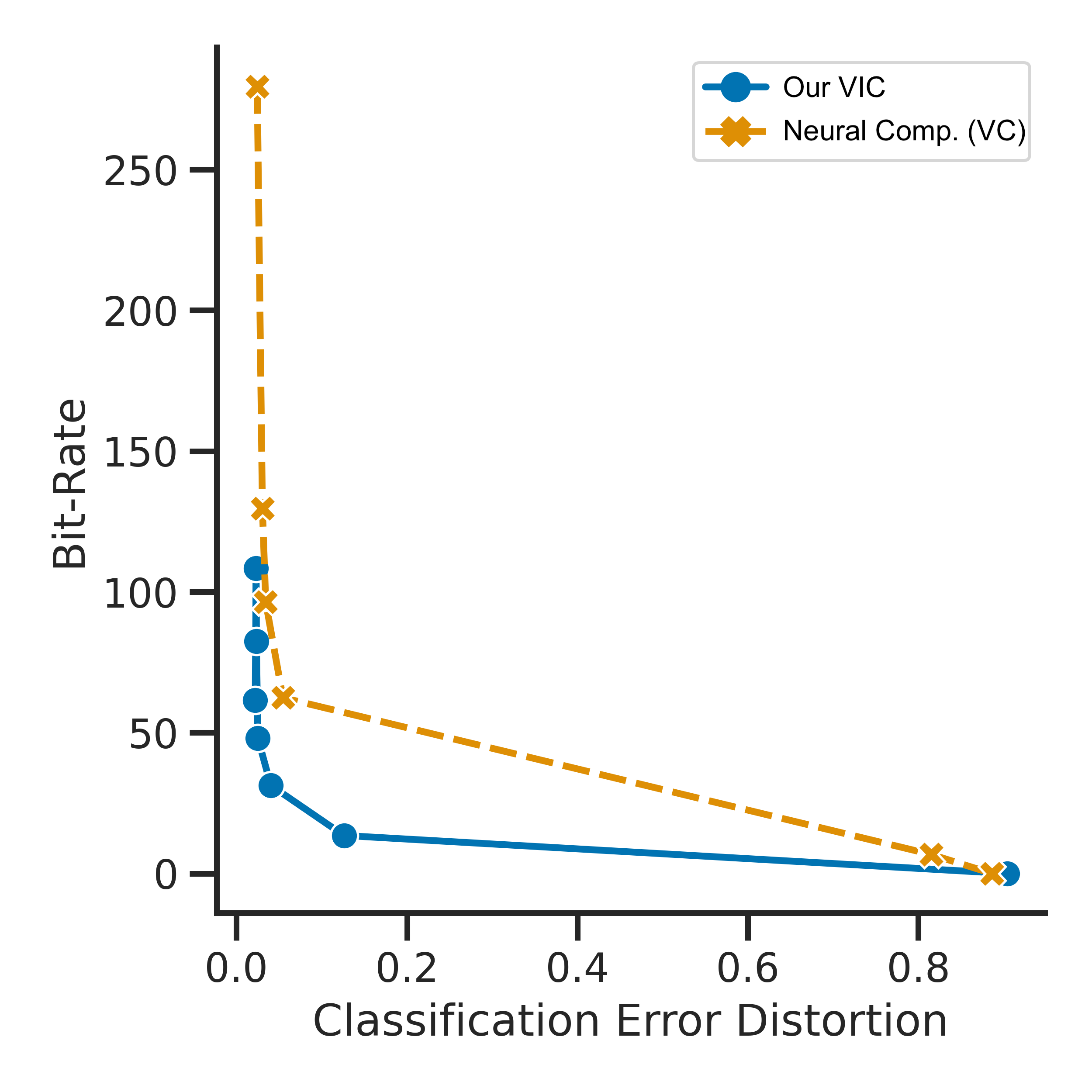}
         \vspace{-1.7em}
         \caption{Rate-Error curve}
         \label{fig:augmnist++_RD}
     \end{subfigure}
     \begin{subfigure}[h]{0.68\columnwidth}
         \centering
         \includegraphics[width=\textwidth]{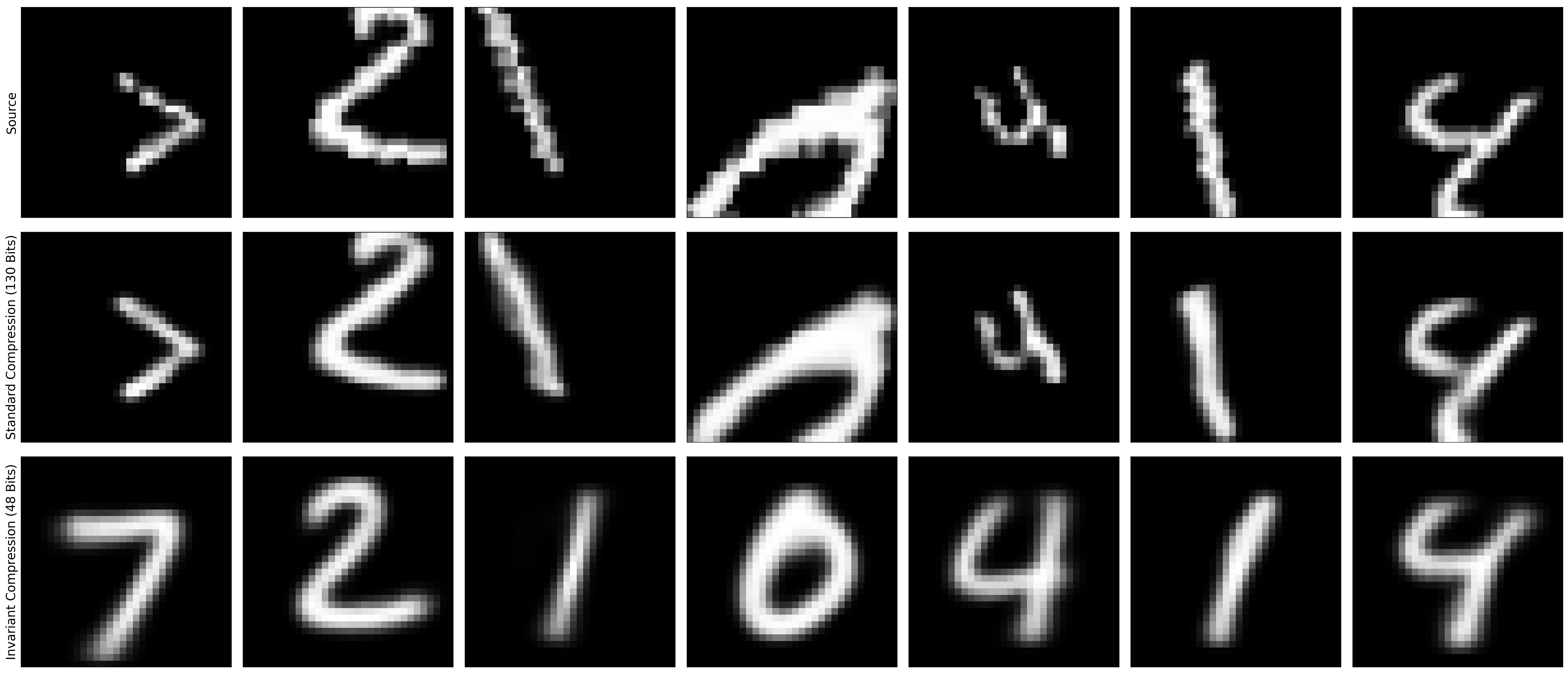}
         \vspace{-1.2em}
         \caption{Reconstructions that allow $99\%$ downstream accuracy}
         \label{fig:augmnist++_RD_rec}
     \end{subfigure}
\caption{
(Left) By reconstructing prototypical digits our VIC (blue) achieves higher compression of augmented MNIST digits than standard neural compressors (VC, orange) without hindering downstream  classification. 5 runs.
(Right) The source examples (first row) as well as reconstructions for the non-invariant (second row) and invariant compressor (last row).
}
\label{fig:augmnist++}
\vspace{-0.5em}
%\vspace{-1\baselineskip}
\end{figure}

\paragraphQ{Can we recover the optimal bit-rate}
We investigated whether our losses can achieve the optimal bit-rate for lossy prediction by using supervised augmentations, \ie, $A(x)$ randomly samples a train example $x^+$ that has the same label.
For MNIST the single-task optimal bit-rate is $\H{Y} = \log(10) \approx 3.3$ bits. 
 VIC and BINCE respectively achieve $5.7$ and $5.9$ bits, which shows that our losses are relatively good despite practical approximations. Details in \cref{appx:mnist}.

\paragraphQ{What is the impact of the choice of augmentations}
The choice of augmentation $A$ implicitly defines the desired task-set $\tasks{}$, \ie, $\tasks{}$ is the set of all tasks for which $A$ does not remove information.
As a result \cref{thm:rate_invariance_distortion} can be rewritten as $Rate(\delta) = \MI{X}{A(X)} - \delta$, so the rate decreases when $A$ removes more information from $\rv X$.
To illustrate this we trained our VIC using three augmentation sets on MNIST, all of which keep the true label invariant but progressively discard more $\rv X$ information.
VIC respectively achieves a rate of $185.3$, $79.0$, and $5.7$ bits, which shows the importance of using augmentations that remove $\rv X$ information.
Details and BINCE results are in \cref{appx:mnist}.

%%%%%%%%%%%%%%%%%%%%%%%%%%%%%%%%%%%%%%%%%%%%%%%%%%%

\subsection{Evaluating our methods with controlled experiments}
To investigate our methods, we compressed the STL10 dataset \cite{coates_analysis_2011}.
We augment (flipping, color jittering, cropping) the train \textit{and} test set, to ensure that the task invariance assumptions are satisfied.
% This makes comparison to existing literature harder.
We focus on more realistic settings in the next section.
%In the main paper, we evaluate compression rates at best predictive performance.
%In \Cref{appx:stl10} we provide rate distortion curves for all experiments.
% \begin{inlinelist}
% \item the compression rate at best performance,
% \item the rate-distortion curves, and
% \item the AURD, the expected rate for a uniformly sampled rate-distortion trade-off.
% \end{inlinelist}
% Metrics that are not presented in the main text are listed in \cref{appx:stl10}.
In each experiment, we sampled $100$ combinations of hyper-parameters to ensure equal computational budget across models and baselines.
%In \cref{appx:stl10} we provide additional results including all RD curves.

\begin{table}[h]
\vspace{-1\baselineskip}
\caption{
Invariant compressors (BINCE, VIC) outperform classical (PNG, JPEG, WebP) and neural (VC) compressors on STL10. BINCE achieves lossless prediction but compresses $121 \times$ better. 
}
\center
\small
\begin{tabular}{lrrrrrrr}
\toprule
& PNG \cite{graphics_png_isoiec_2003}    &  JPEG  \cite{group_jpeg_itu-t_1992}   & WebP   \cite{webp_google_2018} &  VC  $\rv \Tilde{X}$ &  \textbf{VIC} $\rv \Tilde{X}$ & \textbf{VIC} $\rv Z$  & \textbf{BINCE} \\ %
\midrule 
Decrease in test acc.  & $0$ &  $0.7$   & $1.1$ & $21.0$  &  $25.1$ & $16.1$ & $0.0$ \\ 
Compression gains & $1 \times$   &  $3 \times$   & $13 \times$ & $63 \times$  &  $269 \times$ &$175 \times$  & $121 \times$ \\ %
%Predict from &  $\hrv X$  &  $\hrv X$   & $\hrv X$ & $\hrv X$ & $\hrv X$ & $\rv Z$ & $\rv Z$\\
\bottomrule
\end{tabular}
\label{table:distortion_variation}
\end{table}

\paragraphQ{How do our BINCE and  VIC compare to standard compressors}
In \cref{table:distortion_variation} we compare compressors at the lowest downstream error that they achieved.
As benchmark, we use PNG's lossless compression.
Predicting from PNG corresponds to standard image classification, and obtains a rate of $1.42\scip{4}$ bits per image for $80.8 \%$ accuracy.
Classical lossy methods (JPEG, WebP) achieved up to $13 \times$ bit-rate gains with little drop in performance.
In comparison, our BINCE method achieved $121 \times$  compression gains with no impact on predictions.
Both our invariant (VIC) and standard (VC) neural compressors significantly decreased classification accuracy, which we believe can be explained by the encoders architecture (ResNet18) that we use for consistency (see \cref{appx:stl10}).

\paragraphQ{Should we predict from representations $\rv Z$ or reconstructions $\rv X$}
%Downstream performance of VIC and VC in \cref{table:distortion_variation} is measured by training a ResNet18 to predict from reconstructions. 
%In comparison performance of BINCE is measured by predicting using an MLP from $\rv Z$ as there are no reconstructions.
In \cref{table:distortion_variation} we analyzed the impact of predicting from $\rv Z$ instead of $\Tilde{\rv X}$ for VIC and see that this increases accuracy by $9\%$. 
In contrast, predicting from $\rv Z$ for VC decreases performance by $12\%$ (see \cref{appx:stl10}).
This suggests that invariant reconstructions $\Tilde{\rv X}$ might not be easy to predict from with standard image predictors.

%\ydnote{I think that we should remove this one if no space.We don't have any good way of summarizing the results from RD curves, and no space for the curves }
\paragraphQ{Are we learning invariant compressors}
Invariant compressors should provide RD curves that are robust to test distribution shift in the desired augmentations.
We thus trained our VIC by applying the augmentations $50\%$ of the time but varying that probability $p$ at test time.
In \cref{appx:stl10} we show that this distribution shift have negligible influence on RD curves.%, \eg, the AURD is $4.70\scip{3}$ bits for $p=0.2$ against $4.65\scip{3}$ for $p=1.0$.
%To ensure that our invariance assumption still holds we augment both train and test but used a different distributions over augmentations.
%In particular, we train our VIC by applying the augmentations $50\%$ of the time but varying that probability $p$ at test time.
%We hypothesize that this should not influence RD curves since our representations should be invariant to augmentation.
%We confirm this in \cref{appx:stl10} by showing negligible difference across test distribution shift.
%, \eg, the AURD is $4.70\scip{3}$ bits for $p=0.2$ against $4.65\scip{3}$ bits for $p=1.0$.

% \paragraphQ{What is the effect of wrongly assuming invariance}
% In practice, we might assume invariances that do not hold, which should lead to erasing task relevant information.
% We investigated this scenario by making our augmentation  have an invariant range and an non-invariant range for downstream usage of the dataset.
% When training our VIC compressor augmentations sampled from both ranges are treated in the same way.
% During downstream usage of the dataset we change the label with a given probability $p$.
% In \toappx{} we show that best classification error is reduced by exactly $p$.

\subsection{A zero-shot compressor using pre-trained self-supervised models}
\label{sec:clip_experiments}

% Our BINCE objective is a standard contrastive SSL loss with an additional entropy bottleneck.
BINCE includes a standard contrastive SSL loss. So, we investigated whether existing pre-trained SSL models \cite{chen_simple_2020,radford_learning_2021} can be used to build generic compressors. 
In particular, we investigated whether CLIP \cite{radford_learning_2021} could be quickly turned into a powerful task-centric compressor for computer vision.
In the introduction, we motivated large compression gains by noting that typical image classification tasks can be predicted from detailed captions instead of images (around $1000\times$ more bits). 
CLIP is a vision transformer \cite{dosovitskiy_image_2020} pre-trained on 400M pairs of images and text $(x_{image}, x^+_{text})$ using a contrastive loss.
The ``augmentation'' $A$ is then a function that maps $x_{image}$ to its associated $x^+_{text}$ and vis-versa. This will partition the images and texts into sets, each of which are associated directly or by transitivity in CLIP's dataset.
This suggests that CLIP is retaining the image information that corresponds to a detailed caption, and may be turned into a generic compressor for image classification.
% To address this question we chose a model that was pre-trained with augmentations that discard a large amount of unnecessary information.
%Indeed, humans communicate using language instead of vectors so most tasks oriented to human goals should be describable with text.
%We thus used a model which learned text-to-images invariances. %, hypothesizing that it would remove large amount of unnecessary information.
%Specifically,

%To answer this question we built upon CLIP to
%To answer this question we first had to decide which pretrained model to use, which depends on the choice of augmentations that were used during training.
%Indeed, although all models are trained with data augmentations, some might discard more unnecessary information than others.
%The challenge is not about finding augmentations that tasks are invariant to (data augmentations are standard in ML, \eg, cropping) but to find transformations that retain very little unnecessary information.
% In the introduction, we motivated possible large compression gains ($1000\times$ ) using classification of images (around a MB to store) which could instead be performed from a detailed caption  (around a KB to store).
% We thus used a model which learns invariances from text-to-images, hypothesizing that it would remove large amount of unnecessary information.
% Specifically, we built upon CLIP \cite{radford_learning_2021}, which was pretrained using a contrastive loss between a large amount (400 million pairs) of associated images and text.

CLIP can essentially be seen as a BINCE model with an image-to-text augmentation, but without an entropy bottleneck.
(For details about the CLIP-BINCE relation see \cref{appx:clip_bince}.)
We thus constructed an approximation of our desired image-to-text BINCE compressor by two simple steps.
First, we downloaded and froze CLIP's parameters. %, assuming that the encoder learned to retain invariant information during pre-training.
Second, we trained, on the small MSCOCO dataset \cite{lin_microsoft_2015}, an entropy bottleneck to compress CLIP's representation.
The latter step can be done by training any lossy compressor on CLIP's representations, we did so using \citepos{balle_variational_2018} hyperprior entropy model with a learned rounded precision.
%The second stage was trained on the small MSCOCO dataset \cite{lin_microsoft_2015}. %, but found that the choice of dataset had little impact.
We then evaluated our resulting compressor on 8 datasets (various classification tasks and image shapes) that were never seen during training (zero-shot), by training an MLP for downstream predictions on each dataset.
% We note that zero-shot compression can still be analysed using our framework. 
% Indeed CLIP was trained on $400M$ sampled from a r.v. $\rv X$ over images on the internet.
% As these datasets are on internet, they are samples from the joint $(\rv X, \rv Y)$ for a specific task $\rv Y$.
% One can see this as a multi-task setting (each dataset is a distinct task), but we investigate this more directly in \cref{appx:galaxy} using the GalaxyZoo dataset \cite{lintott_galaxy_2008} and its 37 tasks.
One can see this as a multi-task setting (each dataset is a distinct task). We investigate the case of multiple labels per images in \cref{appx:galaxy}.% using the GalaxyZoo dataset \cite{lintott_galaxy_2008} and its 37 tasks.

\begin{table}[h]
\vspace{-1\baselineskip}
\caption{
Converting a pretrained SSL model into a zero-shot compressor achieves substantial bit-rate gains while allowing test accuracies similar to supervised models predicting from raw images. 
% (row 1) bit-rate gains compared to JPEG; 
% (row 2) estimated loss of task-required information;
% (row 3) accuracy of an MLP predicting from our representation.
% (row 4) accuracy of a near SOTA supervised model predicting from uncompressed images.
}
\small
\center
\begin{tabular}{lrrrrrrrrr}
\toprule
& ImageNet  & STL & PCam & Cars & CIFAR10 & Food       & Pets & Caltech  \\ % Food101  / cars / PCAM / Galaxy zoo 2 / text / SUN397 / DTD / oxford pets /caltech 101 / flowers 102 / mnist
% different tasks on galaxy
\midrule 
% JPEG Bit-Rate & 1.39e6  & 2.1e5  & 9.2e4 & 8.1e5 &  2.40e4 & 3.4e5 \\
% Lossyless Bit-Rate & 1350   & 1050 &  1450 & 1350 & 1350 &1350  \\
Rate gains vs JPEG &  $1104\times$  & $35\times$ &  $64\times$  &  $131\times$ & $7\times$ & $109\times$ & $150\times$ & $126\times$  \\ % cifar 100 & $14\times$
\midrule 
Our Acc. $[\%]$ &  $76.3$  &  $98.7$ & $80.9$ &  $79.6$ & $95.2$ &  $88.3$ & $89.5$ & $93.4$  \\
Supervised Acc. $[\%]$ & $76.1$   &  $99.0$ &  $82.6$  & $49.1$ &  $96.7$ &  $81.8$ & $90.4$ & $94.5$ \\ 
\bottomrule
\end{tabular}
\label{table:clip}
\end{table}

\paragraphQ{Can we use pretrained SSL to obtain a generic compressor}
\cref{table:clip} shows that we can exploit existing state-of-the-art (SOTA) SSL models to get a powerful image compressor, which achieves $1000\times$ bit-rate gains on ImageNet compared to JPEG (at the quality level used for storing ImageNet).
The bit-rate gains (1\textsuperscript{st} row) are significant across all zero-shot datasets, even for biological tissues (PCam; \cite{veeling_rotation_2018}).
Importantly, these gains come at little cost in test performance.
Indeed, the test accuracies of MLPs from our representations (2\textsuperscript{nd} row) is similar to a near SOTA model trained on the uncompressed images (3\textsuperscript{rd} row is from \citet{radford_learning_2021}).
These results are not surprising as JPEG is optimized to retain perceptual rather than classification information.
Note that the large variance in rate gains come from JPEG rates due to different images shapes (see \cref{table:clip_vs_EB}).

\paragraph{Our CLIP compressor retains all the information needed to get 0 error for those tasks}
\Cref{table:clip} provides the test performance for MLPs, while our theory discusses Bayes risk, which is independent of specific predictors and generalization. 
We estimated the excess Bayes risk for our datasets by counting the images (in train and test) that get compressed to the same $\rv Z$ but have different labels. 
We found that we are in the lossless prediction regime for those datasets.

\begin{table}[h]
\vspace{-1\baselineskip}
\caption{
Our entropy bottleneck (EB) on CLIP improves compression of representations up to $17\times$ with little impact on predictions.
The same compressor is used across datasets. Rates are per image.
%Baselines: JPEG and compression of features from a ImageNet pretrained classifier (Transfer + EB).
}
\small{}
\center
\begin{tabular}{llrrrrrrrr}
\toprule
 & &  ImageNet  & STL & PCam & Cars & CIFAR10  & Food      & Pets & Caltech  \\ 
\midrule 
\multirow{5}{*}{\rotatebox[origin=c]{90}{\centering ~Bit-Rate  }} 
& JPEG & 1.49e6  & 4.71e4 & 9.60e4 & 1.92e5 & 1.05e4  & 1.54e5     & 1.81e5 & 1.69e5  \\ 
%&  Transfer + EB &   & 3.33e3 &3.99e3  & 3.18e3 & 3.92e3    & 3.26e3     & 3.70e3 & 3.40e3  \\ 
 & CLIP & 1.52e4  & 1.52e4 & 1.52e4 & 1.52e4 & 1.52e4    & 1.52e4     & 1.52e4 & 1.52e4  \\ 
 & \ \  \textbf{+EB high} $\beta$ & 2.47e3  & 2.46e3 & 2.61e3 & 2.59e3 & 2.53e3  & 2.39e3      & 2.33e3 & 2.46e3  \\ 
 & \ \  \textbf{+EB} $\beta$   & 1.35e3  & 1.34e3 & 1.49e3 & 1.47e3 & 1.41e3 & 1.27e3      & 1.21e3 & 1.34e3  \\ 
 & \ \  \textbf{+EB low} $\beta$ & 9.63e2  & 9.52e2 & 1.09e3 & 1.07e3 & 1.02e3 & 8.89e2      & 8.35e2 & 9.53e2  \\ % to double check imagenet
  \midrule 
\multirow{4}{*}{\rotatebox[origin=c]{90}{\centering ~Test Acc. }}  
%& Transfer + EB  &  & 96.1 & 79.4 & 42.0 & 87.0  & 66.8    & 91.3 &  89.9  \\  
& CLIP  & 76.5  & 98.6 & 84.5 & 80.8 & 95.3  & 88.5      & 89.7 &  93.2  \\  
 & \ \  \textbf{+EB high} $\beta$ & 76.6  & 98.7 & 82.7 & 80.4 & 95.3  & 88.5      & 89.6 & 93.5  \\  
 & \ \  \textbf{+EB} $\beta$ & 76.3  & 98.7 & 80.9 & 79.6 & 95.2  & 88.3    & 89.5 & 93.4  \\
 & \ \  \textbf{+EB low} $\beta$ & 76.0  & 98.7 & 80.1 & 78.9 & 94.8 &  87.6      & 88.6 & 92.9  \\ % to double check imagenet 
\bottomrule
\end{tabular}
\label{table:clip_vs_EB}
\end{table}

\paragraphQ{What is the effect of the entropy bottleneck}
In \cref{table:clip_vs_EB} we compare the pretrained CLIP, to our CLIP compressor with an entropy bottleneck (EB) trained at different values for $\beta$.
When trained with a high $\beta$, our EB improves bit-rates by an average of $6\times$ without impacting predictions.
For our compressor from \cref{table:clip} (CLIP+EB $\beta$) the gains increase to $11\times$ with little predictive impact.
The sacrifice in predictions is more clear for $16\times$ bit-rate gains (low $\beta$).
This shows that CLIP's raw representations retain unnecessary information as it not explicitly trained  to discard information.

\paragraphQ{How would end-to-end BINCE compare to staggered training}
Compression gains can likely be larger by end-to-end training of BINCE, which would require access to CLIP's original dataset.%
\footnote{We investigated finetuning CLIP on MSCOCO but it suffered from catastrophic forgetting.}
To get an idea of potential gains we compared end-to-end and staggered BINCE on augmented MNIST in \cref{appx:mnist}. We found significant rate improvements ($358$ to $131$ bits) for similar test accuracy. 

\paragraph{Our CLIP compressor is simple to use}
% The main advantage of our compressor is its generality and simplicity, which makes it usable by anyone.
In \cref{appx:code_clip}, we provide a minimal script (150 lines) to train a generic compressor in less than five minutes on a single GPU.
The script contains an efficient entropy coder for our model (200 images/second), which shows its practicality. As usual in SSL, the compressed representations are also more computationally efficient to work with than standard compressors.
In our minimal script we achieve the desired performance ($98.7\%$ on STL) using a linear model that is trained in one second, which is $1000\times$ faster than the baseline in \cref{table:clip}. %to train.
%Simple models are also much quicker at inference time.
This shows that our pipeline can improve computational efficiency in addition to storage efficiency.

\begin{table}[h]
\vspace{-1\baselineskip}
\caption{
Text-image invariance is better than invariance to standard augmentations for image classification.
CLIP and SimCLR are both ResNet50 pretrained with InfoNCE but different augmentations.
}
\small
\center
\begin{tabular}{llrrrrrrrrr}
\toprule
&& ImageNet  & STL & PCam & Cars & CIFAR10 & Food       & Pets & Caltech  \\ 
\midrule 
\multirow{2}{*}{\rotatebox[origin=c]{90}{\centering \scriptsize ~Rate  }} 
& CLIP+EB   &  $2108$  & $1962$ &  $1949$  &  $2421$ & $2111$ & $1991$ & $1867$ & $1968$  \\ 
& SimCLR+EB   &  $2811$  & $2732$ &  $2769$  &  $2751$ & $2950$ & $2077$ & $2839$ & $2502$  \\  
\midrule 
\multirow{2}{*}{\rotatebox[origin=c]{90}{\centering \scriptsize ~Acc.  }} 
& CLIP+EB   &  $63.2$  &  $92.0$ & $78.6$ &  $68.0$ & $65.5$ &  $74.1$ & $81.8$ & $83.0$  \\
& SimCLR+EB   & $62.8$   &  $91.9$ &  $81.4$  & $29.6$ &  $78.6$ &  $60.0$ & $78.9$ & $79.0$ \\ 
\bottomrule
\end{tabular}
\label{table:ssl}
\end{table}

\paragraphQ{What augmentations to use for SSL compression}
\Cref{table:ssl} compares two ResNet50 pretrained with contrastive learning using invariance to text-image (CLIP) or standard image augmentations (SimCLR \cite{chen_simple_2020}) such as cropping or flipping. 
We see that CLIP's augmentation usually give better compression and downstream performance, which shows the importance of the choice of augmentations.
This also supports our motivation of using text-image augmentations, which are likely label-preserving for a vast amount of tasks but discard large amounts of unnecessary information.

\section{Related work}
\label{sec:related}
In \cref{appx:related} we discuss more related work, including invariances in compression and the link to SSL.
%In \cref{appx:related} we discuss more related work., including the link between our objectives and others in SSL, and how maximal invariants are related to sufficient statistics.

\paragraph{Task-centric compression}
To our knowledge, our paper is the first to formalize compression only for predictions.
IB \cite{tishby_information_2000} uses a task-centric distortion, but is not used for compression as it requires supervised training, so there are no advantages compared to compressing predicted labels.
%Our paper can be seen as a theoretically justified self-supervised IB that does not require the labels at compression time.
Some authors used heuristics to bypass the supervised issue, \eg, focusing on low frequencies for classification \cite{liu_deepn-jpeg_2018} or high frequencies for segmentation \cite{liu_machine_2019}.
Other authors have incorporated predictive errors to perceptual distortions \cite{liu_recognizable_2016,liu_classification-distortion-perception_2019}, but cannot compress without the perceptual distortion for the same reason as IB. 
One exception is \citepos{weber_observer_2020} compressor, which (when removing their perceptual distortion) minimizes MSE in the hidden layers of a pretrained classifier.
%We overcome this issue by acknowledging that there are multiple downstream tasks of interest, and show how to train the compressor in an unsupervised way using invariances.
Even more related is \citepos{singh_end--end_2020} work on compressing pretrained features for transfer learning, which is practice is similar to our compression of SSL features.
Their work do not provide theoretical justifications, and are constrained to tasks that are similar to those used for pretraining.

\section{Discussion and Outlook}
\label{sec:conclusion}

Given the ever increasing amount of data that is processed by task-specific algorithms, it is necessary to rethink the current task-agnostic compression paradigm.
We formalized the first compression framework for retaining only the information necessary for high performance on desired tasks. 
Using our theory, we provide two unsupervised objectives for training neural compressors.
Experimentally, we show that these compressors can achieve  bit-rates that are orders of magnitude ($1000\times$ on ImageNet) smaller than standard image compressors without losing predictive performance.

There are a number of caveats that should be addressed. 
First, to achieve better rates, our theory requires an irrecoverable loss of information. This can be an issue if the set of desired tasks changes.
For example, if one uses text-image invariances then it may be impossible to perform image segmentation from the compressed representations.
One solution would be to keep an original copy and use invariant compression for duplicated data, \eg, for the thousands copies of ImageNet.
A second issue is the interpretability of the compressed representations.
%Even if, as humans, we will never visualize all collected data, standard compressors give us the choice.
This can be partially addressed by reconstructing prototypical data as in \cref{fig:mnist_intro} (post-hoc decoders could be trained for BINCE).
A third caveat is that the compressed representations may be harder to learn from, \eg, neural networks may struggle to predict from representations even if the information is retained.
Although our experiments actually showed the opposite, this should be addressed theoretically, \eg, using decodable information \cite{xu_theory_2020,dubois_learning_2020}.
Finally, successful use of our framework requires access to label-preserving augmentations $A$ that discard significant information about $\rv X$. 
Finding such an $A$ may be challenging for some tasks.
Given that augmentations are ubiquitous in ML, the community will hopefully continue developing task-specific augmentations which we could take advantage of.

% Nevertheless, we found that order of magnitude improvements in compression rates are possible in realistic prediction settings, and we believe that our improvements are just the beginning.
Nevertheless, we achieved orders of magnitude improvements in compression for predictions, and we believe that our improvements are just the beginning.
For example, many tasks can be answered by referencing a detailed natural language description of the data. In these cases, the improvements can be very large, potentially 1M$\times$ for videos.\footnote{A movie takes around 10GB to store, but the information relevant to humans (\eg, ``what happened to the house?'', ``how did the movie end?'') can likely be stored in a detailed movie script and would require 100KB.}
%Yet, even 1M$\times$ might be a conservative estimate. 
%It is estimated that adults store less than a GB of memory \cite{landauer_how_1986} and are still able to perform extremely complicated tasks. 
%To reach such compression and move towards a society that can sustainably take advantage of large data sources we should abandon the idea of having to faithfully reconstruct all stored data.
%In the long-term, abandoning the idea of faithful reconstructions can help us  truly massive datasets accessible to many.
In the long-term, we hope that abandoning perceptual reconstructions will enable individuals to process data at scales that are currently only possible at large institutions, and our society to take advantage of large data sources in a more sustainable way.

\clearpage
\newpage

\begin{ack}
We would like to thank Alex Alemi, David Duvenaud, Andriy Mnih, Emile Mathieu, Jonah Philion, Yangjun Ruan, and Ilya Sutskever for their helpful feedback and encouragements.
Resources used in preparing this research were provided, in part, by the Province of Ontario, the Government of Canada through CIFAR, and companies sponsoring the Vector Institute.
BBR acknowledges the support of the Natural Sciences and Engineering Research Council of Canada (NSERC): RGPIN-2020-04995, RGPAS-2020-00095, DGECR-2020-00343.
\end{ack}

\bibliographystyle{IEEEtranN}
\bibliography{bibliography}

\clearpage
\newpage

%%%%%%%%%%%%%%%%%%%%%%%%%%%%%%%%%%%%%%%%%%%%%%%%%%%%%%%%%%%%

\appendix
\addcontentsline{toc}{section}{Appendix} % Add the appendix text to the document TOC
\part{Appendix} % Start the appendix part

\parttoc % Insert the appendix TOC

%\tableofcontents

\clearpage
\newpage

\section{Preliminaries}
\label{appx:preliminaries}
%!TEX root = ../submission_supplement.tex

%%%%%%%%%%%%%%%%%%%%%%%%%%%%%%%%%%%%%%%%%%%%%%%%%%%%%%%%%%
\subsection{Notation}
\label{appx:notation}

% \ydnote{For notational convenience I'm assuming everywhere that pdf exists $p(\rv X)$. Should we instead use notation directly with distribution $P_X$? The problem if you don't do that is that you have to be careful about how you define continuous entropy. Differential entropy typicall does not work so you have to work with $H[X] \defeq I[X,X]$ which is well defined and can be use with all our proofs but it is the limit of discrete entropy but is not differential entropy and thus not how most people think about it.
% }

\kword{Probability} 
We assume a background standard probability space $(\Omega, \mathcal{H}, \mathbb{P})$ that is rich enough to support all random variables used. 
Letters that are upper-case $\rv X$ represent a random variable, while realizations are denoted with the associated lower case $x$. 
The sample space of a random variable will be written using a calligraphic $\mathcal{X}$, and we will say that $\rv X$ takes value in (t.v.i) $\mathcal{X}$.
We denote the probability distribution of $X$ as $P(X)$ and the probability density function, if it exists, as $p(\rv X)$.
$\rv X \dsim \mathcal{N}(0,1)$ denotes that $\rv X$ has a certain distribution (here, Gaussian).
%\footnote{For clarity and conciseness, in the main paper we use the probability density function notation, despite not making any assumption about its existence in the appendices.}
%We denote the set of all probability distributions on $\mathcal{X}$ as $\mathcal{P}(\mathcal{X})$.
Expectations are written as: $\E{P(\rv x)}{\rv X}$, % \defeq \int x p(x) \sd x$, 
or $\E{p(\rv x)}{\rv X}$ when the density exists.
%\ydnote{In main text we use $\E{p(\rv x)}{\rv X}$, maybe we can say that acceptable notation if density exists ?}
%while their Monte Carlo approximations have a hat $\hE{P(\rv x)}{\rv X} \defeq \frac{1}{|\set{x_i}|}\sum_i x_i$ for independent samples $\set{x_i}$ from $p(\rv X)$. 
Independence between two random variables $\rv X$ and $\rv Y$ is denoted with $\rv X \indep \rv Y$.
To denote conditional independence between two random variables $\rv X$ and $\rv Y$ given $\rv Z$ we either use $\rv X  \indep \rv Y \cond \rv Z$ or say that $\rv X - \rv Z - \rv Y$ forms a Markov Chain. 
$f \circ g$ denotes a composition of two functions $f$ and $g$, but in the case of random variable we also use the shorthand $f(\rv X) \defeq f \circ \rv X$.

\kword{Information theory} 
For notational convenience (see \cref{assumption:density}  below), when dealing with log loss we will always assume the existence of probability densities, in which case the KL divergence between two probability distributions on $\mathcal{X}$, $P$ and $Q$, is $\KL{p(\rv X), q(\rv X)} \defeq \int p(\rv X) \log  \frac{p(\rv X)}{q(\rv X)} \sd x$.
The mutual information between random variables $X$ and $Z$ is $\MI{X}{Z} \defeq \KL{{p(\rv X, \rv Z)}, p(\rv X)p(\rv Z)}$.
The (differential or discrete) entropy of a random variable is $\op{H}{\rv X} \defeq \E{P(\rv x)}{- \log p(\rv X)} $, while the conditional (differential) entropy is $\op{H}{\rv X \cond \rv Z} \defeq \E{P(\rv x,\rv Z)}{- \log p(\rv X| \rv Z)}$.

%The cardinality of a set is denoted by $|\cdot|$.
%$\set{\cdot}$ denotes a set.

\kword{Equivalence} 
$x \sim x'$ denotes that $x$ and $x'$ are equivalent with respect to (w.r.t.) an equivalence relation on $\mathcal{X}$ (the exact relation being implicit).
The equivalence class of $x$ under $\sim$ consist of all elements that are equivalent to $x$, \ie $[x] \defeq \set{x' \in \mathcal{X} \cond x' \sim x}$.
The set of all equivalence classes (the quotient set) will be denoted as $\mathcal{X} / \sim \, \defeq \set{[x] \cond x \in \mathcal{X}}$, while the canonical projection is denoted as $\pi_{\sim} : x \mapsto [x]$.

\kword{Risk minimization} 
We will often use variational optimization. When the variational family is not made explicit it means that the optimization is over all functions with the correct domain and codomain, \eg  $\min_{q(\rv Y \cond \rv X)}$ means that that the optimization is done over the collection of all conditional probability densities  on $\mathcal{Y}$ given the random variable $X$. 

For a fixed ``action'' or ``decision'' space $\mathcal{A}$, a loss function is defined as $L : \mathcal{Y} \times \mathcal{A} \to \mathbb{R}$. The (expected) risk of a predictor $h \colon \mathcal{X} \to \mathcal{A}$ is $\E{P(X,Y)}{L(Y, h(X))}$. 
The Bayes (best achievable) risk when predicting $\rv Y$ from $\rv X$ using some (unspecified) loss is denoted as $\Risk{Y}{X}$. 
When the loss $L$ is specified, we denote the Bayes risk as $\Riskl{Y}{X} \defeq \inf_{h\colon \mathcal{X} \to \mathcal{A}} \E{P(\rv X, \rv Y)}{L(Y,h(X))}$.
For the case of log loss (always assumed in the main text) we have  $\Risklog{Y}{X} \defeq \inf_{q} \E{P(\rv X, \rv Y)}{- \log q(\rv y|\rv x)}$.
For MSE loss we have $\Riskmse{Y}{X} \defeq \inf_{f \colon \mathcal{X} \to \mathcal{Y}} \E{P(\rv X, \rv Y)}{\| Y - f(X) \|^2}$.
Letters $\rv X$, $\rv Z$, and $\rv Y$ refer to the input, representation and target of a predictive task, respectively.

% Expectations will be written as: $\op{E_{p(\rv x)}}{\rv X} \defeq \int_{\mathcal{X}} x \sd P_{\rv X}$. 
% The KL divergence will denoted as $\KL{p(\rv X), q(\rv X)} \defeq \int_{\mathcal{X}} \log \frac{\sd P_{\rv X}}{\sd Q_{\rv X}} \sd P_{\rv X}$.
% The mutual information $\MI{X}{Z} \defeq \KL{P_{\rv X, \rv Z}, P_{\rv X} \bigotimes P_{\rv Z}}$.
% The entropy of a random variable is $\op{H}{\rv X} \defeq \MI{X}{X}$, while the conditional entropy is $\op{H}{\rv X} \defeq \MI{X}{X}$

%%%%%%%%%%%%%%%%%%%%%%%%%%%%%%%%%%%%%%%%%%%%%%%%%%%%%%%%%%
\subsection{Assumptions}
\label{appx:assumptions}

In this section, we discuss the assumptions that we make throughout our paper. 
Specifically, we discuss why we make those assumption and why such assumptions should hold in practice.
\textbf{All our assumptions should hold in most practical scenarios.}
The following assumptions will be implicit in the rest of our work.

\ydnote{We should remove any notion of M(X) in this assumption and have at as a lemma that comes from the fact that X satisfies this assumption}
\begin{assumption}[Finite risk]
\label{assumption:variance}
% \ydnote{Use constant x instead of all Z}
We restrict ourselves to tasks $\rv Y$, such that\  $|\Risk{Y}{\eta}| < \infty$ for any finite constant $\eta$ in the domain of the predictor $f$.
Similarly we restrict ourselves to $\rv X$ with $\Risk{\rv X}{\eta} < \infty$, and to equivalences relations on $\calX$ such if there exists a maximal invariant then there exists \textit{some} maximal invariant $M(\rv X)$ with $\Risk{\rv M(X)}{\eta} < \infty$ for any finite constant $\eta$.
%This ensures that the Bayes risk is always bounded, \ie $\forall Z$ we have $|\Risk{Y}{Z}| < \infty$ and $|\Risk{M(X)}{Z}| < \infty$, both for log loss 
%\footnote{The entropy for discrete and continuous random variable is always upper bounded by $\frac{1}{2} \log (2 \pi e (\op{Var}{\rv Y}) + \frac{1}{12} )$ \cite{massey_entropy_1988})}  
%and MSE loss.
% We could also directly assume that the difference of Bayes risk is well defined, but that would require dealing with limits which would unnecessarily complicate the proofs.
%Note that there will always exist an $M(\rv X)$ with finite variance if $\rv X$ has finite variance.
%\ydnote{\ben{} we weren't sure last time: I'm pretty sure that this implies that there exists at least one $M(X) \in L^2$. Simply take the smallest canonical $M(X)$, i.e., $M: x \mapsto min_{x' \in [x]} x'$ and this has to have a smaller expected square so it must be in L2.}
\end{assumption}

\Cref{assumption:variance} ensures that we can take differences of Bayes risks as in \cref{def:excess_distortion}.
For the case of log loss our assumption is equivalent to requiring finite (differential or discrete) entropy of $\H{X}$, $\H{M(X)}$ and $\H{Y}$.
For MSE loss, this is equivalent to finite variance for $X$, $Y$ and $M(X)$. Specifically, we will restrict ourselves to random variables $Y$ and $M(X)$ that are bounded in $L^2(\Omega, \mathcal{H}, \mathbb{P})$. (See \cref{assumption:l2:bounded} in \cref{appx:theorem_mse}.)
Note that $\Risk{\rv M(X)}{\eta} < \infty$ comes directly from $\Risk{\rv X}{\eta} < \infty$ for the two main losses that we consider.
Indeed, for log loss this comes directly from the data processing inequality.
For MSE such $M(X)$ can easily be constructed by mapping any $x$ to a value in $[x]$ that is smaller than the expected value over the equivalence class $\op{E}{X | X \in [x]}$.
\ydnote{Should probably be a lemma. The problem is that I don't know how to prove it for \textit{any} loss which would be needed for \cref{appx:theorem_lossless}.}

\begin{assumption}[Existence of regular conditional probabilities] 
\label{assumption:regular} 
We restrict ourselves to standard Borel measurable spaces, so that the existence of regular conditional probability distributions is ensured. 
This is necessary to ensure the existence of probability kernel in \cref{lemma:desintegration}. 
This technical assumption essentially holds for all practical purposes. Unless stated otherwise, we denote $\borel(\mathcal{Y})$ the Borel $\sigma$-algebra of a set $\mathcal{Y}$.
\end{assumption}

\begin{assumption}[Measurability of functions] \label{assumption:measurability} 
We assume that all functions introduced in the following sections are measurable with respect to the ``natural'' measurable spaces of the functions' domain and codomain. (A few special functions will be shown to be measurable.) 
In particular, we require 
\begin{inlinelist}
\item the measurability of $M(\cdot)$ which implies that $M(\rv X)$ is a random variable; and
\item the measurability of the projection $\pi_{\sim}: \mathcal{X} \to \mathcal{X}/\sim$, which implies that there always exists a maximal invariant in the form of the projection $\pi_{\sim}$. 
\end{inlinelist}
This technical assumption holds for essentially all practical purposes.
\end{assumption}

\Cref{assumption:variance,assumption:measurability,assumption:regular} are used throughout our work.
Two further assumptions are needed for log loss, which we remove in \cref{appx:theorem_mse} when we obtain results for MSE.

\begin{assumption}[Countably many equivalence classes] 
\label{assumption:discrete} 
For the log loss risk (\cref{appx:theorem_logloss}) we restrict our discussion to equivalences  such that the quotient set $\mathcal{X}/\sim$ is countable.
This ensures that $M(\rv X)$ is a discrete random variable thereby ensuring that our invariance distortion $\Risk{M(X)}{Z}$ is independent of the choice of maximal invariant $M$ as the conditional entropy is  invariant to bijections.
\end{assumption}

Note that \cref{assumption:discrete}  holds when $\mathcal X$ is countable which always happens in practice due to floating point arithmetic, \ie every real number has to be rounded to the closest 64 bits number.
Another perspective is to say that $\mathcal{X}$ is actually uncountable, but that all tasks we care about are always invariant to rounding to the nearest 64 bits number due to floating point arithmetic.
As a result, the maximal invariant is the usual maximal invariant rounded to the closest floating point. 
For example, if $\rv X$ is a 2D Gaussian we cannot work directly with translations on the y-axis (which gives uncountably many $[x]$, one for each real number on the x-axis), but can work with y-axis invariance combined with invariance to rounding on the x-axis (e.g. closest 64 bits number).

\begin{assumption}[Convenience assumption: Existence of densities]\label{assumption:density} 
In sections \cref{appx:invariant_distortion,appx:theorem_logloss}, where we work with log loss, we restrict ourselves to cases where the (conditional) probability mass/density function exist, \ie, to probability distributions that are absolutely continuous w.r.t.\ to some (shared) underlying measure.
This assumption is not needed but it simplifies the notation, and ensures that the differential entropy of random variables is well defined.
Such assumption could be removed by using the general definition of mutual information as a supremum over partitions and by defining continuous entropy as $\H{X} = \MI{X}{X}$ \cite{kolmogorov_shannon_1956,pinsker_information_1964}, also known as \citepos{jaynes_information_1957} limiting density of discrete points.
%Despite often using it in the notation we do not explicitely state that assumption unless we use it.
\end{assumption}

%%%%%%%%%%%%%%%%%%%%%%%%%%%%%%%%%%%%%%%%%%%%%%%%%%%%%%%%%%
\subsection{Definitions}
\label{appx:definitions}

In the main paper we were relatively informal in our definitions, here we restate our main definitions more formally.

\begin{definition}[Maximal invariant]\label{def:maximal_invariant}
Let $\sim$ denote an equivalence relation on $\mathcal{X}$.
We say that a measurable function $M: \mathcal{X} \to \mathcal{M}$ is a \textit{maximal invariant} w.r.t.\  $(\mathcal{X}, \sim)$ if
\begin{equation}\label{eq:max_inv}
\forall x,x' \in \mathcal{X} \quad  x \sim x' \iff M(x) = M(x')
\end{equation}
\end{definition}

Note that our notion of maximal invariants generalizes the notion of maximal invariants in probabilistic group theory \cite{eaton_group_1989}. 
We refer the reader to \citet{lehmann_testing_2005} for many examples in the group case. As in the group case, a maximal invariant typically is not unique.

The invariance structure that we want our tasks to have is based on their conditional distributions given $X$, defined as follows.

\begin{definition}[Conditional invariance]\label{def:cond_invariance}
We say that $\rv Y$ is conditionally invariant w.r.t.\  $(\mathcal X, \sim)$, if the regular conditional distribution $x \mapsto P(\rv Y \cond x)$ is invariant w.r.t.\ $\sim$, \ie $\forall x, x' \in \mathcal{X}$ we have 
\begin{equation}\label{appx:eqn:task_invariance}
x \sim x' \implies P(\rv Y \cond x) = P(\rv Y \cond x')
\end{equation}
\end{definition}
\ydnote{use notation from lemma 4.}

\begin{definition}[Invariant tasks of interest]\label{def:invariant_tasks_interest}
The set of all invariant tasks of interest $\tasksinv{}$ w.r.t.\ to a loss and an  equivalence $(\mathcal X, \sim)$ is the set of all random variables $\rv Y$ that are conditionally invariant w.r.t.\ $(\mathcal X, \sim)$ and that satisfy \cref{assumption:variance} (finite risk).
\end{definition}

First we require the notion of a valid distortion \cite{berger_rate_1968}, which ensures that we can apply the rate distortion theorem.%
% \footnote{
% The RD theorem has been extended to more general cases \cite{pinkston_encoding_1967,gallager_information_1968,berger_rate_1968} but we use the original theorem for clarity. 
% }

\begin{definition}[Valid distortion]\label{def:valid_distortion}
Let $\rv X$ and $\rv Z$ be two random variables that take values in $\mathcal{X}$ and $\mathcal{Z}$, respectively.
Then an (expected) distortion $\mathrm{D}$ is \textit{valid} w.r.t.\ $\rv X,\rv Z$ if there exists a point-wise distortion $d : \mathcal{X} \times \mathcal{Z} \to \mathbb{R}_{\geq 0}$ such that $\E{p(\rv X, \rv Z)}{d(\rv X, z)} \leq \infty$ for some $z \in \mathcal{Z}$ and 
%\footnote{
%\citet{coverelements2006} use the stronger assumption of $\max_{x \in \mathcal{X}, z \in \mathcal{Z}} d(x,z) \leq d_{max} < \infty$.
%But in their proofs they only require the point-wise distortion to be almost surely bounded in order to bound the worst case expected distortion between two sequences (achievability proof of the rate distortion theorem on page 321 and the achievability proof on page 327).
%For a more general proof see \citet{berger_rate_1968}.
%} 
\begin{equation}
\op{D}{\rv X, \rv Z} \defeq \E{p(\rv X, \rv Z)}{d(\rv X, \rv Z)} \;.
\end{equation}
\end{definition}

In the context of the current work, a representation $\rv Z$ which arises by encoding $\rv X$ using $p(\rv Z \cond \rv X)$ should not depend on any particular task $\rv Y$. 

\begin{definition}[Representation for a task set]\label{def:representations}
Let $\rv X, \rv Z$ be two random variables and $\tasks{}$ be a set of random variables.
$\rv Z$ is a \textit{representation} of $\rv X$ for $\tasks{}$ if for all $\rv Y \in \tasks{}$ such that $Y$ and $Z$ are not almost surely equal, we have the pairwise conditional independence $\rv Y \indep \rv Z \cond \rv X$.
\end{definition}

Note that if $Z \notin \tasks{}$ then it is not almost surely equal to any $Y \in \tasks{}$. The condition allows for the possibility that $Z \in \tasks{}$, but it must be conditionally independent, given $X$, of all other random variables in $\tasks{}$.

We now recall the excess risk distortion \disttext{}.

\begin{definition}[Excess risk distortion]\label{def:excess_distortion}
Let $\rv X$ and $\rv Z$ be two random variables.
Let $\tasks{}$ be a set of random variables such that under a loss $L$, the Bayes risks in \eqref{eq:excess:risk:dist} below are well defined for each $Y \in \tasks{}$.
The \textit{excess risk distortion} \disttext{} is defined as:
\begin{equation} \label{eq:excess:risk:dist}
\distst{} \defeq \sup_{\rv Y \in \tasks{}} \quad  \Risk{Y}{Z} - \Risk{Y}{X}
\end{equation}
\end{definition}

\clearpage
\newpage

\section{Proofs: optimal bit-rate}
\label{appx:proofs}
%!TEX root = ../submission_supplement.tex

In this section we prove all results from \cref{sec:theory}. 

%%%%%%%%%%%%%%%%%%%%%%%%%%%%%%%%%%%%%%%%%%%%%%%%%%%%%%%%%%
\subsection{Basic properties of equivalence relations and maximal invariants}
\label{appx:maximal:invariants}

To begin, we collect some basic properties of equivalence relations and maximal invariants.
As these a general result that might be of interest beyond our work (especially \cref{lemma:desintegration}) we will prove them without assuming the existence of densities, \ie, without \cref{assumption:density}.
Recall that $\pi_{\sim} \colon \mathcal{X} \to \mathcal{X}/\sim$ is the projection from $\mathcal{X}$ onto its quotient by $\sim$, denoted $\mathcal{X}/\sim$.

\begin{lemma}[\citet{mac_lane_algebra_1999}, Theorem 19]
\label{lemma:projection:theorem}
Given an equivalence relation $\sim$ on $\mathcal{X}$, let $f \colon \mathcal{X} \to \mathcal{S}$ be any function such that $x\sim x' \Rightarrow f(x) = f(x')$. Then there is exactly one function $g \colon \mathcal{X}/\sim \to \mathcal{S}$ for which $f = g \circ \pi_{\sim}$. If $f$ is a surjection and $f(x) = f(x') \Rightarrow x \sim x'$, then $g$ is a bijection.
\end{lemma}

\begin{lemma}\label{lemma:maxinv_bijection}
Let $M \colon \mathcal{X} \to \mathcal{M}$ and $M' \colon \mathcal{X} \to \mathcal{M'}$ be two maximal invariants w.r.t.\ $(\mathcal{X}, \sim)$. Then there exists a bijective function $f \colon \mathcal{M} \to \mathcal{M}'$ such that $ M' =  f \circ M$.
\end{lemma}
\begin{proof}
From \cref{lemma:projection:theorem}, $M$ is a maximal invariant if and only if there is a bijective function $g \colon \mathcal{X}/\sim  \; \to \mathcal{M}$ such that\ the maximal invariant is the composition of $g$ and the projection onto equivalence classes, \ie $M = g \circ \pi_{\sim}$.
Let $g'$ be the corresponding bijection for $M'$.
Then we have $M' = f \circ  M$ with $f \defeq g'  \circ  g^{-1}$ which is indeed bijective: $f^{-1} \defeq g  \circ  g'^{-1}$.
\end{proof}

\begin{lemma}\label{lemma:maxinv_invariance}
Let $M$ be any maximal invariant w.r.t. $(\mathcal{X}, \sim)$.
Then a measurable function $f \colon \mathcal{X} \to \mathcal{S}$ is invariant with respect to $(\mathcal{X}, \sim)$ if and only if there exists a measurable function $h \colon \mathcal{M} \to \mathcal{S}$ such that $f(x) = (h \circ M)(x)$ for all $x \in \mathcal{X}$, in which case $f$ is measurable with respect to the $\sigma$-algebra generated by $M$.
\end{lemma}
\begin{proof}
Clearly $f = h \circ M $ is $(\mathcal{X}, \sim)$-invariant because $M$ is, and measurability of $f$ follows from measurability of $M$ and $h$.

From \cref{lemma:projection:theorem}, if $f$ is $(\mathcal{X}, \sim)$-invariant then there is a function $s : \mathcal{X}/\sim  \; \to \mathcal{X}$ such that $f = s \circ \pi_{\sim}$. Since $\pi_{\sim}$ and $f$ are measurable, so too is $s$. 
Again by \cref{lemma:projection:theorem}, there exists a bijective mapping $g \colon \mathcal{X}/\sim  \; \to \mathcal{M}$ such that $M = g \circ \pi_{\sim}$.
We thus conclude that $f = h \circ M$, for $h \defeq s \circ g^{-1}$.
The measurability of $h$ follows from the measurability of $s$ and of $g^{-1}$. 
\end{proof}

Finally, we establish a key conditional independence relationship, which shows that for invariant tasks,  $\rv Y - M(\rv X) -\rv X$ forms a Markov Chain. 
% Specifically, we prove that that any conditionally invariant random variable can be decomposed as a function of a maximal invariant and independent noise.
% This can be seen as a probabilistic extension of the theorem on projections (Theorem 19 and its corollary in \citet{lane_algebra_1999}) .
This is a generalization of an probabilistic group theoretical results (Theorem 4.4 in \citet{eaton_group_1989}, Theorem 7 in \citet{bloem-reddy_probabilistic_2020}), to any equivalences (rather than only group orbits) and without making the assumption of (marginal) invariance of $P(\rv X)$ to $\rv \sim$.

% \ydnote{\ben{} I changed a little the proof to a way that went a little slower / lengthier for me to be able to follow. Can you double check ?} 
% \bbnote{To double-check.}
\begin{lemma}\label{lemma:desintegration}
Let $\rv X$ and $\rv Y$ be two random variables, and $M: \mathcal{X} \to \mathcal{M}$ be a maximal invariant w.r.t. $(\mathcal{X}, \sim)$ as in \cref{def:maximal_invariant}.
Then $\rv Y$ is conditionally invariant w.r.t. $(\mathcal{X}, \sim)$ as in \cref{def:cond_invariance} if and only if $\rv Y \indep \rv X \cond M(\rv X)$.
% there exists a measurable function $f : [0,1] \times  \mathcal{M} \to \mathcal{Y}$ such that:
% \begin{equation}
%     \rv Y = f(\eta, M(\rv X))  \text{ with } \eta \sim \mathrm{Unif}[0,1] \text{ and } \eta \cond \rv X
% \end{equation}
\end{lemma}
\begin{proof}
Let $\rv Y$ be a $\mathcal{Y}$-valued random variable that is conditionally invariant w.r.t. $(\mathcal{X}, \sim)$. Recall that $\borel(\calY)$ is the Borel $\sigma$-algebra of $\calY$. 
By \cref{assumption:regular}, there exists a probability kernel $\kappa_{\rv Y}(A,x)$ from $(\calX,\borel(\calX))$ into $(\calY,\borel(\calY))$, such that for each set  $A \in \borel(\calY)$, $x \mapsto \kappa_{\rv Y}(A,x)$ is a measurable function mapping $\mathcal{X} \to \mathbb{R}_{\geq 0}$.

Conditional invariance means that $x \sim x' \implies \kappa_{\rv Y}(A,x) = \kappa_{\rv Y}(A,x')$ for each $A\in\borel(\calY)$. 
That is, as a function of $x$, $\kappa_{\rv Y}(A,\argdot)$ is invariant w.r.t.\ $(\mathcal{X}, \sim)$. 
By \cref{lemma:maxinv_invariance}, $x \mapsto \kappa_{\rv Y}(A,x) = \kappa_{\rv Y}'(A,M(x))$, where $\kappa_{\rv Y}'$ is a probability kernel from $(\calM, \borel(\calM))$ into $(\calY, \borel(\calY))$. Therefore, for any $A \in \borel(\calY), B \in \borel(\calX)$,
\begin{align*}
	\bbE_{X,Y}[\mathds{1}_B(X) \mathds{1}_A(Y)] = \bbE_X[\mathds{1}_B(X) \kappa_{\rv Y}(A,X)] = \bbE_X[\mathds{1}_B(X) \kappa_{\rv Y}'(A,M(X))] \;,
\end{align*}
\ydnote{detail why k' is a probability kernel for camera ready}
which can be extended to arbitrary measurable functions on $\calX \times \calY$ by a standard (monotone class) argument. This in turn implies that $P_{Y|M(X)}$ is a version of $P_{Y|X}$, i.e., they are equal almost surely $\bbP(X)$, and therefore $Y \indep X \cond M(X)$.
\ydnote{Is there a reason to use $\bbP(X)$ instead of $P(X)$ ? }
\end{proof}

Finally, we prove that in realistic settings, there exists at least one $M(X) \in \tasksinv{}$.
\begin{lemma}\label{lemma:one_Mx_is_task}
Let $\tasksinv{}$ be the invariant tasks of interest w.r.t.\  $(\mathcal{X},\sim)$ and any loss function.  
% (\cref{def:invariant_tasks_interest}).
%If $X$ and $\sigma$ satisfy \cref{assumption:variance} 
Then there exists at least one maximal invariant that belongs to $\tasksinv{}$.
\end{lemma}
\begin{proof}
First, we have to prove by construction that a maximal invariant always exists. 
By definition equivalent elements have the same equivalence class and so $x \sim x' \iff \pi_{\sim}(x) = \pi_{\sim}(x')$. The projection map is measurable by assumption (\cref{assumption:measurability}), so it is a maximal invariant.

\ydnote{This should be changed to a lemma and \cref{assumption:variance} on  $\Risk{X}{\eta}$}
Second, due to the existence of at least one maximal invariant $M =\pi_{\sim}$ we have by \cref{assumption:variance} that that there exists at least one $M$ s.t. $\Risk{M(X)}{\eta}$.
This $M$ is therefore in $\tasksinv{}$.

%Second, $\Risk{\pi_{\sim}(X)}{\eta} \leq \Risk{X}{\eta}$ by the DPI for Bayes risk (\label{cref:dpi:bayes}). Therefore, if $X$ satisfies \cref{assumption:variance} then so does $\pi_{\sim}(X)$. 
\end{proof}

We close this section by establishing some properties of Bayes risk in this context. The following lemma, a data-processing inequality, appears in \citet{xu_minimum_2020}; we include it here for completeness, and provide a slightly more detailed proof.

\begin{lemma}[Data-processing inequality for Bayes risk] \label{lemma:dpi:bayes}
	Let $Z - X - Y$ be a Markov chain of random variables. For any loss function $L$,
	\begin{align}
		\Risk{Y}{X} \leq \Risk{Y}{Z} \;.
	\end{align}
\end{lemma}
\begin{proof}
	Recall that one characterization of conditional independence is that $Y \indep Z \cond X$ if and only if $Z = f(X,U)$ almost surely for some measurable function $f$ and $U \sim \textrm{Unif}(0,1)$ with $U \indep (X,Y)$ \citep[][Prop.\ 6.13]{kallenberg_foundations_2002}.

	Let $\psi_z$ be a Bayes decision rule for predicting $Y$ from $Z$, and likewise for $\psi_x$. By definition,
	\begin{align*}
		\Risk{Y}{Z} = \bbE_{Z,Y}[L(Y,\psi_z(Z))] = \bbE_{X,Y,U}[L(Y,\psi_z(f(X,U)))]  \;.
	\end{align*}
	For any $u \in (0,1)$, $\psi_z(f(\argdot,u))$ is a valid decision rule for predicting $Y$ from $X$ with risk at least as great as $\psi_x$. Therefore, $\Risk{Y}{X} \leq \Risk{Y}{Z}$.

\end{proof}

\begin{corollary}\label{lemma:Mx_is_X}
	Let $\tasksinv{}$ be the invariant tasks of interest with respect to  $(\mathcal{X},\sim)$ and any loss function, and $M$ any maximal invariant. For any $Y \in \tasksinv$,
	\begin{align}
		\Risk{Y}{M(X)} = \Risk{Y}{X} \;.
	\end{align}
\end{corollary}
\begin{proof}
	The result follows from applying \cref{lemma:dpi:bayes} to the trivial conditional independence $Y \indep M(X) \cond X$ and the non-trivial conditional independence from \cref{lemma:desintegration} $Y \indep X \cond M(X)$.
\end{proof}

%%%%%%%%%%%%%%%%%%%%%%%%%%%%%%%%%%%%%%%%%%%%%%%%%%%%%%%%%%
\subsection{\texorpdfstring{\Cref{prop:nicer_dist}}{Prop. 1}: simplifying and validating \texorpdfstring{\disttext{}}{invariant distortion} for log loss}
\label{appx:invariant_distortion}

In this section we show that \cref{def:excess_distortion} is a valid distortion for log loss, and we prove the equivalence between \cref{def:excess_distortion} and $\Risklog{M(X)}{Z}$. That equivalence is the key to prove \cref{thm:rate_invariance_distortion}.

The main steps in the proof are the following:
\begin{enumerate}[noitemsep]
% \item We show that if $\rv Y$ is an invariant task then $\rv Y - M(\rv X)- \rv X - \rv Z$ forms a Markov chain. 
%
\item Using the strict properness of the log loss, we relate the Bayes risk to the entropy:
\begin{equation}
\Risklog{Y}{Z}  =  \CH{Y}{Z} 
\end{equation}
\item We show that the supremum is achieved by $M(\rv X)$:
\begin{equation}
\sup_{\rv Y \in \tasksinv{} } \Risklog{Y}{Z} - \Risklog{Y}{X}  = \CH{M(\rv X)}{Z} - \CH{M(\rv X)}{X}
\end{equation}
\item Since $M$ is a deterministic function and $M(\rv X)$ is discrete, we have $\CH{M(\rv X)}{X}=0$. Therefore,
\begin{equation}
\sup_{\rv Y \in \tasksinv{} } \Risklog{Y}{Z} - \Risklog{Y}{X}  = \CH{M(\rv X)}{Z} 
\end{equation}
\item We conclude, as desired, that
\begin{equation}
\sup_{\rv Y \in \tasksinv{} } \Risklog{Y}{Z} - \Risklog{Y}{X}  = \Risklog{M(\rv X)}{Z} 
\end{equation}
\end{enumerate}

The first step consists of relating the log loss Bayes risk and conditional entropy. 
This is a simple lemma that directly comes from the fact that the conditional distribution $p(\rv Y \cond \rv Z)$  is the Bayes predictor.

\begin{lemma}\label{lemma:risk_entropy}
Let $\rv Y, \rv X$ be random variables then the log loss Bayes risk is equal to the conditional (discrete or differential) entropy:
\begin{equation}
\Risklog{Y}{X}= \CH{Y}{X}
\end{equation}
\end{lemma}
\begin{proof}
\begin{align}
\Risklog{Y}{X} 
&=  \inf_{q(\rv Y \cond \rv X)} \E{P(\rv X, \rv Y)}{- \log q(\rv y|\rv x)} & \text{Definition} \\
&=  \E{P(\rv X, \rv Y)}{- \log p(\rv y|\rv x)} 
& \text{Strict Proper.} \label{eq:risk_MI:strict_proper} \\
&= \CH{Y}{X} 
& \text{Definition}  
\end{align}
Where \cref{eq:risk_MI:strict_proper} uses the strict properness of the logarithmic scoring function rule \citep{gneiting_strictly_2007}.
\end{proof}

In the rest of the section, we will often be working with $\H{M(X)}$ and  $\CH{M(X)}{Z}$. 
Importantly, we would like our results to be independent of the choice of maximal invariant $M$.
We now prove that this will indeed be the case as all these (conditional) entropy terms are independent of the choice of $M$.
We only prove it for the marginal entropy $\H{M(X)}$ but the same proof holds for conditional entropies.

\begin{lemma}\label{lemma:same_entropy}
Let $\sim$ denote an equivalence relation on $\mathcal{X}$ satisfying \cref{assumption:discrete}.
Let $M$ and $M'$ be two different maximal invariants w.r.t.\  $(\mathcal{X}, \sim)$. Then $\op{H}{M(\rv X)} = \op{H}{M'(\rv X)}$.
\end{lemma}
\begin{proof}
Due to \cref{assumption:discrete}, $M(\rv X)$ is a discrete random variable and so $\op{H}{M(\rv X)}$ is the discrete entropy, which is invariant to bijective functions \cite{kraskov_estimating_2004}.
From \cref{lemma:maxinv_bijection} we know that there exists a bijection between $M$ and $M'$ from which we conclude that $\op{H}{M(\rv X)}=\op{H}{M'(\rv X)}$ as desired.
\end{proof}

One of the requirements on $Y$ to be set of downstream tasks $\tasks{}$ is the finiteness of $\op{H}{\rv Y}$. Thus, as a consequence of \cref{lemma:one_Mx_is_task,lemma:same_entropy}, in the case of log loss, all $M(X)$ are always in the set of downstream tasks $\tasks{}$.
%In particular, \cref{lemma:one_Mx_is_task} shows that  \cref{lemma:same_entropy} shows that they all satisfy \cref{assumption:variance} (finite entropy)

\begin{lemma}\label{lemma:all_Mx_is_task}
Let $\tasksinv{}$ be the invariant tasks of interest w.r.t.\  $(\mathcal{X},\sim)$ and the log loss.
Then all maximal invariants are in $\tasksinv{}$.
\end{lemma}
\begin{proof}
%We have to prove that any maximal invariant is conditionally invariant and has finite entropy (\cref{assumption:variance} for log loss).
Any $M(X)$ is conditionally invariant due to the \cref{def:maximal_invariant}.
From \cref{assumption:variance} we know that there exists at least one $M(\rv X)$ with finite entropy, by \cref{lemma:same_entropy} they must all have finite entropy.
We conclude that all $M(X) \in \tasksinv{}$ (\cref{def:invariant_tasks_interest}).
\end{proof}

We are now ready to prove the desired proposition.

% \ydnote{Do we need to always link back to the definition ? Probably can remove those, right ?}

\begin{manualprop}{\ref{prop:nicer_dist}}[Invariant Distortion for log loss]\label{appx:prop:invariant_distortion}
Let $\tasksinv{}$ be the invariant tasks of interest w.r.t.\  $(\mathcal{X},\sim)$ %(that satisfies \cref{assumption:discrete}) 
and the log loss. % (\cref{def:invariant_tasks_interest}).
Let $M$ be any maximal invariant, and $\rv Z$ be a representation of $\rv X$ for $\tasksinv{}$.
Then the excess distortion %(\cref{def:excess_distortion}) 
w.r.t.\ log loss, \disttextinv{}, is a valid distortion and
\begin{equation}
\diststinv{} = \Risklog{M(\rv X)}{Z}
\end{equation}
\end{manualprop}
\begin{proof}
We first prove that $\dists{} = \op{H}{M(\rv X) \cond Z}$, from which it is straightforward to show that \disttextinv{} is a valid distortion. Starting from the definition of \disttextinv{}, we have
\begin{align}
\diststinv{} &\defeq \sup_{Y \in \tasksinv{} } \Risklog{Y}{Z} - \Risklog{Y}{X} & \text{\cref{def:excess_distortion}} \\
&= \sup_{Y \in \tasksinv{} } \CH{Y}{Z} - \CH{Y}{X} & \text{\cref{lemma:risk_entropy}} \\
&= \sup_{Y \in \tasksinv{} } \CH{Y}{Z} - \CH{Y}{X,M(X),Z} & \rv Y \indep (M(X),Z) | \rv X \label{eq:Y_MZ_X}  \\
&= \sup_{Y \in \tasksinv{} } \CH{Y}{Z} - \CH{Y}{M(X),Z} & \rv Y \indep \rv X | (M(X),Z)  \label{eq:Y_X_MZ} \\
 &= \sup_{Y \in \tasksinv{}} \op{I}{Y ; M(X)|Z} \label{eq:IYMZ} & \text{Def.} \\
&= \sup_{Y \in \tasksinv{}} \CH{M(X)}{Z} - \CH{M(X)}{Y,Z} \label{eq:post_chainrule} & \text{Symmetry and def.} \\
&= \CH{M(X)}{Z} - \inf_{Y \in \tasksinv{}}  \CH{M(X)}{Y,Z}  \\
&= \CH{M(X)}{Z} - 0 \label{appx:eq:posent} & \text{Discrete H and $\rv Y = M(X)$} \\
&= \Risklog{M(X)}{Z}   & \text{\cref{lemma:risk_entropy}}
 \end{align}

\Cref{eq:Y_MZ_X} uses the fact that $\rv Y \indep M(X) \cond  \rv X$ (\cref{lemma:desintegration}), that $\rv Y \indep Z \cond  \rv X$ by \cref{def:representations}, and that  $M(X) \indep Z \cond  \rv X$ because $M(X) \in \tasksinv{}$ (again using \cref{def:representations}).
\Cref{eq:Y_X_MZ} uses \cref{lemma:desintegration}.
To go from \cref{eq:Y_X_MZ} to \cref{eq:post_chainrule} we use the symmetry of conditional mutual information.
\Cref{appx:eq:posent} uses the discreteness of $M(X)$ due to \cref{assumption:discrete}, so $\CH{M(X)}{Y,Z} \geq 0$ with equality when $\rv Y = M(X)$ which is possible due to \cref{lemma:one_Mx_is_task,lemma:all_Mx_is_task}.

From \cref{appx:eq:posent} it is easy to see that \disttextinv{} is valid as $\diststinv{}=\op{H}{M(\rv X) \cond \rv Z}=\E{p(\rv X, \rv Z)}{d(\rv X, \rv Z)}$ with $d(x,z) \defeq - \log p(M(x)\cond z)$ which due to the discreteness of $M(\rv X)$ (\cref{assumption:discrete}) is a function whose codomain is $\mathbb{R}_{\geq 0}$ as desired.
Due to \cref{assumption:variance} we know that for all constant $z \in \mathcal{Z}$ we have $\op{H}{M(\rv X) \cond z} \leq \infty$, so \disttextinv{} is valid.
%In particular we have that this holds for at least one $z \in \mathcal{Z}$ so \disttextinv{} is valid.
%Conditioning decreases entropy so we also have $\op{H}{M(\rv X) \cond \rv Z} \leq \H{M(X)} < \infty$ where the last inequality comes from \cref{assumption:variance}.
%We conclude that there exists $z^* \in \mathcal{Z}$ such that\  $\forall x,z$ with $p(x,z) > 0$, $d(x,z) \leq d_{max} < \infty$, and so \disttextinv{} is valid.
\end{proof}

Note that $\diststinv{} = \Risklog{M(\rv X)}{Z}$ is very simple to work with if we have access to some $M(\rv X)$.
Unfortunately, in practice $M(\rv X)$ might not be known, but often we will have access to some other random variable $\Tilde{X}$ which has all the information necessary about $M(\rv X)$.
See for example \cref{appx:vic_bince}.
We now prove that in such case we can optimize $\Risklog{\Tilde{\rv X}}{Z}$ instead of $\Risklog{M(\rv X)}{Z}$.

\begin{proposition}[Invariant Distortion without $M(X)$]\label{prop:dist_no_Mx}
Let $\tasksinv{},\sim,M,\rv X,\rv Z,\diststinv{}$ be as in \cref{prop:nicer_dist}.
Let $\Tilde{X}$ be a random variable\ such that $\Tilde{X} \sim \rv X$ and $\Tilde{X} \indep X \cond M(X)$ almost surely. Then
\begin{equation}
\diststinv{} = \Risklog{\Tilde{\rv X}}{Z} + c
\end{equation}
where $c$ depends only on $\Tilde{\rv X}$ and not on $\rv Z$.
\end{proposition}
\begin{proof}
From \cref{def:representations} and the fact that $M(X) \in \tasksinv{}$ (\cref{lemma:all_Mx_is_task}) we know that $Z - X - M(X)$ forms a Markov Chain (MC). 
Due to our assumption 
$\Tilde{X} \sim \rv X$ a.s. we also have that $M(\trv{X}) =M(\rv{X})$ a.s. so $M(X)  - \Tilde{X}  - M(X)$  forms a MC.
% Due to our assumption in the proposition we also have the MC $M(X) -\Tilde{X} - M(X)$.
Putting all together we obtain the MC $Z - X - M(X) - \trv{X} - M(X)$, which allows us to derive the following.
\begin{align}
\diststinv{} &= \Risklog{M(\rv X)}{Z} & \text{\cref{prop:nicer_dist}} \\
&= \H{M(\rv X)} - \MI{M(\rv X)}{Z} & \text{\cref{lemma:risk_entropy}}  \\
&= \H{M(\rv X)} - \MI{(M(\rv X),\trv{X})}{Z} & Z-M(X)-\trv{X}  \\
&= \H{M(\rv X)} - \MI{\trv{X}}{Z}  & Z-\trv{X}-M(X)  \\
&= \H{M(\rv X)} - \H{\trv X} + \CH{\trv{X}}{Z}  \label{appx:eq:no_Mx_1} \\
&= \CH{\trv X }{M(\rv X)}  +  \Risklog{\trv{X}}{Z}  & \text{\cref{lemma:risk_entropy}} 
\end{align}
where the last line uses the Markov Chain $M(X) -\trv  X- M(X)$ to provide a more interpretable value for the constant.
%Note that when $\CH{\trv X }{M(\rv X)}=0$ we have that $\trv X$ is some maximal invariant task $M'(X)$.
%This proposition is essentially a generalization to
\end{proof}

% Intuitively, the Markov Chain $X-\trv{X}-M(X)$ ensures that $\trv{X}$ has all information about $M(X)$.
% While $X-M(X)-\trv{X}$ ensures that being able to predict $\trv{X}$ will not retain unnecessary information from $\rv X$.

%%%%%%%%%%%%%%%%%%%%%%%%%%%%%%%%%%%%%%%%%%%%%%%%%%%%%%%%%%
\subsection{\texorpdfstring{\Cref{thm:rate_invariance_distortion}}{Theorem 2}: optimal bit-rate under log loss}
\label{appx:theorem_logloss}

Our main theoretical result is to characterize the minimal achievable rate to bound the Bayes risk of any invariant task.
Here we provide the proof for the case of log loss risk.
The result follows from \citepos{shannon_coding_1959} rate distortion theorem, and the validity of \disttextinv{}  (\cref{appx:prop:invariant_distortion}).

For convenience, we restate the well known rate distortion theorem. %\cite{cover_elements_2006}.
%Here we use the statement as given in \citet{cover_elements_2006}.
%Note that achievability is usually defined for deterministic encoders (\eg on p.\ 306 of \citet{cover_elements_2006}) but the proof holds for stochastic encoders (as noted on p.\ 316 of \citet{cover_elements_2006}), which we use in the present work.

\begin{lemma}(\citet{shannon_coding_1959}; Theorem 7.2.4 and 7.2.5 from \citet{berger_rate_1971})\label{lemma:rate_distortion}
Let $\op{D}{\rv X; \rv Z}$ be a valid distortion. 
The minimum achievable
bit-rate for transmitting an \iid source $\rv X$ with expected distortion less than $\delta \geq 0$ is given by the rate-distortion function:
\begin{equation}\label{appx:eq:rate_distortion}
R(\delta) =  \min_{p(\rv Z| \rv X) \ \text{such that} \ \op{D}{\rv X; \rv Z} \leq \delta}   \op{I}{\rv X; \rv Z} 
\end{equation}
\end{lemma}

We can now state our rate-invariance theorem.

\begin{manualthm}{\ref{thm:rate_invariance_distortion}}[Rate-invariance for log loss]
Let $\delta \geq 0$.
Let $\sim$ be an equivalence relation on $\mathcal{X}$ that partitions $\mathcal{X}$ into countably many equivalence classes (\cref{assumption:discrete}).
Let $\tasksinv{}$ be the invariant tasks of interest w.r.t.\  ($\mathcal{X}$,$\sim$) and the log loss, $M$ be any maximal invariant, and $\rv Z$ be a representation of $\rv X$ for $\tasksinv{}$.
Let $Rate(\delta)$ denote the minimum achievable bit-rate for transmitting an \iid source of $\rv Z$ such that\ for any $\rv Y \in \tasksinv{}$ we have $\Risklog{Y}{Z} \leq \delta + \Risklog{Y}{X}$.
Then $Rate(\delta)$  is finite and given by
\begin{align}
Rate(\delta) %&= \min_{p(\rv Z| \rv X)}   \op{I}{\rv X; \rv Z} +  \beta(\delta) \cdot \op{H}{M(\rv X) \cond \rv Z}    \label{eq:rate_inv_dist_IZX_unconstrained}\\
&=  \max \pa{0, \ \ \H{M(\rv X)} - \delta} \label{eq:appx:rate_invariance_distortion:HMX} \\
&= \max(0, \ \ \H{X} - \CH{ X}{M(\rv X)} - \delta) \label{eq:appx:rate_invariance_distortion:HX}
\end{align}
\end{manualthm}

\begin{proof}
We first prove that $Rate(\delta) \geq \max \pa{0, \ \ \op{H}{M(\rv X)} - \delta}$.
We then prove that the rate $\max \pa{0, \ \ \op{H}{M(\rv X)} - \delta}$ is achievable and so \cref{eq:appx:rate_invariance_distortion:HMX} holds.
Finally, we prove that $\op{H}{M(\rv X)} = \H{X} - \op{H}{\rv X \cond M(\rv X)}$ so \cref{eq:appx:rate_invariance_distortion:HX} holds which concludes the proof.

We want to transmit $\rv Z$ such that $\forall \rv Y \in \tasksinv{}$ we have $\Risklog{Y}{Z} \leq \delta + \Risklog{Y}{X}$, in other words we would like $\sup_{ \rv Y \in \tasks{}} \Risklog{Y}{Z} - \Risklog{Y}{X} \eqdef \diststinv{} \leq \delta $.
As \disttextinv{} is valid (\cref{appx:prop:invariant_distortion}) we can directly apply the rate distortion theorem (\cref{lemma:rate_distortion}):
\begin{align}
Rate(\delta) &= \min_{p(\rv Z| \rv X) \ \text{s.t.} \ \dists{} \leq \delta}   \op{I}{\rv X; \rv Z}     & \text{\cref{lemma:rate_distortion} and \cref{appx:prop:invariant_distortion}} \label{eq:rate_invariance_constrained}\\
&\geq \min_{p(\rv Z| \rv X) \ \text{s.t.} \ \dists{} \leq \delta} \op{I}{M(\rv X); \rv Z}  & \text{DPI} \label{eq:rate_invariance_distortion:ineq_positivity} \\
&= \min_{p(\rv Z| \rv X) \ \text{s.t.} \ \dists{} \leq \delta} \H{M(X)} - \CH{M(X)}{Z}  \\
&= \min_{p(\rv Z| \rv X) \ \text{s.t.} \ \dists{} \leq \delta} \H{M(X)} - \diststinv{}  & \text{\cref{appx:prop:invariant_distortion} and \cref{lemma:risk_entropy}} \\
&\geq \min_{p(\rv Z| \rv X) \ \text{s.t.} \ \dists{} \leq \delta} \H{M(X)} - \delta \label{eq:rate_invariance_distortion:ineq_delta}\\
&=  \H{M(X)} - \delta & \text{No } \rv Z 
\end{align}

Where \cref{eq:rate_invariance_distortion:ineq_positivity} uses the data processing inequality (DPI). 
As the rate is always non-negative we have  $Rate(\delta) \geq \max(0, \op{H}{M(\rv X)} - \delta)$.

We now prove that $\max(0,\H{M(\rv X)} - \delta)$ is attainable and so $Rate(\delta) = \max(0, \H{M(\rv X)} - \delta)$.
Specifically we need to find a representation $\rv Z$ of $\rv X$ such that 
\begin{equation}\label{eq:two_cases_theorem}
Rate(\delta) = 
\begin{cases}
0 & \text{If } \delta \geq \H{M(X)} \\
\H{M(X)} - \delta & \text{Else}  \\
\end{cases}
\end{equation}
The first case is trivial: set $\rv Z$ to be independent of $M(\rv X)$ and $\rv X$, \eg a constant.
Then, $\dists{}=\CH{M(X)}{Z}=\H{M(\rv X)}\leq \delta $ and $Rate(\delta) =\MI{Z}{X}=0$.

For the second case we need $Rate(\delta) \geq \H{M(X)} - \delta$ to be an equality when $\delta < \H{M(X)}$. 
This happens iff  inequalities \cref{eq:rate_invariance_distortion:ineq_positivity} and \cref{eq:rate_invariance_distortion:ineq_delta} are equalities, \ie iff 
$\rv X \indep \rv Z \cond M(\rv X)$ (for equality of the DPI \cite{cover_elements_2006}) and
$\dists{}=\delta$.
We do so by starting from $\rv Z = M(X)$ (such that $\rv X \indep \rv Z \cond M(\rv X)$) and ``erasing'' a fraction $\alpha$ of bits, similarly to binary erasure channels, until $\dists{}=\delta$.
Let $\mathcal{Z} \defeq \mathcal{M} \cup \set{ \epsilon}$ for some $\epsilon \not\in \mathcal{M}$ and let $\rv Z$ be a random variable that t.v.i. in $\mathcal{Z}$ and have the following conditional density parametrized by $\alpha \in [0,1[$:
\begin{equation}
\forall z \in \mathcal{Z}, \forall m \in \mathcal{M}, \quad p(z \cond m) =
\begin{cases}
1-\alpha & \text{if } z=m \\
\alpha & \text{if } z=\epsilon \\
0 &  \text{else } \\
\end{cases}
\end{equation}
A simple computation then gives $\dists{}\defeq \Risklog{M(X)}{Z} = \CH{M(X)}{Z}=(1-\alpha) \CH{M(X)}{Z=M(X)} + \alpha \CH{M(X)}{Z=\epsilon}= 0 + \alpha \H{M(X)}$, where the first equality uses \cref{lemma:risk_entropy} and the last equality uses $\CH{M(X)}{M(X)}=0$ due to the discreteness of $M(X)$ (\cref{assumption:discrete}).
We can thus achieve $\dists{}=\delta$ by setting $ \alpha = \frac{\delta}{\H{M(X)}}$.
Note that we will never divide by zero as $\H{M(X)}=0$ would be in the first case of \cref{eq:two_cases_theorem}. 
Importantly this $\rv Z$ still satisfies $\rv X \indep \rv Z \cond M(\rv X)$ as it was constructed solely using $M(X)$ and independent noise.

We thus proved that $\max \pa{0, \ \ \H{M(\rv X)} - \delta}$ is obtainable and that $Rate(\delta) \geq \max \pa{0, \ \ \H{M(\rv X)} - \delta}$.
From which we conclude that the best achievable bit-rate is $Rate(\delta) = \max \pa{0, \ \ \op{H}{M(\rv X)} - \delta}$.
\Cref{eq:appx:rate_invariance_distortion:HX}, follows from $\op{H}{M(\rv X)} = \MI{M(\rv X)}{X} = \H{X} - \op{H}{\rv X \cond M(\rv X)}$, which is a valid decomposition as both (differential conditional) entropy term are finite due to \cref{assumption:variance}.
The finiteness of $Rate(\delta)$ comes from the  fact that $Rate(\delta) \leq \op{H}{M(\rv X)} < \infty$ due to \cref{assumption:variance}.
\ydnote{To replace with lemma sayign that $\Risk{M(X)}{\eta} < \infty$ due to \cref{assumption:variance}}
\end{proof}

By setting $\delta=0$ we directly get the best achievable rate for the lossless prediction but lossy compression setting. 

\begin{corollary}[Invariant source coding for log loss]\label{corr:invariant_source_coding}
Let $\rv X$,$\sim$, $\tasksinv{}$, $M$, $\rv Z$ be as in \cref{thm:rate_invariance_distortion}.
Let $Rate(0)$ denote the minimum achievable bit-rate for transmitting an \iid source of $\rv Z$ such that for any $\rv Y \in \tasksinv{}$ we have $\Risklog{Y}{Z} = \Risklog{Y}{X}$.
Then $Rate(0)$  is finite and given by
\begin{align}
Rate(0) &=   \op{H}{M(\rv X)} \\
&= \H{X} - \op{H}{\rv X \cond M(\rv X)} \label{corr:gains}
\end{align}
\end{corollary}

%%%%%%%%%%%%%%%%%%%%%%%%%%%%%%%%%%%%%%%%%%%%%%%%%%%%%%%%%%
\subsection{Recovering previous results in the literature}
\label{appx:recovering}

\Cref{corr:invariant_source_coding} recovers many previous results in the literature:
\begin{description}
\item[Unlabeled Graphs] 
Let us consider the task of compressing unlabeled graphs, here we consider tasks that are invariant to graph isomorphisms.
A possible maximal invariant is the graph canonization and $\op{H}{M(\rv X)}$ becomes the well known \textit{structural entropy} (also called topological information content) \cite{rashevsky_life_1955,yongwook_choi_compression_2009}.
If all isomorphic graphs are permissible and equiprobable, \citet{yongwook_choi_compression_2009} show that the structural entropy is  $\op{H}{\rv S} = \op{H}{\rv X} - \E{x \sim p(\rv X) }{\log \frac{n!}{|\mathrm{Aut}_{\mathcal{G}}(x)|}}$.
This is \cref{corr:gains}, where the second term corresponds to $\op{H}{\rv X \cond M(\rv X)}$ with a uniform distribution on isomorphic graphs.
\item[Multisets] 
Let us derive the best achievable bit-rate for compressing multisets.
Let $\rv X$ be any sequence and $\tasksinv{}$ be invariant to permutations of that sequence.
One possible maximal invariant in that case is the empirical measure (also called type), \ie, the counts $K_1,\dots,K_n$ of each of the $n$ elements that are present in the sequence $\rv X$.
Lossless compression of multisets thus requires $\H{M(X)}=\H{K_1,\dots,K_n}$, as discussed by \citet{varshney_benefiting_2007}.
Using \cref{corr:gains} we can also characterize the bits gains that you obtain by considering the invariance, namely, $\H{\rv X | M(X)}$.
This recovers theorem 1 of \citet{varshney_benefiting_2007}, where $\H{\rv X | M(X)}$ is called the ``order entropy''.
Note that similarly to our example in the main text about \iid coin flips, the amount of bits needed to losslessly compress the multiset grows as $\Theta(\log n)$  \cite{varshney_benefiting_2007}. 
\item[Information Bottleneck (IB)] 
Suppose you are interested in predicting a single task $\rv Y = t(\rv X)$, where $t$ is a (deterministic) ``target function''.
The task is invariant to any transformations between examples in the preimage of the labeling.
So the maximal invariant is $t(\cdot)$ and the distortion becomes $\CH{t(\rv X)}{\rv Z}=\CH{\rv Y}{\rv Z}$.
Then the rate-distortion function (\cref{appx:eq:rate_distortion}) becomes the information bottleneck (IB) \cite{tishby_information_2000}.
Using \cref{corr:invariant_source_coding} we see that for lossless predictions the optimal rate is $Rate(0) = \H{Y} = \H{X} - \CH{X}{Y} = \MI{X}{Y} $ as  shown in \cite{wu_learnability_2019,fischer_conditional_2020}.
From a compression stand point this is nevertheless not very useful as $Rate(0) = \H{Y}$, so IB for deterministic labels tells you to entropy code the labels $\rv Y$. 
\item[Lossless]
Let $\rv X$ be discrete.
Every task will always be invariant to the equality ``$=$`` equivalent relation.
In this case the maximal invariant is the identity function, and we recover Shannon's source coding theorem $Rate(0) = \op{H}{M(\rv X)}=\op{H}{\rv X}$.
\end{description}

%%%%%%%%%%%%%%%%%%%%%%%%%%%%%%%%%%%%%%%%%%%%%%%%%%%%%%%%%%
\subsection{Generalizing \texorpdfstring{\Cref{thm:rate_invariance_distortion}}{Theorem 2}: optimal bit rate for lossless prediction and any loss}
\label{appx:theorem_lossless}

\bbnote{Skipping this section for now, will return.}

\Cref{corr:invariant_source_coding} characterizes the minimal achievable bit rate for the lossless prediciton regime w.r.t. log loss.
Here we show that the same result generalizes to essentially all loss function of practical interest.

% Our result relies on the following DPI for Bayes Risks:
% \ydnote{I'm citing this recent lemma (which \ben{} pointed me to) instead of giving my (more detailed) proof because they did it first, but I think their proof is not not entirely correct / lacks details.
% FOr example they say it's true for all losses but I don't think it's true for losses where the Bayes risk is not finite.
% That's why I added ``such that Bayes risk is well defined'' in the lemma.
% }

% \begin{lemma}[Lemma 1 from \citet{xu_minimum_2020}]\label{lemma:dpi}
% Let $\rv Y$ t.v.i. in $\mathcal{Y}$.
% Let $L : \mathcal{Y} \times \mathcal{A} \to \mathbb{R}_{\geq 0}$ be any loss function such that the Bayes risk $\Riskl{Y}{Z}$ is well defined for all random variable $\rv Z$.
% Let $\rv X$ be a random variable such that $\rv Y \indep \rv Z \cond \rv X $ then
% \begin{equation}
% \Riskl{Y}{X} \leq \Riskl{Y}{Z}
% \end{equation}
% \end{lemma}

The invariant source coding theorem does not hold for \textit{any} loss, for example if a loss is a constant function then the Bayes risk $\Riskl{Y}{Z}$ will not depend on the input $\rv Z$, and so the best achievable bit rate will trivially be 0 which is different than $\H{M(X)}$.
But it essentially holds for all losses that are minimized only by the ``correct'' predictor.
Specifically it holds for all losses that we dub information preserving.

\begin{definition}[Information preserving losses]\label{def:meaningful}
Let $L : \mathcal{Y} \times \mathcal{A} \to \mathbb{R}_{\geq 0}$ be any loss function such that the Bayes risk $\Riskl{Y}{Z}$ is well defined for all random variable $\rv Z$.
We say that $L$ is an \textit{information preserving loss} iff the optimal risk of deterministic targets is achieved only using inputs that have all the information about the output, \ie, iff for any function $t : \mathcal{X} \to \mathcal{Y}$ and and r.v.s $Z,X$ we have
\begin{equation}
\Riskl{t(X)}{X} = \Riskl{t(X)}{Z} \implies \exists h: \mathcal{Z} \to \mathcal{Y} \text{ s.t. } t(X) \aseq h(Z)
\end{equation}
\end{definition}

% %\ydnote{work with discrete Y instead}
% \begin{definition}[Meaningful loss functions]\label{def:meaningful}
% Let $L : \mathcal{Y} \times \mathcal{A} \to \mathbb{R}_{\geq 0}$ be any loss function such that the Bayes risk $\Riskl{Y}{Z}$ is well defined for all random variable $\rv Z$.
% We say that $L$ is \textit{meaningful} iff the optimal risk is achieved only using inputs that have all the information about the output, \ie, iff
% \begin{equation}
% \Riskl{Y}{Z} = \Riskl{Y}{Y} \implies \rv Y \indep \rv Y \cond Z
% \end{equation}
% \end{definition}

In particular we have that if $t(X)$ is a discrete r.v. then $\Riskl{t(X)}{X} = \Riskl{t(X)}{Z} \implies \CH{t(X)}{Z} = 0$.

Essentially all losses used in practice satisfy \cref{def:meaningful}.
%Meaningfulness of the loss function is a very mild assumption which holds for essentially all losses used in practice.
For example it holds for the following very general families of losses:
\begin{description}
\item[Strictly proper scoring rules] Let $L : \mathcal Y \times \mathcal{P}(\mathcal{Y})  \to \mathbb{R}_{\geq 0}$ be a scoring rule that essentially quantifies with $L(y, q(\rv y),)$ the price/loss incurred by probabilistic prediction $q(\rv y)$ when $y$ is observed (lower is better).
$L$ is \textit{strictly proper} \cite{gneiting_strictly_2007} (w.r.t. $\mathcal{P}(\mathcal{Y})$) iff:
\begin{equation}\label{eq:strict_proper}
\forall p,q \in \mathcal{P}(\mathcal{Y}) \quad \E{p(\rv y)}{L(y,q(\rv y))} \leq \E{p(\rv Y)}{L(y,p(\rv Y))}
\end{equation}
with equality if and only if $p=q$.
Common examples are the log loss \cite{good_rational_1952}, 
%pseudo-likelihood (for strictly positive distributions) \cite{csiszar_consistent_2006},
Brier score \cite{brier_verification_1950}, spherical score \cite{good_comment_1971}, or the maximum mean discrepancy with characteristic bounded kernels \cite{sriperumbudur_injective_2008,huszar_scoring_2013}.
\item[Point-wise loss functions] Let $L :  \mathcal Y \times \hat{\mathcal{ Y}}  \to \mathbb{R}_{\geq 0}$ be a loss function that essentially quantifies with $L(y,\hat y)$ the price/loss incurred by the point prediction $\hat y$ when $y$ is observed (lower is better).
As is standard \cite{gneiting_making_2010} we assume that :
\begin{equation}\label{eq:loss_function}
L(y,\hat y)=0 \iff \hat y = y
\end{equation}
This holds for most pointwise losses of interest: mean squared error, mean absolute error, 0-1 loss (accuracy), Huber loss, \dots 
%\ydnote{Hinge loss is messy to deal with, I prefer the current simple assumption. BUt could generalize}
%\footnote{A notable exception is the Hinge loss, although we can also prove that the theorem holds for it.}
\end{description}

\begin{lemma}
Strictly proper scoring rules and point-wise loss functions are information preserving (\cref{def:meaningful}).
\end{lemma}
\begin{proof}
First let us consider point-wise loss functions.
Suppose that $\rv Z$ is such that $\Riskl{t(X)}{Z} = \Riskl{t(X)}{X}$.
As $L$ is a point-wise loss function (\cref{eq:loss_function}) and $t$ is a deterministic function, we have that $\Riskl{t(X)}{X} = 0$.
As a result we have $\Riskl{t(X)}{Z}=0$ which using again \cref{eq:loss_function} is equivalent to the existence of $h$ s.t.   $h(Z) \aseq t(X)$ as desired.
%Using the fact that $f$ is a deterministic function we have  $\rv Y \indep f(\rv Z) \cond \rv Z$ from which we conclude that $\rv Y \indep \rv Y \cond \rv Z$.
%As a result we have $\op{H}{\rv Y \cond f(\rv Z)} = \CH{Y}{Y}$.
%As $\rv Y -  \rv Z - f(\rv Z)$ we can use the DPI to get $\CH{Y}{Z} \leq \op{H}{\rv Y \cond f(\rv Z)} = \CH{Y}{Y}$.
%Trivially, $\rv Y - \rv Y - \rv Z$ also forms a Markov Chain so using again the DPI we get $\CH{Y}{Y} \leq \CH{Y}{Z}$
%We conclude that $\CH{Y}{Z} = \CH{Y}{Y}$ as desired.

Now let us consider the case of proper scoring rules.
Suppose that $\rv Z$ is such that $\Riskl{t(X)}{Z} = \Riskl{t(X)}{X}$.
As $L$ is a strictly proper scoring rule we have that $p(t(X) | Z) \aseq p(t(X) | X)$. 
The latter is a delta function, so the former must also. We thus have the existence of $h$ s.t.   $h(Z) \aseq t(X)$ as desired.
%We thus have that $\Riskl{Y}{Z} = \Riskl{Y}{Y,Z}$ so using the generalized conditional information for strictly proper scoring rules (statement 2 in \citet{huszar_scoring_2013}) we have that $\rv Y \indep \rv Y \cond \rv Z$ as desired.
%As a result $p(Y|Z) =  p(Y|Y)$ almost surely from which we conclude that $\CH{Y}{Z} = \CH{Y}{Y}$ as desired.
\end{proof}

We now have all the tools to prove the general invariant source coding theorem.

\begin{theorem}[General invariant source coding]\label{thm:invariant_source_coding}
Let $\sim$ be an invariance relation on $\mathcal{X}$ that satisfies \cref{assumption:discrete}.
Let $\tasksinv{}$ be the invariant tasks of interest w.r.t ($\mathcal{X}$,$\sim$) and any loss function $L : \mathcal{Y} \times \mathcal{A} \to \mathbb{R}_{\geq 0}$ as in \cref{def:invariant_tasks_interest}, $M$ be any ($\mathcal{X}$,$\sim$) maximal invariant as in \cref{def:maximal_invariant}, and $\rv Z$ be a representation of $\rv X$ for $\tasksinv{}$ as in \cref{def:representations}.
Let $Rate(0)$ denote the minimum achievable bit-rate for transmitting an \iid source of $\rv Z$ such that for any $\rv Y \in \tasksinv{}$ we have $\Riskl{Y}{Z} = \Riskl{Y}{X}$.
Then $Rate(0)$  is finite and given by
\begin{align}
Rate(0) 
&= \op{H}{M(\rv X)} \label{eq:appx:invariant_source_coding:HMX} \\
&= \H{X} - \op{H}{\rv X \cond M(\rv X)}  \label{eq:appx:invariant_source_coding:HX}
\end{align}
\end{theorem}
\begin{proof}
By \cref{lemma:Mx_is_X}  we know that for all $\rv Y \in \tasksinv{}$ we have $ \Riskl{Y}{X} = \Riskl{Y}{M(X)}$  for any loss function.
As a result, the lossless prediction bit rate is at most $Rate(0) \leq \op{H}{M(\rv X)}$ because by \citepos{shannon_mathematical_1948} source coding theorem  $M(X)$ can be transmitted using $\H{M(\rv X)}$ bits as it is discrete (\cref{assumption:discrete}) and its entropy is finite (\cref{assumption:variance}).
\ydnote{TO replace with lemma finite risk of M(X)}
 
Let us now show that it is not possible to achieve a lower rate.
By \cref{lemma:one_Mx_is_task} there exists at least one maximal invariant such that $M(\rv X) \in \tasksinv{}$.
We now prove that there is no $\rv Z$ such that $\Riskl{M(\rv X)}{Z} = \Riskl{M(\rv X)}{X}$ and can be transmitted with less than $\H{M(\rv X)}$ bits.
Suppose that $\Riskl{M(\rv X)}{Z} = \Riskl{M(\rv X)}{X}$,  then because $L$ is a meaningful loss function we have there exists  function $h$ s.t. $h(Z) = M(X)$.
Using the discreteness of $M(X)$  (\cref{assumption:discrete}), we thus have $\CH{M(X)}{Z}=0$.
Using the RD theorem (\cref{lemma:rate_distortion}) we know that the minimum bit rate for transmitting $\rv Z$ under the constraint $\CH{M(X)}{Z}=0$ is $\MI{Z}{M(X)} = \H{M(X)} - \CH{M(X)}{Z} = \H{M(X)} - 0$.
We thus find that transmitting a $\rv Z$ which ensures lossless predictions cannot require less  than $\H{M(\rv X)}$ bits which concludes the proof that $Rate(0) = \H{M(\rv X)}$.
To get \cref{eq:appx:invariant_source_coding:HX} we use the same decomposition as in \cref{thm:rate_invariance_distortion}.
\end{proof}

%%%%%%%%%%%%%%%%%%%%%%%%%%%%%%%%%%%%%%%%%%%%%%%%%%%%%%%%%%
\subsection{Generalizing \texorpdfstring{\Cref{thm:rate_invariance_distortion}}{Theorem 2}: optimal bit rate under MSE loss}
\label{appx:theorem_mse}

In \cref{appx:theorem_logloss} we proved the rate-invariance theorem for the case of log loss.
Log loss is the standard loss function for classification in ML, but in the case of regresssion it is more common to use the MSE loss function.
In \cref{thm:invariant_source_coding} we have seen that our results for lossless prediction regime also holds for MSE (and other losses).
Due to the importance of MSE in ML, we also provide a full rate-invariance theorem for MSE. 

We assume that $Y \in \mathbb{R}^k$ for all $Y \in \tasksinv{}$, with some $k < \infty$. Tasks taking values in fewer dimensions can always be padded with zeros. In this section $||\argdot ||$ denotes the Euclidean norm, the $L^2$-norm of a random variable $X$ is $\sqrt{\mathbb{E}[||X||^2]}$, and $L^2(\Omega, \mathcal{H}, \mathbb{P})$ is the Hilbert space of all $\mathbb{R}^k$-valued random variables with finite $L^2$-norm (random variables that are almost surely equal are identified as the same element of $L^2$). Since $\Omega$ and $\mathbb{P}$ remain unchanged but we may consider different $\sigma$-algebras, we use, e.g., $L^2(\mathcal{H})$ for short.

Importantly, we do not require \cref{assumption:discrete} (countable $\mathcal{X}/\sim$). 
Instead, we require the following common (known as a finite power constraint in compression \cite{cover_elements_2006}) regularity condition on $\tasksinv{}$ to ensure that we can attain a relevant supremum.

\ydnote{should change this assumption and paragraph to $Y$ and $X$ and then have lemma for $M(X)$
}
\begin{assumption}[$L^2$-boundedness]\label{assumption:l2:bounded}
    We assume that $\tasksinv$ is bounded in $L^2$. That is, there is some $0<B^2<\infty$ such that $\bbE[||Y||^2] \leq B^2$ for all $Y \in \tasksinv$.
\end{assumption}

Note that this is a slightly more stringent version of \cref{assumption:variance}, as it essentially requires bounded variance of $\rv Y$ rather than only finite variance.

Let $\borel(\calM)$ denote the $\sigma$-algebra generated by a maximal invariant (all maximal invariants generate the same $\sigma$-algebra because they are one-to-one functions of each other \cref{lemma:maxinv_bijection}), and let $\calM_b$ denote all $\mathbb{R}^k$-valued functions that are $\borel(\calM)$-measurable and bounded in $L^2$. Then $\calM_b \subset \tasksinv$, and by \cref{lemma:one_Mx_is_task}, $\calM_b$ is non-empty. 

The main step in the proof of a rate-invariance theorem for MSE is the following proposition, which shows that the excess risk distortion for MSE is a valid distortion that can be expressed in terms of a maximum over $\calM_b$.

% The main steps in the proof are the following:

% \begin{enumerate}
% \item Simplify the excess risk distortion for MSE to:
% \begin{equation}
%   \diststinv{} = \max_{f \in \calM_b \cap \tasksinv} \Riskmse{f(M(X))}{Z}  
% \end{equation}
% % \item Show that the supremum is achieved and conclude that the excess risk is a valid distortion, and conclude that the RD theorem hold.
% \item Provide guarantees when using $\Riskmse{M(X)}{Z}$ for any $M$
% \end{enumerate}

% First we simplify the excess risk distortion for MSE and invariant tasks. In order to do so, 

\begin{proposition}[Invariant Distortion for MSE]\label{appx:prop:invariant_distortion_mse}
Let $\tasksinv{}$ be the invariant tasks of interest w.r.t.\  $(\mathcal{X},\sim)$  and w.r.t.\ the MSE.
Fix any maximal invariant $M$ that is also in $\tasksinv{}$, and let $\rv Z$ be a representation of $\rv X$ for $\tasksinv{}$. 
Then the excess distortion  w.r.t.\ MSE, \disttextinv{}, is a valid distortion and
\begin{equation} \label{eq:mse:distortion}
\diststinv{} = \sup_{f \in \calM_b}  \Riskmse{f(M(\rv X))}{Z} \;.
\end{equation}
\end{proposition}
\ydnote{If we wanted to be picky we should use the same expectation notation everywhere (I don't feel strongly about either). Also not important for main deadline.}
\begin{proof}
	For compactness of notation, we use, for example, $\bbE_{X,Y}$ to denote expectation with respect to $\bbP$, and $\bbE_X \bbE_{Y|X}$ to denote an iterated expectation. Our proof makes use of conditional expectation in $L^2$ being defined as projection in a Hilbert space. See \citep[e.g.,][Ch.\ 22-23]{jacod_probability_2004}.

	Firstly, fix some $Y\in \tasksinv$. It is well known that
	\begin{align*}
	    \bbE[||Y - \phi(X)||^2] & = \bbE[\; \bbE[||Y||^2\mid X] - ||\bbE[Y\mid X]||^2 \;] + \bbE[\; ||\bbE[Y|X] - \phi(X)||^2 \;] \;.
	\end{align*}
	Taking the infimum over all measurable $\phi\colon \mathcal{X} \to \mathbb{R}^k$, we have
	\begin{align*} 
	    \inf_{\phi\colon \mathcal{X} \to \mathbb{R}^k} \bbE_{X,Y}[||Y - \phi(X)||^2] = \bbE_{X}\left[\; \bbE_{Y|X}[||Y||^2 ] - ||\bbE_{Y|X}[Y]||^2 \; \right] \;,
	\end{align*}
	when $\phi(X) = \bbE[Y\mid X]$ $\bbP(X)$-almost everywhere. Now by the conditional invariance, $Y \indep X \cond M(X)$ (\cref{lemma:desintegration}), which implies $\bbE_{Y|X}[f(Y)] = \bbE_{Y|M(X)}[f(Y)]$ for any measurable function $f$. Therefore, 
	\begin{align} \label{eq:phi:decomp}
	    \inf_{\phi\colon \mathcal{X} \to \mathbb{R}^k} \bbE_{X,Y}[||Y - \phi(X)||^2] = \bbE_{M(X)}\left[\; \bbE_{Y|M(X)}[||Y||^2 ] - ||\bbE_{Y|M(X)}[Y]||^2 \; \right] \;,
	\end{align}
	when $\phi(X) := \phi'(M(X)) = \bbE_{Y|M(X)}[Y]$ $\bbP(X)$-almost everywhere.

	Similarly, for fixed $Z$ t.v.i $\calZ$ with $Z \indep Y \cond M(X)$,
	\begin{align} \label{eq:psi:decomp}
    & \bbE_{Z,Y}[||Y - \psi(Z)||^2]  = \bbE_{M(X)} \bbE_{Y,Z|M(X)}[[||Y - \psi(Z)||^2] \\ 
    & \quad = \bbE_{M(X)}\left[\; \bbE_{Y|M(X)}[||Y||^2 ] - ||\bbE_{Y|M(X)}[Y]||^2 \; \right]  \nonumber \\
    & \quad\;  + \bbE_{M(X)}\left[\; ||\bbE_{Y|M(X)}[Y] - \bbE_{Z|M(X)}[\psi(Z)] ||^2 \;  
      + \bbE_{Z|M(X)}[|| \psi(Z) - \bbE_{Z|M(X)}[\psi(Z)]||^2 ] \;\right] \nonumber
	\end{align}

	Observe that for any $Y\in\tasksinv$, \eqref{eq:phi:decomp} and the first term of \eqref{eq:psi:decomp} will cancel in the excess risk distortion. Therefore, 
	\begin{align*}
		\diststinv{} = \sup_{Y \in \tasksinv} \inf_{\psi : \calZ\to\bbR^k} &
			\bbE_{M(X)}\left[\; ||\bbE_{Y|M(X)}[Y] - \bbE_{Z|M(X)}[\psi(Z)] ||^2 \right. \\
      & \left. + \bbE_{Z|M(X)}[|| \psi(Z) - \bbE_{Z|M(X)}[\psi(Z)]||^2 ] \;\right] \;. \nonumber
	\end{align*}
	
	When taking the supremum over $Y\in\tasksinv$, $Y$ can only affect \disttextinv{} through its conditional expectation given $M(X)$, $\bbE_{Y|M(X)}[Y]$. That conditional expectation is a $\borel(\calM)$-measurable function, so $\bbE_{Y|M(X)}[Y] \in \calM_b$ for all $Y\in\tasksinv$. Therefore,
	\begin{align*}
		\{ \bbE_{Y|M(X)}[Y] \colon Y \in \tasksinv \} \subset \calM_b \subset \tasksinv \;,
	\end{align*}
	and we can take the supremum over functions $f \in \calM_b$ instead of $Y \in \tasksinv$, which yields
	\begin{align*}
		\diststinv{} = \sup_{f \in \calM_b} \inf_{\psi : \calZ\to\bbR^k} &
			\bbE_{M(X)}\left[\; ||f(M(X)) - \bbE_{Z|M(X)}[\psi(Z)] ||^2 \right. \\
      & \left. + \bbE_{Z|M(X)}[|| \psi(Z) - \bbE_{Z|M(X)}[\psi(Z)]||^2 ] \;\right] \;. \nonumber
	\end{align*}
	Expanding each quadratic and canceling terms involving $\bbE_{M(X)}[||\bbE_{Z|M(X)}[\psi(Z)]||^2]$, we find
	\begin{align}
		\diststinv{} &= \sup_{f \in \calM_b} \inf_{\psi : \calZ\to\bbR^k} 
			\bbE_{M(X),Z}\left[ || f(M(X)) - \psi(Z) ||^2 \right] \\
			& = \sup_{f \in \calM_b}  \Riskmse{f(M(\rv X))}{Z} \\
			& = \sup_{f \in \calM_b} \bbE_{M(X),Z}\left[ || f(M(X)) - \bbE_{M(X)|Z}[f(M(X))] ||^2 \right] \;.
	\end{align}
	
	Now, since conditional expectation given $Z$ is just projection onto the (Hilbert) subspace $L^2(\borel(\calZ))$, we have
	\begin{align*}
	    \diststinv{} &= \sup_{h \in \calM_b \cap L^2(\borel(\calZ))^{\perp}} \bbE_X[ ||h(M(X))||^2 ] \;,
	\end{align*}
	where $L^2(\borel(\calZ))^{\perp}$ is the subspace orthogonal to $L^2(\borel(\calZ))$ in $L^2(\mathcal{H})$. Since $L^2(\borel(\calZ))^{\perp}$ and $L^2(\borel(\calM))$ are both closed (sub-)Hilbert spaces, it is straightforward to show that so too is their intersection $L^2(\borel(\calM)) \cap L^2(\borel(\calZ))^{\perp}$. The bounded (by $B$) elements of $\calM_b \cap L^2(\borel(\calZ))^{\perp}$ are just the closed ball of radius $B$, so 
	\begin{align*}
	    \diststinv{} &= \sup_{\substack{h \in L^2(\borel(\calM)) \cap L^2(\borel(\calZ))^{\perp} \\ \bbE_X[ ||h(M(X))||^2 ] \leq B^2}} \bbE_X[ ||h(M(X))||^2 ] \;.
	\end{align*}
	Now, since neither $\calM_b$ nor $L^2(\borel(\calZ))^{\perp}$ is empty, their intersection is empty if and only if $\calM_b \subset L^2(\borel(\calZ))$, i.e., $\borel(\calM)yeah  \subset \borel(\calZ)$: all maximal invariants can be written as functions of $Z$. In that case, $\diststinv{}=0$. Alternatively, if $\calM_b \cap L^2(\borel(\calZ))^{\perp}$ is not empty, then choose some $h^*$ from it such that $\bbE_X[||h^*(M(X))||^2] = B^2 = \sup_{f \in \calM_b} \Riskmse{f(M(\rv X))}{Z}$.
	
% 	By definition, $\calM_b$ is $L^2$-bounded.  and therefore the supremum is achieved by some $f^* \in \calM_b$, which is
% 	\begin{align} \label{eq:dist:max}
% 		\diststinv{} & =  \bbE_{M(X),Z}\left[ || f^*(M(X)) - \bbE_{M(X)|Z}[f^*(M(X))] ||^2 \right] \\
% 			& = \Riskmse{f^*(M(\rv X))}{Z} \\
% 			& = \max_{f \in \calM_b}  \Riskmse{f(M(\rv X))}{Z} \;,
% 	\end{align}
% 	which is \eqref{eq:mse:distortion}.
	
	Defining $d(x,z) := || h^*(M(x)) ||^2$ yields a valid distortion \disttextinv{}.
\end{proof}

Note that the last last part of the proof makes it clear that for MSE the invariance distortion is either $0$ or $B^2$.
Intuitively this happens because MSE risk is not invariant to bijections so it possible to make any predictive mistake arbitrarily bad by setting $M(X)$ to be arbitrarily large at this mistaken prediction. 
This suggests that for the MSE risk (and other loss functions that are not invariant to bijections) the expected excess risk might be better suited than the worst case excess risk that we considered.

As the invariant distortion under MSE is valid, we can now simply incorporate it into the rate distortion theorem to get the desired theorem.

\begin{theorem}[Rate-invariance for MSE]\label{thm:rate_invariance_distortion_mse}
Let $\delta \geq 0$.
Let $\tasksinv{}$ be the invariant tasks of interest w.r.t.\ ($\mathcal{X}$,$\sim$) and the MSE, $M$ be any maximal invariant in $L^2(\Omega,\mathcal{H},\mathbb{P})$, and $\rv Z$ be a representation of $\rv X$ for $\tasksinv{}$.
Let $Rate(\delta)$ denote the minimum achievable bit-rate for transmitting an \iid source of $\rv Z$ such that\ for any $\rv Y \in \tasksinv$ we have $\Riskmse{Y}{Z} \leq \delta + \Riskmse{Y}{X}$.
Then $Rate(\delta)$  is given by
\begin{align}
Rate(\delta) & = 
\inf_{P(\rv Z | \rv X) \text{ s.t. } \diststinv{} \leq \delta }   \op{I}{\rv X; \rv Z} \label{eq:rate_invariance_distortion_mse} \;.
\end{align}
where $\diststinv{} \defeq \sup_{f \in \calM_b} \Riskmse{f(M(\rv X))}{Z}$.
\end{theorem}
\begin{proof}
The result \eqref{eq:rate_invariance_distortion_mse} follows from the fact that \disttextinv{} is a valid distortion (\cref{appx:prop:invariant_distortion_mse}) and the rate-distortion theorem (\ref{lemma:rate_distortion}). 
\end{proof}
\ydnote{For final version we'll have to deal with the fact that the distortion seems to be either 0 or B so there are only 2 rates that can be reached.
Is it stillf usefule ?
THe derivations seems very useful, but the result less so.
Really for MSE we should a about median / mean instead of sup excess risk ...
}
\ydnote{For final version (not a priority) we might want to provide also an upper bound, which is achieved if $f^*(M(X))$ is gaussian with variance $sigma^2$, specifically, $\frac{1}{2}\log \frac{\sigma^2}{\delta}$. This shows that for $\delta > 0$ we do have finite rate.} 
As a corollary, we obtain the following lower bound for the rate, which may be useful in practice.
\begin{corollary}
	Let $M$ be any $\bbR^k$-valued maximal invariant with a probability density with respect to Lebesgue measure. Let $g \colon \mathcal{M} \to \bbR^k$ be any homeomorphism of $M$ (including the identity map), and $f^*$ any maximum distortion achieving function. Then the following lower bounds hold:
	\begin{align} \label{eq:mse:rate:lower:bounds}
	 Rate(\delta) & \geq  h(g(M(X))) - \frac{k}{2}\ln(2\pi e \delta/k) \\
	 Rate(\delta)	&  \geq h(f^*(M(X))) - \frac{k}{2}\ln(2\pi e \delta/k) \;.
	\end{align} 
\end{corollary}
\begin{proof}
	First, by the DPI, for any homeomorphism $g$ of $M$,
	\begin{align*}
		\op{I}{\rv X; \rv Z} \geq \op{I}{\rv M(X); \rv Z} = \op{I}{ g(M(X)); \rv Z} \geq  \op{I}{ f^*(M(X)); \rv Z} \;.
	\end{align*}
	The first inequality is an equality if and only if $X \indep Z \cond M(X)$; the second if and only if $f^*$ is a homeomorphism of $M$ (and therefore is itself a maximal invariant). Second, using the translation-invariance of differential entropy and the fact that conditioning reduces differential entropy, 
	\begin{align} \label{eq:information:chain}
		\op{I}{\rv M(X); \rv Z} & = h(M(X)) - h(M(X) \cond Z) \\
			& \geq h(M(X)) - h(M(X) - \bbE_{M(X)|Z}[M(X)] \cond Z) \\
			& \geq h(M(X)) - h(M(X) - \bbE_{M(X)|Z}[M(X)]) \;,
	\end{align}
	Now, 
	\begin{align*}
		\bbE_{M(X),Z}\left[ ||M(X) - \bbE_{M(X)|Z}[M(X)]||^2 \right] & \leq \bbE_{M(X),Z}\left[ || f^*(M(X)) - \bbE_{M(X)|Z}[f^*(M(X))] ||^2 \right] \\
		& = \diststinv{} \;. %\leq \delta \;.
	\end{align*}
	The maximum entropy distribution subject to this second-moment constraint is the $k$-dimensional Gaussian distribution $\mathcal{N}(0, K)$, where $K$ is a diagonal covariance matrix with entries $K_{ii} = \bbE_{M(X),Z}\left[ (f^*(M(X))_{ii} - \bbE_{M(X)|Z}[f^*(M(X))]_{ii})^2 \right]$. The differential entropy of that Gaussian distribution is $\frac{1}{2}\log\left( (2\pi e)^k \det{K} \right)$, and by Jensen's inequality,
	\begin{align*}
		\log\det{K} & = \sum_{i=1}^k \log K_{ii} = \sum_{i=1}^k \log \left( \bbE_{M(X),Z}\left[ (f^*(M(X))_{ii} - \bbE_{M(X)|Z}[f^*(M(X))]_{ii})^2 \right] \right) \\
		& \leq k \log \left( \sum_{i=1}^k  \frac{1}{k} \bbE_{M(X),Z}\left[ (f^*(M(X))_{ii} - \bbE_{M(X)|Z}[f^*(M(X))]_{ii})^2 \right] \right) \\
		& = k\log (\diststinv{}/k) \;.
	\end{align*}

	Putting this together with the first inequality in \eqref{eq:information:chain},
	\begin{align*}
		\op{I}{\rv M(X); \rv Z} & \geq h(M(X)) - \frac{k}{2} \log(2\pi e) - \frac{k}{2}\log (\diststinv{}/k) \\
		& \geq h(M(X)) - \frac{k}{2} \log(2\pi e) - \frac{k}{2}\log(\delta/k) \;.
	\end{align*}
	The same argument holds for either of $g(M(X))$ or $f^*(M(X))$, yielding the stated lower bounds.
\end{proof}

% Note that contrary we do not assume \cref{assumption:discrete} this is possible because MSE allows finite rate when predicting a continuous random variable. with finite distortion.
% \ydnote{
% Initially with \ben{} we discussed about discussing why we can work with $\Riskmse{M(\rv X)}{Z}$ instead of $\max_{f \in L^2(\mathcal{M})} \Riskmse{f(M(\rv X))}{Z}$ if we assume some lipschitz continuity, but actually I don't think it's necessary: this section is really to show that we can deal with other losses and should not be a complete discussion / paper in itself.
% }
%Also note that in practice we can use $\Riskmse{M(\rv X)}{Z}$ as a distortion instead of $\max_{f \in L^2(\mathcal{M})} \Riskmse{f(M(\rv X))}{Z}$, as the former is easier to work with and if $f$ is Lipschitz continuous then minimizes can be seen as minimizing an upper bound of the latter term. \ydnote{I think this is is sufficient, maybe even drop  this.}
% \ydimp{
% Actually I'm not sure that the following is needed
% }
% In practice \cref{thm:rate_invariance_distortion_mse} tells us that the optimal encoder can be learned achieved by optimizing $\argmin_{p(Z|X)} \MI{Z}{X} - \beta(\delta)  \max_{f \in L^2(\mathcal{M})} \Riskmse{f(M(\rv X))}{Z}$. 
% Contrast this with the case of log loss where we had the much easier to work with $\argmin_{p(Z|X)} \MI{Z}{X} - \beta(\delta)  \Risk{M(\rv X)}{Z}$.
% In the following we prove that minimizing $\Riskmse{M(\rv X)}{Z}$ instead of $\max_{f \in L^2(\mathcal{M})} \Riskmse{f(M(\rv X))}{Z}$ for the case of MSE, still gives you strong performance guarantees.

\clearpage
\newpage

\section{Variational objectives}
\label{appx:objectives}

In this section we will derive the variational bounds for estimating the rate and the distortion.
In contrast to the proofs of main theoretical results (previous section) derivations  will be less formal.
Throughout this section we focus on the log loss and implicitly make all assumptions described in \cref{appx:assumptions}.

Recall that the optimal bit-rate is simply the Rate Distortion function using our invariance distortion (Rate-Invariance function; \cref{eq:rate_invariance_constrained} ), so any optimal encoder (for a given $\delta$) can be obtained by using the following arg minimum:
\begin{equation}\label{appx:unconstrained_RD}
 \argmin_{p(\rv Z| \rv X) \ \text{s.t.} \ \Risklog{M(X)}{Z} \leq \delta}  \quad  \op{I}{\rv X; \rv Z}
\end{equation}
As optimization in machine learning is typically unconstrained, we prefer using the following Lagrangian formulation.
\begin{equation}\label{appx:eq:unconstrained_RD}
\argmin_{p(\rv Z| \rv X)}  \quad  \op{I}{\rv X; \rv Z} + \beta \cdot \Risklog{M(\rv X)}{\rv Z}
\end{equation}
Both of these formulations are equivalent in that the set of encoders that minimize \cref{appx:eq:unconstrained_RD} for some $\beta \in \mathbb{R}_{\geq 0} $ is equal to the set of encoders that minimize \cref{appx:unconstrained_RD} for some $\delta \in \mathbb{R}_{\geq 0}$ \cite{everett_generalized_1963,berger_rate-distortion_2003}.

Note that due to the piece-wise linearity of our RI function (\cref{fig:schema_RD}), \citet{kolchinsky_caveats_2019} tells us that sweeping over $\beta \in \mathbb{R}_{\geq 0}$ using \cref{appx:eq:unconstrained_RD} will only enable us to learn the vertices of the RI function, namely, $Rate(\H{M(X)})=0$ for $\beta \in [0,1[$ and $Rate(0)=\H{M(X)}$ for $\beta > 1$,
while any point on the RI curve is simultaneously optimal for $\beta=1$.
In other words, although the solutions of \cref{appx:eq:unconstrained_RD} span the entire RI curve, it is, in theory, not possible to decide which points on the RI curve to obtain by sweeping over beta.
%theoretically, if we sweep over $\beta \neq 1$ there are only two different points of the RI function that we can learn.
\citet{kolchinsky_caveats_2019} shows that this can be easily solved by considering the squared distortion $\Risklog{M(\rv X)}{\rv Z}^2$, in which case sweeping over $\beta$ would be equivalent to sweeping over delta $\delta$ in \cref{appx:unconstrained_RD}.
We did not see any difference in practice so preferred using the more understandable  $\Risklog{M(\rv X)}{\rv Z})$.

Both terms $\MI{ X}{Z}$ and $\Risklog{M(\rv X)}{Z}$ are hard to estimate from samples, so the rest of the section is devoted to deriving variational upper bounds on them.

%%%%%%%%%%%%%%%%%%%%%%%%%%%%%%%%%%%%%%%%%%%%%%%%%%%%%%%%%%
\subsection{Variational upper bound for the rate term \texorpdfstring{$\MI{Z}{X}$}{I[Z;X]}}
\label{appx:variational_rate}

Let us discuss how to approximate the rate term $\MI{X}{Z} $.
The mutual information is well known to be hard to estimate from samples \cite{paninski_estimation_2003,mcallester_formal_2020}, but fortunately many variational bounds have previously proposed, see \citet{poole_variational_2019} for examples.
In the following we denote a family of variational distributions over $\rv Z$ (priors or entropy models) as $\mathcal{Q} \defeq \set{q \in \mathcal{P}(\mathcal{Z})}$.

\subsubsection{Mutual information bottleneck}
\label{appx:mi_bottleneck}

 The first bound that we consider is the standard upper bound on $\op{I}{\rv X; \rv Z}$, \eg, in VAE \cite{kingma_auto-encoding_2014} or VIB \cite{alemi_deep_2017}.
Specifically:
\begin{align}
\MI{Z}{X} &\defeq \H{Z} - \CH{Z}{X} \\
&= \E{p(\rv Z)}{- \log p(\rv Z)} - \CH{Z}{X} \\
&= \inf_{q \in \mathcal{Q}} \E{p(\rv X)p(\rv Z | \rv X)}{- \log \frac{p(\rv Z)q(\rv Z)}{q(\rv Z)} } - \CH{Z}{X} \\
&= \inf_{q \in \mathcal{Q}}  \E{p(\rv X)p(\rv Z | \rv X)}{- \log q(\rv Z)} - \E{p(\rv X)p(\rv Z | \rv X)}{ \log \frac{p(\rv Z)}{q(\rv Z)} }   - \CH{Z}{X} \\
&= \inf_{q \in \mathcal{Q}}  \E{p(\rv X)p(\rv Z | \rv X)}{- \log q(\rv Z)} - \KL{p(\rv Z),q(\rv Z)}   - \CH{Z}{X} \\
&\leq \inf_{q \in \mathcal{Q}}   \E{p(\rv X)p(\rv Z | \rv X)}{- \log q(\rv Z)}  - \CH{Z}{X} \label{eq:IZX_bound}
\end{align}
The approximation gap is then $\min_{q \in \mathcal{Q}} \KL{p(\rv Z) , q(\rv Z)}$.
The bound has the advantage that if $p(\rv Z) \in \mathcal{Q}$ then the bound is tight.
The major issue with the mutual information bottleneck, is that no efficient compressors can in general achieve the rate given by it \citep{agustsson_universally_2020}.%
\footnote{
See \citet{flamich_compressing_2020} or \citet{schulman_sending_2020} for an $\mathcal{O}(\exp(\MI{Z}{X}))$ algorithm.
}
For example, if we decided to entropy code $\rv Z$ using the entropy model $q(Z)$ then we would achieve $\E{p(\rv X)p(\rv Z | \rv X)}{- \log q(\rv Z)}$ bits which is $\CH{Z}{X}$ more than what is given by our bound.%
\footnote{
Bits-back coding \citep{wallace_classification_1990} can efficiently reach the desired bit-rate only because it is in the lossless setting.
}

\subsubsection{Entropy bottleneck}
\label{appx:entropy_bottleneck}

One specific case of the mutual information bottleneck which enables efficient compression, is when $\rv Z$ is discrete and arises from a deterministic transformation of $\rv X$. 
Indeed, in this case $\CH{Z}{X}=0$ so entropy coding (\eg \cite{rissanen_generalized_1976,duda_asymmetric_2009}) can reach the rate given by our bound.
Using the same derivation as for the mutual information bottleneck, we get,
\begin{equation}\label{eq:HZ_bound}
\MI{Z}{X} = \H{Z}  \leq \inf_{q \in \mathcal{Q}} \E{ p(\rv X)p(\rv Z | \rv X) } {- \log q(\rv z)}.
\end{equation}
This is the standard bound used in neural compressors \cite{balle_end--end_2017,theis_lossy_2017}.
The entropy bottleneck bound has the following downsides compared to mutual information bottleneck:
\begin{itemize}
\item It is generally not true that for any $\delta$ the optimal rate can be achieved by a discrete and deterministic $\rv Z$.
For the specific case of $\delta=0$ and with \cref{assumption:discrete} it is the case, as we can simply set $\rv Z = M(\rv X)$.
\item It is not suitable for gradient based optimization w.r.t. to the encoder (due to the discreteness of $\rv Z$) so we typically have to add noise during training \cite{balle_end--end_2017} which can cause a mismatch between training and testing \cite{agustsson_universally_2020}.
\end{itemize}

Despite these issues we will mostly use the entropy bottleneck bound in experiments as we want our method to give rise to practical compressors.

%%%%%%%%%%%%%%%%%%%%%%%%%%%%%%%%%%%%%%%%%%%%%%%%%%%%%%%%%%
\subsection{Variational upper bound for the distortion term 
\texorpdfstring{$\Risklog{M(X)}{Z}$}{R[M(X)|Z]}}
\label{appx:variational_distortion}

Let us now consider variational upper-bounds on the distortion $\Risklog{M(X)}{Z}$.
For conciseness we will consider the same setting as in the main paper, \ie, log loss risk and countably many equivalence classes (\cref{assumption:discrete}).
But it is easy to see that the direct distortion bound generalizes to any loss without \cref{assumption:discrete}. \footnote{$\Riskl{M(X)}{Z} \leq \inf_{h \in \mathcal{H}'} \E{p(\rv X)p(\rv Z |\rv X)}{L(M(X),h(Z))}$ which comes from the fact that we are taking an $\inf$ over a subset $\mathcal{H}'$ of all possible predictors.}

\subsubsection{Direct distortion}
\label{appx:direct_dist}

The obvious variational bound on the Bayes risk is the Bayes risk constrained to some functional family. 
$\mathcal{Q}' \defeq \set{q : \mathcal{Z} \to  \mathcal{P}(\mathcal{M})}$ denotes a variational family of regular conditional distributions (decoders), then, 
\begin{align}\label{eq:direct_distortion_bound}
\Risklog{M(X)}{Z} &\defeq \inf_{q} \E{p(\rv X)p(\rv Z |\rv X)} {- \log q(M(\rv x) \cond \rv z)}  \\
&\leq \inf_{q' \in \mathcal{Q}'} \E{p(\rv X)p(\rv Z |\rv X)} {- \log q'(M(\rv x) \cond \rv z)}
\end{align}
which comes from the fact that we are taking an $\inf$ over a subset $\mathcal{Q}'$ of all possible distribution.
A simple derivation shows that the approximation gap is $\min_{q' \in \mathcal{Q}'} \operatorname{D}_{\text{KL}}\left[p(M(\rv x),\rv z)  \|
 q'(M(\rv x) \cond \rv z)p(\rv Z) 
\right]$, so the bound is tight if $p(M(\rv x) \cond \rv z) \in \mathcal{Q}'$.
This direct distortion is simple, but typical variational families will require predicting (``reconstructing'') an expected prediction $\E{q'(M(X)|Z)}{M(X)|Z}$ which is challenging when  $\mathcal{M}$ is in high dimension (\eg unaugmented images).

\subsubsection{Contrastive distortion}
\label{appx:contrastive_dist}

We now consider a bound that does not require explicitely predicting $M(X)$, by considering a noise contrastive estimator \cite{gutmann_noise-contrastive_2010}.
Suppose that for any $\rv Z$ we can sample from a sequence $\mathbf{M} = (M^+, M^-_1, \dots, M^-_n)$, where $M^+ \dsim p(M(\rv X) \cond Z)$ and $\set{M^-_i}_{i=1}^n \iidsim p(M(\rv X))$.
%Suppose that we have access to samples from the ``positive'' pair $(Z,M(X)^+) \dsim p(\rv Z)p(M(\rv X) \cond Z)$ and from the ``negative'' pair $(Z,M(X)^-) \dsim p(\rv Z)p(M(\rv X))$.
Let $\mathcal{F} \defeq \set{f : \mathcal{M} \times \mathcal{Z} \to \mathbb{R}}$ be a variational family of discriminators which is used scores how likely $M(X),\rv Z$ are to be sampled from the joint $p(M(\rv X),\rv Z)$ rather than the product of the marginal $p(M(\rv X))p(\rv Z)$, then,
\begin{align}\label{eq:contrastive_distortion_bound}
lhs &= \Risklog{M(X)}{Z}\\
&= \CH{M(X)}{Z} & \text{\cref{lemma:risk_entropy}}\\
&= \H{M(X)} - \MI{M(X)}{Z} \label{eq:MIMXz} \\
&\leq \H{M(X)} - \inf_{f \in \mathcal{F}}  \E{p(\rv Z) p(\boldsymbol{\rv M} | \rv Z)}{\log n + \log \frac{\exp f(\rv M^+, \rv Z)}{\sum_{\rv M' \in \mathbf{M}} \exp f(\rv M', \rv Z) }}  \label{eq:infonce_bound} & \text{InfoNCE} \\
&= \inf_{f \in \mathcal{F}}  \E{p(\rv Z) p(\boldsymbol{\rv M} | \rv Z)}{- \log \frac{\exp f(\rv M^+, \rv Z)}{\sum_{\rv M' \in \mathbf{M}} \exp f(\rv M', \rv Z) }} + (const)   
\end{align}
\Cref{eq:infonce_bound} uses InfoNCE \cite{oord_representation_2019}, which is a lower bound on mutual information \cite{poole_variational_2019,song_multi-label_2020}.
The last equation removes constants w.r.t. $p(\rv Z \cond \rv X)$ and $\mathcal{F}$, as these terms do not have to be optimized over.
We see that we are only left with a log softmax term
\footnote{Taking exponentials is not necessary, any function $g : \mathcal{M} \to \mathbb{R}_{\geq 0}$ would work as a discriminator, we use $g \defeq \exp \circ f$ to ensure positivity as this has a nice softmax interpretation and is standard in practice.
Our derivation is still general as we can set $f \defeq log \circ g$.
}
that essentially aims to classify which of all the $M'$ comes from the $p(M(\rv X) \cond \rv Z)$.
The bound is tight if the variational family $\mathcal{F}$ contains the optimal predictor 
and as the number of negatives tends to infinity.
For a detailed discussion about noise contrastive estimation under the log loss, refer to \cite{gutmann_noise-contrastive_2010,ma_noise_2018,rhodes_variational_2019}.

Note that the contrastive distortion has the advantage of not having to reconstruct high dimensional data (\eg for images), but it suffers from bias in the case where the number of negatives $n$ is small \cite{poole_variational_2019}.

One additional derivation which we will need in the following section, is that an upper bound can also be obtained by replacing $M(X)$ by any other r.v. $U$ s.t. $U - M(X) - \rv Z$ forms a Markov Chain. 
Indeed starting from \cref{eq:MIMXz}, we have,
\begin{align}\label{eq:contrastive_distortion_bound_markov}
lhs &= \H{M(X)} - \MI{M(X)}{Z} \\
&\leq \H{M(X)} - \MI{\rv U}{Z} & \text{DPI}\\
&\leq \inf_{f \in \mathcal{F}}  \E{p(\rv Z) p(\boldsymbol{\rv U} | \rv Z)}{- \log \frac{\exp f(\rv U^+, \rv Z)}{\sum_{\rv U' \in \mathbf{U}} \exp f(\rv U', \rv Z) }} + (const)   
\end{align}
The bound can still be tight if in addition we have the following Markov Chain $ M(X) - U -\rv Z$, which implies that $\MI{\rv U}{Z} = \MI{M(X)}{Z}$.

%%%%%%%%%%%%%%%%%%%%%%%%%%%%%%%%%%%%%%%%%%%%%%%%%%%%%%%%%%
\subsection{Case study: VIC and BINCE under data augmentations}
\label{appx:vic_bince}

The derivations in the previous 2 subsections are relatively general and abstract.
As a case study, we now discuss the two objectives that we propose in the main paper for the case where we have access to the desired data augmentation  and where we use neural functional families.
Namely, the variational families for the entropy model $p_{\phi}(\rv Z)$, the encoder $p_{\varphi}(\rv Z \cond \rv X)$, the decoder $p_{\phi}(M(\rv X) \cond \rv Z)$, and the discriminator $f_{\psi}(\rv X , \rv Z)$ are all parametric neural families.
Throughout this subsection we will consider that we only have access to $p(\rv X)$ through a dataset $\mathcal{D} \defeq \set{x_i}_{i}$ of samples which were independently sampled from $p(\rv X)$.

Let us formalize what we mean by having access to the correct data augmentations.
Let us denote as $\rv A$ the r.v. over a set of augmentations $a : \mathcal{X} \to \Tilde{\mathcal{X}}$, \ie, a stochastic process.
Let $\trv X = A(X)$ be the augmented source.
Note that by $A(\rv X)$ we mean the r.v. which arises by sampling an augmentation $a$ from the stochastic process $p(A)$, and then applying it to some samples $x$ from $\rv X$.
\begin{assumption}[Augmentations]\label{assumption:augmentations}
We assume knowledge of a random augmentation generator $\rv A$ that satisfies the following two key properties
\begin{itemize}
\item \textbf{Retain invariance}. We require $A$ to retain the invariance structure to $(\sim,\mathcal{X})$, specifically, $X \sim A(X)$ almost surely.
\item \textbf{Remove information}. We require $A$ to remove as much information as possible about the input. 
Specifically, $X \indep A(X) \cond M(X) $ almost surely.
\end{itemize}
\end{assumption}
The first requirement is simple but clearly not sufficient.
For example, the identity function does satisfy such requirement for any equivalence relation ($X \sim X$ by definition), yet it does not correspond to what we think as an augmentation because it does not remove any information about the input.
The second requirement formalizes exactly what is required, namely that the augmentation must remove all information about the input besides the knowledge about its equivalence class (which is needed for the first requirement).
%$p(X|\tilde{X} = x) = p(X|\tilde{X} = x,M(x)) =  p(X|M(x))=  p(X|M(x')) =  p(X|M(x'),\tilde{X} = x') =  p(X|\tilde{X} = x')$ where the first and last equality come from the ``retain invariance'' requirement, the second and penulimate equality come from ``remove information'', and the third equality come from the equivalence of $x$ and $x'$.

The first requirement will typically hold. 
The second in more stringent.
Note that it holds if for all equivalent examples $x\sim x'$ in $\mathcal{X}$ we have $p(A(x)) = p(A(x'))$.
Indeed $p(A(x)) =p(\trv X|x) =p(\trv X|x,M(x))$ and similarly $p(A(x')) =p(\trv X|x',M(x'))$, using $M(x)=M(x')$ we have $p(\trv X|x,M(x)) =p(\trv X|x',M(x))$ for all equivalent $x,x'$ which implies that $X \indep \trv X \cond M(X) $ as desired.
In practice this only needs to hold for examples in our datasets, \ie, the second requirement holds if for all equivalent $x,x' \in \mathcal{D}$ in a dataset we have $p(A(x)) = p(A(x'))$.
% \footnote{
% To see why this is the case notice that the Markov Chain $\trv{X} - M(\trv{X}) - X $ means that $p(X|M(\trv X),\trv X) = p(X|M(\trv X))$.
% As $M(\trv X)$ is a function of $\trv X$, we can simplify that to $p(X|\trv X) = p(X|M(\trv X))$.
% As the maximal invariant of any equivalent examples is by definition the same, we need that for all $ \tilde{x},\tilde{x}' \in Supp(\tilde{X})$ we have $p(X|\tilde{x}) = p(X|\tilde{x}')$.
% }
This is likely to hold in practice as the number of examples that are equivalent in a dataset will be small if $|\mathcal{X}| \gg |\mathcal{D}|$ as is typically the case.
In particular, if a dataset does not contain any equivalent examples, \ie, for any $x,x' \in \mathcal{D}$ we have $x \not\sim x'$ then the requirement trivially holds.

% Let us explain what we mean by access to the correct data augmentations.
% Intuitively, we want the augmentations to be invariant to $\sim$ but no other finer equivalence relations.
% In other words $\sim$ should be the coarsest equivalence relation that our augmentations are invariant to.
% Specifically, we assume that we have access to a r.v. $\rv A$ (a stochastic process) that t.v.i in a set of augmentations $\mathcal{A}$ that induces the $(\sim,\mathcal{X})$ equivalence w.r.t. which we are invariant. 
% By this we mean that $\mathcal{A}$ is a set of functions $a : \mathcal{X} \to \mathcal{X} $ s.t. for all $x \in \mathcal{X}$ we have that the equivalence class $[x]$ is equal to all examples that can be obtained by applying to $x$ any sequence $s_{\mathcal{A}}$ of composed images and preimages by aumgentations  in $\mathcal{A}$, \ie,  $[x] = \bigcup_{s_{\mathcal{A}}}  \op{s_{\mathcal{A}}}{\set{x}}$.
% For example, one possible sequence of image and preimages might be $s_{\mathcal{A}}' = a_i \circ a_j^{-1} \circ a_k^{-1}$.
% A more computational perspective consists in building a graph of all examples in $\mathcal{X}$ with a directional edge between  any $x$ and $x'$ iff there is an augmentation in $\mathcal{A}$ from $x$ to $x'$.
% Then we say that $\mathcal{A}$ induces $(\sim,\mathcal{X})$ if the set of connected components of the corresponding undirected graph are equal to set of equivalence class induced by $\sim$.
% \ydnote{\ben{} it seems like such notion should already exist. Anything that comes to mind ?
% }

%Access to augmentation that induce the equivalence relation 

\subsection{Issue: dealing with unknown \texorpdfstring{$M(X)$}{M(X)}}
\label{appx:unkown_M}

One issue which arises in practice is that we generally do not have access to $M(X)$.
We will overcome this issue by taking advantage of the fact that we have access to data augmentations $\rv A$  that induce our equivalence relation.
Intuitively, we will treat the augmented r.v. $\trv{X} \defeq A(\rv X)$ as if it were the source, and use the actual source $\rv X$ instead of $M(\trv{X})$.
Note that by $A(\rv X)$ we mean the r.v. which arises by sampling an augmentation $a$ from the stochastic process $p(A)$, and then applying it to some samples $x$ from $\rv X$.
From now on, let us denote as $\rv{Z} \dsim p(Z|\trv{X})$ the representation that arises from the augmented source.
\ydnote{Would be slightly more precise to introduce a new letter $\trv{Z}$ but I feel it will confuse people}
Under suitable conditions on $\rv A$ we can replace the previous objective \cref{appx:eq:unconstrained_RD} with the following equivalent objective, which we denote as $\Lri^\beta$, 
\begin{equation}\label{appx:eq:unconstrained_RD_no_Mx}
\argmin_{p(\rv Z | A(\rv X) )}  \quad  \MI{A(\rv{X})}{\rv{Z}} + \beta \cdot \Risklog{X}{\rv{Z}}.
\end{equation}
By equivalence of those objectives we mean that for any $\beta \in \mathbb{R}_{\geq 0}$ the set of RD tuples $(\MI{X}{ Z},\Risklog{M(X)}{Z})$ that are achieved by solutions of \cref{appx:eq:unconstrained_RD}  is equal to the set of RD tuples $(\MI{\trv X}{ Z},\Risklog{X}{Z})$ that are achieved by solutions of \cref{appx:eq:unconstrained_RD_no_Mx}.
In other words, they generate the same segment of the RI function.

First, let us show why we can replace $\rv X$ by $\trv{  X}$, \ie, show that for any $\beta$  we have that $\argmin_{p(Z| X)}  \quad  \MI{X}{ Z} + \cdot \Risklog{M(X)}{Z}$ is equivalent to $\argmin_{p(\rv{Z}| \trv{X})}  \quad  \MI{\trv{X}}{ \rv{Z}} + \beta \Risklog{M(\trv{X})}{\rv{Z}}$.
This can be seen from the fact that the optimal bit rate in \cref{thm:rate_invariance_distortion} only depends on $\H{M(X)}$ and $\delta$.
In particular, $Rate(\delta)$ does not depend on the distribution of the source inside the equivalence classes, $p(X|M(X))$.
Indeed, an optimal representation will compress all that information.%
\footnote{Note that the bit rate gains $\CH{X}{M(X)} - \delta$ clearly depend on $p(X|M(X))$, but not the actual bit-rate $\H{M(X)} - \delta$.}
As a result, we can attain the same optimal bit rate by considering any source $\trv{\rv X}$ that is a transformed version of $X$ as long as the transformation does not change the distribution of the maximal invariant, \ie, $p(M(\rv X))=p(M(\trv{ X}))$.
This is clearly the case for $\trv{X} \defeq A(X)$ as our augmentation retains invariance (\cref{assumption:augmentations}).

Now let us consider why and when replacing $M(\trv{X})$ by $\rv X$ makes sense.
%This is essentially valid if $\rv A$ is s.t. given some augmented outcome $\Tilde{x} = A(x)$ you will not be able to say from which $x$ it was generated from.
Using \cref{prop:dist_no_Mx} we know that if $\rv A$ is s.t. $\trv{X} \sim X $ and $\trv{X} \perp X \cond M(X) $ forms a Markov Chain then we can replace (up to constants which are not important for $\argmin$) the distortion term
$\Risklog{M(\trv{X})}{\rv{Z}}$ by $\Risklog{X}{\rv{Z}}$. 
These are exactly our requirements on augmentations (\cref{assumption:augmentations}).

For the rest of this section we will thus be working with $\Lri^\beta$ (\cref{appx:eq:unconstrained_RD_no_Mx}) instead of \cref{appx:eq:unconstrained_RD}.
Note that this means that, in theory, we should always use the augmented $\Tilde{X}$ from now on, \ie, not only at train time but also at test time.

\subsubsection{Variational Invariant Compressor (VIC)}
\label{appx:vic}

As seen in the main text the VIC loss is essentially a neural compressor where inputs are augmented but not the target reconstructions.
We derive it by combining our entropy bottleneck bound (\cref{appx:entropy_bottleneck}) and our direct distortion (\cref{appx:contrastive_dist}), which gives the following upper bound on $\Lri{}^\beta$,
\begin{align}
\Lri{}^\beta(\varphi) 
&\defeq \MI{A(\rv X)}{\rv Z} + \beta \cdot \Risklog{\rv X}{\rv Z} \\
&\leq \inf_{\theta} \E{ p(\rv A)p(\rv X)p_{\varphi}(\rv Z | A(\rv X)) } {- \log q_{\theta}(\rv z)} & \text{\cref{eq:HZ_bound} }\\
&+ \beta \cdot \inf_{\phi} \E{p(\rv A)p(\rv X)p_{\varphi}(\rv Z |A(\rv X))} {- \log q_{\phi}(\rv X \cond \rv z)} & \text{\cref{eq:direct_distortion_bound}} \\
&=  \inf_{\theta,\phi} - \E{ p(\rv A)p(\rv X)p_{\varphi}(\rv Z | A(\rv X)) } {  \log q_{\theta}(\rv z) + \beta \log q_{\phi}(\rv X \cond \rv z)}  \label{eq:upperbound_vic}
\end{align}
Using a Monte Carlo estimate for the expectation over $p(\rv X)$, we get our desired objective,
\begin{align}
\Livae{}^\beta(\phi, \theta, \varphi)
\defeq 
-
\frac{1}{|\mathcal{D}|}\sum_{x \in \mathcal{D}}
\E{p(\rv A)p_\varphi(\rv Z|A(x))}{
\log q_\theta(\rv Z) + \beta \cdot \log q_\phi(x \cond \rv Z)}.
\end{align}
In practice, we approximate the expectation over $\rv A$ and $\rv Z$ using a single sample for computational efficiency.
A full algorithm is provided in \cref{alg:vic} and illustrated in \cref{fig:objectives} of the main text.

\begin{algorithm}[H]
\caption{Variational Invariant Compressor (VIC). Single sample forward pass.}
\label{alg:vic}
    \begin{algorithmic}[1]
    \Require Encoder $p_{\varphi}(\rv Z | A(\rv X))$,  Entropy Model $q_{\theta}(\rv Z)$,  Decoder $q_{\phi}(\rv X | \rv Z)$
    \Require Dataset $\mathcal{D}$, random augmentation generator $\rv A$, Lagrange multiplier $\beta$
    \State $x  \leftarrow  \mathrm{select}(\mathcal{D})$  \Comment{sample}
    \State $\Tilde{x} \leftarrow \text{sample}(A(x))$ \Comment{random augment}
    \State $z \leftarrow \mathrm{sample}(p_{\varphi}(\rv Z | \Tilde{x} ))$ \Comment{encode}
    \State $\mathrm{rate\_loss} \leftarrow - \log q_{\theta}(z) $ \Comment{Entropy Bottleneck}
    \State $\mathrm{distortion\_loss}  \leftarrow - \log q_{\phi}(x | z ) $ \Comment{Direct Distortion} \\
    \Return $\mathrm{rate\_loss} + \beta \cdot \mathrm{distortion\_loss}$
\end{algorithmic}
\end{algorithm}

Note that $\Livae{}^\beta$ tends to $\Lri{}^\beta$ when using unconstrained variational families and as the dataset grows to infinity.
This essentially shows that VIC objective (with infinite samples and unconstrained families) will learn the optimal deterministic and discrete $\rv Z$ (as discussed in \cref{appx:entropy_bottleneck}), in particular, when $\beta > 1$ it will learn an encoder which is optimal for the lossless prediction regime.

\subsubsection{Bottleneck InfoNCE (BINCE)}

VIC for images and data augmentation suffers from the issue that it needs a predictor which reconstructs a high dimensional image (as discussed in \cref{appx:direct_dist}).
To solve this issue we discuss our BINCE objective, which as seen in the main text, is essentially a standard contrastive self-supervised (SSL) objective with an additional entropy bottleneck.
We derive it by combining our entropy bottleneck bound (\cref{appx:entropy_bottleneck}) and our contrastive distortion (\cref{appx:contrastive_dist}).

\ydnote{Sorry this paragraph is not easy to explain, also from what I know we are the first ones to derive that properly and show that all of these deep learning tricks are actually valid.
I'm hesitating to use the infamous "proof left to the reader'' because no time / space to typeset everything and I know it's true (proof in my notes)
}
Note that in \cref{appx:contrastive_dist} for each $Z$ we needed a sequence $\textbf{M}$ of outcomes of $M(X)$ that are sampled either from the conditional $p(\rv M(X) | \rv Z)$ or the marginal  $p(\rv M(X))$.
As we will replace $M(X)$ by $X$ ( see \cref{appx:unkown_M}) we now need a sequence of r.v. $\textbf{X} \defeq (X^+,X^-_1,\dots,X^-_n)$  s.t. $X^+$ is  sampled from the conditional $p(\rv X|\rv Z)$ while each $X^-_i$ are independently sampled from the marginal $p(\rv X)$.
Furthermore, as is standard in self-supervised learning (\eg \cite{chen_simple_2020,oord_representation_2019}) we will actually be using a sequence $\check{\textbf{Z}}$ of positive and negative representations instead of $\textbf{X}$.
We do so by independently augmenting and encoding each r.v. in $\textbf{X}$.
Using our requirement on the augmentations (\cref{assumption:augmentations}) we thus have the following Markov Chain $\check{Z} - X - M(X) - A(X) - \rv Z$.
As a result, we can use $\check{\mathbf{Z}}$ instead of $\textbf{X}$ in InfoNCE (see \cref{eq:contrastive_distortion_bound_markov}).
For conciseness we will denote the above sampling procedure as $p_\varphi(Z,\check{\boldsymbol{Z}} \cond A, \rv X)$.
We then have the following upper bound on $\Lri{}^\beta$,
\begin{align}
\Lri{}^\beta(\varphi) 
&\defeq \op{I}{A(\rv X); \rv Z} + \beta \cdot \Risklog{X}{\rv Z} \\
&\leq \inf_{\theta} \E{ p(\rv A)p(\rv X)p_{\varphi}(\rv Z | A(\rv X)) } {- \log q_{\theta}(\rv z)} + (const)    \\
&+ \beta \cdot \inf_{\psi}  \E{p(A)p(X)p_{\varphi}(\rv Z | A(\rv X)) p(\check{\boldsymbol{\rv Z}} | \rv Z)}{- \log \frac{\exp f_\psi(\check Z^+, \rv Z)}{\sum_{\check Z' \in \check{\mathbf{Z}}} \exp f_\psi(\check Z', \rv Z) }} \\
&=  \inf_{\theta,\psi} - \E{ p(\rv A)p(\rv X)p_\varphi(Z,\check{\boldsymbol{Z}} \cond A, \rv X) } {  \log q_{\theta}(\rv z) + \beta \log \frac{\exp f_\psi(\check Z^+, \rv Z)}{\sum_{\check{Z}' \in \check{\mathbf{Z}}} \exp f_\psi(\check{Z}', \rv Z) }}  \label{eq:upperbound_bince} 
\end{align}
Using a Monte Carlo estimate for the expectation over $p(\rv X)$, we get our desired objective,
\begin{equation}
\Lbince{}(\varphi, \theta, \psi) \defeq
- 
\sum_{x \in \mathcal{D}}  \E{p(A) p_{\varphi}(Z, \check{\boldsymbol{\rv Z}}| A,\mathcal{D},x) }{ \log q_\theta(\rv Z) 
+ \beta \cdot \log \frac{\exp f_\psi(\check Z^{+}, \rv Z)}{\sum_{\check Z' \in \check{\mathbf{Z}}} \exp f_\psi(\check Z', \rv Z) }}.
\end{equation}
In practice we approximate the expectation over $\rv A$,$\check{\boldsymbol{\rv Z}}$ and $\rv Z$ using a single sample for computational efficiency.
Just as with VIC we have that $\Livae{}^\beta$ tends to $\Lri{}^\beta$ when using unconstrained variational families and as the dataset and number of negatives $n$ grows to infinity.

\ydnote{Should batch forward pass really be in this section?}
In the main paper we provided a simple algorithm (\cref{alg:BINCE}) to compute BINCE for a single example $x \in \mathcal{D}$.
This is computationally intensive as it requires sampling one sequence of r.v. for each example.
In practice, this is nevertheless easily amenable to batch computation.
Indeed, negative representations $\check Z^-$  are positive representations $\check Z^+$ for a different example.
As a result, we can first sample a batch  $\mathbf{x} \defeq (x_1, \dots, x_{n})$ from $\mathcal{D}$.
Then augment it to two different sequences $\Tilde{\mathbf x},\Tilde{\mathbf x}'$.
And finally represent each sequences to obtain $\mathbf{z},\mathbf{z}'$.
Then for any $z_i \defeq \mathbf{z}[i]$ we have that $z_i' \defeq \mathbf{z}'[i]$ is a positive example while all other $z \in \mathbf{z},\mathbf{z}'$ are negatives.
We thus only need to sample a single representation per example in the dataset.
A full algorithm for batch computations is provided in \cref{alg:BINCE_batch} (using only one of the 2 augmented batches for notational convenience).

\begin{algorithm}[t]
\caption{Batch forward pass for BINCE}
\label{alg:BINCE_batch}
    \begin{algorithmic}[1]
    \Require encoder $p_{\varphi}(\rv Z | A(\rv X))$,  entropy model $q_{\theta}(\rv Z)$, discriminator $f_{\psi}(\rv Z', \rv Z)$
    \Require augmentations $\rv A$, data $\mathcal{D}$, Lagrangian coefficient $\beta$, batch size $b$ 
    \State $\mathbf{x} \leftarrow  \mathrm{select}(\mathcal{D})\ b$ times   \Comment{sample}
    \State $\tilde{\mathbf{x}} \leftarrow \text{sample})(A(\mathbf{x}))$ \Comment{Random augment 1}
    \State $\tilde{\mathbf{x}}' \leftarrow \text{sample})(A(\mathbf{x}))$ \Comment{Random augment 2}
    \State $\mathbf{z},\mathbf{z}' \leftarrow \mathrm{sample}(p_{\varphi}(\rv Z | \tilde{\mathbf{x}}  )), \mathrm{sample}(p_{\varphi}(\rv Z | \tilde{\mathbf{x}}'  ))$ \Comment{Encode}
    \State $\mathbf{zs} \leftarrow \mathrm{concat}(\mathbf{z},\mathbf{z}') $
    \State $\mathrm{rate\_loss}  \leftarrow$ average $ - \log q_{\theta}(z_i)$ over $z_i \in \mathbf{zs}$ \Comment{Entropy Bottleneck}
    \State $\mathrm{distortion\_loss}  \leftarrow 0$
    \For{$i\leftarrow 1, \dots, b$}
    \State $z^+ \leftarrow \mathbf{z}'[i]$ \Comment{Select positive} 
    \State $\mathrm{softmax} \leftarrow \exp f_\psi( z^{+}, z) / (\sum_{z' \in \mathbf{zs}} \exp f_\psi( z',  z))$ \Comment{Softmax} 
    \State $\mathrm{distortion\_loss}  \mathrel{{-}{=}} \frac{1}{b}  \log (\mathrm{softmax}) $ \Comment{Contrastive Distortion} 
    \EndFor{}
    \Return $\mathrm{rate\_loss} + \beta \cdot \mathrm{distortion\_loss}$
\end{algorithmic}
\end{algorithm}

%%%%%%%%%%%%%%%%%%%%%%%%%%%%%%%%%%%%%%%%%%%%%%%%%%%%%%%%%%
\subsection{CLIP as BINCE's distortion}
\label{appx:clip_bince}

One of our main experiment (\cref{sec:clip_experiments}), consists in using a pretrained CLIP to make a powerful image compressor.
We are able to do so by realizing that CLIP essentially corresponds to BINCE's distortion (second term in \cref{eq:bottlenecked_simclr} with the following choices: 
\begin{itemize}
\item \textbf{Augmentation}: text to image transformation. CLIP’s dataset contains pairs of associated images and detailed text  $(x_{img},x_{txt})$. The ``augmentation'' is then a function that maps $x_{img}$ to its associated  $x_{txt}$ and vis versa.
This will partition the joint image-text space of $\mathcal{X} \defeq \mathcal{X}_{img} \cup \mathcal{X}_{txt}$ into sets, each of which are associated (directly or by transitivity) with a common text description or image.
\item \textbf{Discriminator}: a dot product, i.e, $f_{\psi}(Z’,Z) = Z’^T Z$.
\item \textbf{Encoder}: a deterministic function defined by cases. Specifically, sampling from $p(Z|X)$ gives the output of the visual transformer (image encoder) $Z = ViT(X)$ if $X$ is an image and the output of the text transformer (text encoder) $Z = transformer(X)$ if $X$ is a sentence.
\end{itemize}

The only minor difference is that CLIP performs contrastive learning between text-image and image-text but never text-text and image-image. BINCE would instead make no distinction between modalities as the equivalence class is on the joint image and text space. Both are nevertheless valid approximations to $\Risk{M(X)}{Z}$.

Although CLIP's augmentation will always give rise to a valid equivalence relation, 
it would in theory recover the degenerate solution of $[x] = \mathcal{X}$ for all $x \in \mathcal{X}$ if the dataset was ``infinite''.
Indeed, any image could possible just be described by the text ``an image'', which would recover the aforementioned degenerate solution. 
There are different ways of collecting the datasets that could avoid this issue, \eg, ensuring that the description is more precise than that.
In practice, this is unlikely to be an issue as the dataset is finite.

Another possible theoretical issue of CLIP's augmentation/equivalence structure, is that it is likely that very few images have a common associated text in CLIP's dataset (or vis-versa). 
In theory, this would thus recover the degenerate solution where no points are equivalent to one another, \ie, $[x] = \set{x}$ for all $x \in \mathcal{X}$.
In practice, this issue is probably avoided due to the fact that images will get clustered as long as the the text description is similar enough for the text encoder to provide (essentially) the same text encoding (due to computational/architectural constraints). 
I.e., the images will actually get partitioned based on the value of the \textit{representation} of their associated text rather than the text itself.  

\clearpage
\newpage

\section{Extended related work}
\label{appx:related}
\paragraph{Invariances and symmetries}
 Invariances are ubiquitous in ML, as seen by the use of data augmentations \cite{shorten_survey_2019} and invariant models \cite{shawe-taylor_building_1989,wood_representation_1996,bruna_invariant_2013,cohen_group_2016,zaheer_deep_2017,kondor_generalization_2018,bloem-reddy_probabilistic_2020}.
These force models not to rely on nuisances to improve generalization \cite{dao_kernel_2019,chen_group-theoretic_2020,lyle_benefits_2020}.
We directly discard such nuisances from the data to improve compression.
%In contrast to us, the usual group theory formalization of invariances in ML cannot work with standard data augmentation (e.g. cropping).
Others have used symmetries in $\rv X$ for lossless compression of multisets \cite{varshney_benefiting_2007}, graphs \cite{choi_compression_2012,dehmer_history_2011,kontoyiannis_compression_2020}, or structured images \cite{sanchez_symmetry-based_2009,amraee_compression_2011,mitra_symmetry_2013,gnutti_representation_2015,bairagi_symmetry-based_2015}.
We, instead, use invariance of the tasks $\rv Y$ for lossless prediction.

\paragraph{Neural lossy compression}
Most research in neural compression is either focused on estimation and optimization of the rate term \cite{chen_variational_2017,minnen_joint_2018,johnston_computationally_2019,yang_improving_2020,yang_variational_2020,minnen_channel-wise_2020,lee_context-adaptive_2019,agustsson_universally_2020} or on developing perceptually meaningful distortions \cite{blau_rethinking_2019,chen_perceptually_2020,agustsson_generative_2019,mentzer_high-fidelity_2020}.
Our paper also develops a new distortion, but does not optimize for perception.
Improvements in the rate objectives are orthogonal to our work and can also help our method.
%Improvements in the rate objectives also help our method, and are thus orthogonal to our work.
%The latter line of work on perceptual distortions answer a different question to ours.

\paragraph{Self-supervised learning}
Our objective (\cref{eq:unconstrained_no_Mx} in main text)  can be seen as contrastive SSL \cite{oord_representation_2019,chen_simple_2020} with an information bottleneck,
a version of \cite{zbontar_barlow_2021,bardes_vicreg_2021} with an information instead of a variance bottleneck,
a SSL VIB \cite{alemi_deep_2017}, or
an invariant VAE \cite{kingma_auto-encoding_2014}.
At a higher level our work differs on two key aspects.
First, minimizing the information $\MI{A(X)}{Z}$ %between the representation and the inputs 
arises from our desire to perform compression rather than to (optimally \cite{dubois_learning_2020}) help generalization \cite{shamir_learning_2010,vera_role_2018}.
Second, we provide the first formalism of a minimal pretext task $M(X)$ that retains all information about any invariant task.
This is related to the multi-view literature, where one only needs to retain information which is invariant across views \cite{sridharan_information_2008,tosh_contrastive_2021,lee_predicting_2020,tsai_self-supervised_2021}.
The main difference is that we prove the existence of a single pretext task.

The most similar setting to ours is the recent work of \citet{mitrovic_representation_2021}, which (in Appx. D) analyses contrastive learning using equivalence relations. 
Specifically, they also consider tasks $Y$ whose conditional distribution are invariant to an equivalence relation.
Their Theorem 1 is then similar to our \cref{lemma:desintegration}, but only considers the restricted case of deterministic labeling and finite sample space $\mathcal{X}$. Furthermore, they only talk about invariant representations (sufficiency), while we characterize all invariant representations using the maximal invariant (necessity and sufficiency).
%., while possibly removing information that differs across views  \cite{}.

\paragraph{Information theory and predictions}
\cref{thm:rate_invariance_distortion} relates exactly predictive loss and compression rate.
Although such results is to our knowledge (surprisingly) new, it fits in a long line of work that relates Bayes predictions and generalized information theory \cite{degroot_uncertainty_1962,grunwald_game_2004,gneiting_strictly_2007,duchi_multiclass_2018,farnia_minimax_2016,dubois_learning_2020,xu_minimum_2020}.

\paragraph{Maximal invariants and minimal sufficient statistics}
%At the core of our work is the notion of maximal invariance $M(X)$.
As seen by our coin toss example, if the marginal $p(X)$ is invariant to the equivalence, \ie, $X \in \tasksinv{}$, then maximal invariants coincide with minimal sufficient statistics  \cite{halmos_application_1949,bahadur_sufficiency_1954}.
%In other words, if $X \in \tasksinv{}$ then maximal invariance essentially coincides with minimal sufficiency because  $\rv X - \rv M(X) - \rv X$ forms a Markov Chain (due to \cref{lemma:desintegration}).
In our work we are interested in predicting a target $\rv Y$ rather than reconstructing the source $\rv X$. 
A sufficient statistic w.r.t. to another r.v. $\rv Y$ is referred to as adequate statistics.\footnote{The standard definition of adequacy from \citet{skibinsky_adequate_1967} also requires the statistic to be sufficient, here we use ``adequacy in the wide sense'' as defined by \citet{takeuchi_characterizations_1975}}
%usually referred to as a predictive sufficient statistic \cite{bernado_bayesian_2009} or adequate statistics \cite{skibinsky_adequate_1967}
%(predictive / adequate statistics usually also require sufficiency of $\rv X$, here we refer to \citepos{Takeuchi} ``adequate ).
%or adequate statistics \cite{skibinsky_adequate_1967} if it is also sufficient for $\rv X$ ).
Maximal invariants can thus be seen as minimal adequate statistics for the set of all invariant tasks of interest $\tasksinv{}$.
Using minimal adequate statistics as good representations for performing a task has been well investigated in ML to improving generalization \cite{shamir_learning_2010,jiang_learning_2017,cvitkovic_minimal_2019,achille_information_2018,soatto_visual_2016,dubois_learning_2020}.
The main difference with our work is that 
\begin{inlinelist}
\item we consider adequacy for a collection of tasks instead of a single task;
\item minimality arises from a compression perspective rather than for generalization.
\end{inlinelist}
Although we are not aware of any use of minimal adequacy for compression (even single task), minimal sufficiency is often used for compressing distributions \cite{hayashi_minimum_2018,iri_fine_2019}.

%\ydimp{
%Fill in link with sufficiency / minimality / predictive suff / adequacy .... \ben{You probably have strong thoughts on that}
%
%At the core what we are doing in the lossless prediction setting is saying that the optimal compression is a (possibly stochastic) minimal sufficient statistic for predicting multiple tasks. 
%We should probably task about predictive / Adequate statistics, about IB stuff, and about the use in computer vision (e.g. soatto), and about use in compression (e.g. https://arxiv.org/pdf/1612.02542.pdf) .
%}

% At the core of our work is the maximal invariant $M$ which retains all and only information which is needed for desired tasks.
% The maximal invariant is closely related to the notion of m

\clearpage
\newpage

\section{Reproducibility}
\label{appx:reproducability}
In this section we provide further details of the hyperparameters chosen for the various experiments in the main text.
The code to reproduce all experiments can be found at \codeurl{}.
We checkpoint and use the model which achieves the smallest \textit{validation} loss for evaluation.
%Results are averaged over 5 random seeds, and standard errors are reported.
Unless stated otherwise, all the models are trained for 100 epochs, using Adam \cite{kingma_adam_2015} as the optimizer, and a batch-size of 128. 
The learning rate starts at $1\sci{3}$ that decreases exponentially until reaching $1\sci{6}$ at the end of training.
For all convolutional layers we use Kaiming normal initialization\cite{he_delving_2015}, for all linear layers we use Kaiming uniform initialization\cite{he_delving_2015},  while all biases are always initialized at 0.
Activation functions are ReLUs while other unspecified parameters are PyTorch \cite{paszke_pytorch_2019} defaults.
For our invariant models, instead of optimizing $\MI{Z}{X} + \beta \dists{}$ we optimize $\lambda \MI{Z}{X} + \dists{}$, which is a more standard formulation for VIB, VAE, and neural compressors.
In the following sections we will sometimes refer to $\beta$ as $\frac{1}{\lambda}$.

%%%%%%%%%%%%%%%%%%%%%%%%%%%%%%%%%%%%%%%%%%%%%%%%%%%%%%%%%%
\subsection{Banana}
\label{appx:reproducability_banana}

For the Banana dataset most of the arguments were selected so as to replicate Fig.1.B. from \cite{balle_nonlinear_2020}. 
\footnote{Their code can be found at \url{https://github.com/tensorflow/compression/blob/master/models/toy_sources/toy_sources.ipynb}} 

\paragraph{Data}
The data distribution is obtained by starting from a bivariate Gaussian $\rv X \sim \mathcal{N}( \mathbf{0}, \mathrm{diag}([3, 0.5]) )$.
It is then transformed to a banana distribution using the following transformation: $x_2 = x_2 + 0.1  x_1 ^ 2 - 9$.
We then rotate it and shift it: $\vec x = (\mathrm{Rot}(-40) \cdot \vec x) + [-3, -4]^T$.
For every epoch we resample $1024000$ new points, \ie, examples are never seen twice during training).

\paragraph{Hyperparameters}
For all Banana experiments we use a 2 dimensional representation $\rv Z \in  \mathbb{R}^2$, a learning rate of $1\sci{3}$ that decreases exponentially until reaching $1\sci{6}$ at the end of training, and a batch size of $8192$.
The encoder (and decoder if there is one) is always a 2-hidden layer MLP with 1024 hidden neurons, and softplus activation.
%All optimizers are Adam with $1e-3$ learning rate.
%The learning rate scheduling consists in decreasing the learning rate by a factor 10 at epochs 50,75,87, and 120.
In all cases we an entropy bottleneck with a factorized prior from \cite{balle_variational_2018}.

\paragraph{Experiment: \cref{fig:bananas_sweeps}}
We train both a standard variational compressor (VC) and our variational invariant compressor (VIC).
%We compared our VIC, which ``reconstructed'' unrotated vectors from rotated ones, to standard neural compressors (VC), which reconstructed raw samples.
The downstream performance loss is the MSE when predicting the maximal invariant.
In both cases we use $\lambda=0.07$, which was chosen so that the downstream performance is similar for both.
For VIC the data is first augmented using rotations, passed through an encoder, then the decoder predicts the maximal invariant $M: x \mapsto \mathrm{Rot}(225) \cdot [0,\|x\|_2]^T$, \ie, the point with the same radius but positioned at 225 degrees.
Note that we use this maximal invariant (instead of the more natural $\|x\|_2$) to ensure that the reconstructions (codebooks) can be plotted in in a nice way in the original space $\mathcal{X}$.
The choice of maximal invariant does not impact the learned partition of the space.

Each plot (\cref{fig:bananas_sweeps} right) is generated by first taking a meshed grid of $500^2$ source points in $[-5,5]^2$.
Then we quantize every point in the mesh by passing it through our encoder.
The partition of the space (delimited with pink contours) corresponds to all points in the mesh that got mapped to the same quantized representation.
To obtain the codebook (pink dots), we pass the quantized representations through our learned decoders.
Finally, we plot the distribution of our learned entropy model by rescaling the codes so that their area is proportional to the rate assigned by the entropy model, \ie, $- \log q_{\theta}(z)$.

To obtain rate-invariance curves (\cref{fig:bananas_sweeps} left), we  sweep over $\lambda=0.00001,0.0001,0.001,0.01,0.1,1,10,100,1000$.
For each point in \cref{fig:bananas_RI} we plotted the average over 5 seeds and plotted in gray the standard errors (both in the rate and distortion direction).
To compute the area under the curve we used the trapezoidal rule on each of the RI curves obtained by a single seed, we then aggregated to area under the curve for the 5 seeds to obtain the mean and standard error.

\paragraph{Experiment: \cref{fig:bananas_xtrnslt}}
Here augmentations are translations on the $x$-axis.
The BINCE model was trained using \cref{alg:BINCE}, \ie , without assuming knowledge of the maximal invariant.
For VIC we used $M : x \mapsto [0,x_2]^T$ as the maximal invariant.

\paragraph{Experiment: \cref{fig:bananas_ytrnslt}}
Here augmentations are translations on the $y$-axis.
For VIC we used $M : x \mapsto [x_1,0]^T$ as the maximal invariant.
To plot of the induced distribution in $\mathcal{M}$ (here the $x$-axis), we sample $1024000$ new points, pass them through our encoder and decoder to obtains the codes, and then plot a histogram of the obtained codes (shown in salmon).
In blue we also plot the (approximate) distribution of the source when marginalized our invariances (\ie only consider the $x$ component).

\subsection{General Image framework}
\label{appx:reproducability_general}

Here we discuss the framework which we used for most of our image experiments. Unless stated otherwise, for all (ours or standard) neural image compressors we use the same general framework and architectures.

Specifically, an image $\rv X$ first passes through a ResNet18 \cite{he_deep_2016} to obtain a 128 dimensional representation $\rv Z$ that t.v.i $\mathbb{R}^{128}$.
We then pass $\rv Z$ through an entropy bottleneck with a scaled hyperprior entropy model from \citet{balle_variational_2018} which gives us the quantized $\hat Z$.
For our entropy bottleneck we used \citepos{begaint_compressai_2020} implementation which is a Pytorch re-implementation of \cite{balle_variational_2018}.
Note that the choice of entropy model and quantizer is orthogonal to our work, and any choice that works neural compression would work for us.

In the case where we have to decode an image (VIC and VC models), we pass the quantized $\hat Z$ through a linear layer to reshape it to a latent image in $\mathbb{R}^{2 \times 2 \times 256}$.
The latent image then passes through a 4-layer transposed CNN decoder where after each layer the number of channels gets divided by two and the width and height of the image doubles. 
The last layer outputs an image with the correct number of channels (1 for MNIST, 3 for other datasets), which is treated as the reconstruction $\hrv X$ of the augmented input (for standard compressors) or the non-augmented input (for our VIC).

To simulate how well you could perform on downstream tasks of interest (that are not known when learning the compressors), we evaluate how well a model can classify the labels from the dataset.
Specifically, once the models are trained we freeze them, apply them to the dataset and train neural network to classify the inputs using either the quantized representation $\hrv Z$ or the reconstruction $\hrv X$.
In the former case we a $|\mathcal{Z}|-2048-2048-|\mathcal{Y}|$ MLP with preactivation batch normalization \cite{ioffe_batch_2015}.
In the latter case we use a ResNet18 for predictions.

Finally, we obtain the desired bit-rate by considering the expected log loss of the trained entropy model on the test distribution (\ie theoretical bit-rate).
The desired distortion is obtained by evaluating the predictor on the compressed test dataset.

%%%%%%%%%%%%%%%%%%%%%%%%%%%%%%%%%%%%%%%%%%%%%%%%%%%%%%%%%%
\subsection{MNIST}
\label{appx:reproducability_mnist}

For our MNIST \cite{lecun_gradient-based_1998} experiments we compare again our VIC (as described in \cref{alg:vic}) against a standard neural compressor.

\paragraph{Data}
In order to evaluate our framework in a relatively well understood setting we use the well known  MNIST \cite{lecun_gradient-based_1998} dataset, which we rescale to $32\times32$ pixels.
For this toy setting we want to understand how our model performs when trained with augmentations that induce the equivalent relation w.r.t. which we are invariant, \ie, we assume that we know the ``correct'' augmentation.
To do so we augment both the training and the test set in the same way.
Specifically, we apply random rotations sampled from $[-45,45]$ degrees, random translations between $[0,25]$ percentage of pixels, random shearing between $[0,25]$ degrees, and random scaling by a factor in $[0.6,1.4]$. 

\paragraph{Experiment: \cref{fig:mnist_intro} and \cref{fig:augmnist++_RD_rec}}
For a fair comparison we took trained a standard compressor and a VIC so that the downstream accuracy on augmented MNIST is the closest possible to $99\%$ accuracy (note that augmented MNIST is slightly harder than standard MNIST).
We then randomly sampled reconstructions for the source, standard reconstructions, and VIC reconstructions which we plot.
The quantitative results are average over 5 runs and standard errors are provided in \cref{fig:augmnist++_RD_rec}.

\paragraph{Experiment: \cref{fig:augmnist++_RD}}
For the rate-error curve we swept over $\lambda=0.001,0.01,0.03,0.1,0.3,1,3,10,100$ and plotted the curves and computed the area under the curve in the same way as previously discussed for the Banana rate-invariance curve.

%%%%%%%%%%%%%%%%%%%%%%%%%%%%%%%%%%%%%%%%%%%%%%%%%%%%%%%%%%
\subsection{STL10}
\label{appx:reproducability_stl10}

\paragraph{Data}
We use the STL10 dataset \cite{coates_analysis_2011} which is similar to CIFAR but with fewer labeled training examples. There are 10 classes of 96x96 pixeled, colored images. There are 500 labeled and 100000 unlabeled examples for training, 800 labels for test. Note that the unlabeled images come from a broader distribution of images.
For augmentations we use horizontal flips, resizing and cropping, the color transform and we randomly transform the image to gray scale with a likelihood of $0.2$.
As for MNIST we augmented both the train and the test distribution.
The compressors were trained on the unlabeled data, while the predictors were trained on the train distribution.

\paragraph{Hyperparameters}
We used an entropy bottleneck with a scaled hyperprior entropy model from \cite{balle_variational_2018}.  
When training with BINCE, VIC or VC, the encoder is a ResNet18 architecture. For hyper-parameter tuning we randomly sampled 100 hyperparameters from the following search space:  latent dimension size $(32-512)$, rate-distortion trade-off $\lambda$ $(10^{-13},100)$, the optimizer's (ADAM) learning rate $(10^{-4},10^{-3})$, the learning rate schedule(exponential decay or cosine decay), and the batch size $(64-128)$.
For prediction on the learned features$\hrv Z$ we trained an MLP with $1024-4096$ hidden units, one or two layers, and dropout probability between $0.0-0.5$. We optimized again the learning rate of the ADAM optimizer as before.
For predictions from the reconstructions $\hrv X$ , we trained a ResNet18, with the same optimizer parameters as above.

\paragraph{Experiment: \cref{table:distortion_variation}}
In this experiment we compare the compression performance of PNG \cite{graphics_png_isoiec_2003}, WebP \cite{webp_google_2018}, JPEG \cite{group_jpeg_itu-t_1992}, VC, VIC and BINCE. 
Since VC and VIC allow to predict either on features or on reconstructions we test both. We sweep uniformly over the log-scale of $\lambda=10^{-5},100$ for the neural compressors and sweep the classical compressors over an equivalent quality range. 
The extensive results from which \cref{table:distortion_variation} is derived are in \cref{table:distortion_variation_long}. The rate-distortion curves belonging to this experiment are \cref{fig:STL10_dist}.
The rate-distrotion curves correspond to the pareto optimal curves of the encoders and predictors from the 1000 models.

%%%%%%%%%%%%%%%%%%%%%%%%%%%%%%%%%%%%%%%%%%%%%%%%%%%%%%%%%%
\subsection{Galaxy Zoo}
\label{appx:reproducability_galaxy}
 \paragraph{Data}
Celestial objects and events emit radio frequencies. These frequencies are recorded through large antennas. 
Modern radio astronomy relies on the aggregation of radio signals in time and space. This means that one antenna records over long stretches of time. Due to the rotation of the earth, this translates to many spatial measurements. Further, the inclusion of many antennas in various locations can provide a dense net of observations. The entire system is refereed to as aperture synthesis telescope (AST). 
Images of the sky are generated by combining the sequences of observations stemming from different antennas. 
ASTs generate enormous amounts of data, much of which is redundant and further will never be observed by humans.
In fact commonly applied techniques to the observations, such as weighting (e.g. Briggs weighting) and blurring of signals removes information from the original observations. 
Our approach is thus a natural extension to the techniques already present in the radio astronomy community.
However, the process of image reconstruction from radio frequency observation series is too complex for the scope of this paper. 
We thus work on the Galaxy Zoo 2 (GZ2) dataset, that contains of already inferred images of celestial objects. 
We believe that good rates on this dataset should hint at even better possible rates when working directly with the raw data.
GZ2 contains 37 classification tasks, such as answering queries about shapes and counts information of galaxies.
Although the tasks are classification tasks, we use the standard GZ2 evaluation that consists in regressing (RMSE evaluation) the expected (over different labellers) label probability.
Our data is hosted on the kaggle platform.
This means we have no access to test labels but only for total test loss. 
This is why we compute summary statistics on the validation data set.
We choose to reduce the original dataset by center cropping to 256 pixels per dimension.
We applied random rotations, horizontal and vertical flips, scaling ($1-1.3 \times$) and color transforms to this data. 
We used CNN encoders \cite{balle_end--end_2017,minnen_joint_2018} and an entropy bottleneck with a scaled hyperprior entropy model from \cite{balle_variational_2018}.
 
\paragraph{Hyperparameters}
For all experiments we used ResNet50 when predicting from images (\ie encoders and predictors from reconstructions).
As with the STL10 experiments, we trained each model and baseline by selecting a set of 100 hyperparameters randomly selected from a large search space.
When training with BINCE, we sampled the latent dimension size $(32-2048)$, rate-distortion trade-off $\lambda$ $(1e-12,1e-4)$, the optimizer's (ADAM) learning rate $(1e-4,1e-3)$, the learning rate schedule(exponential decay or cosine decay) and the batch size $(64,128)$. For prediction on the learned features we would train an MLP with 2048 hidden unit, two layers and dropout probability $(0.0-0.5)$. We optimized the again the learning rate of the ADAM optimizer as before.
For the classical compressors we trained a ResNet50 on their reconstructions, with the same optimizer parameters as above.

%%%%%%%%%%%%%%%%%%%%%%%%%%%%%%%%%%%%%%%%%%%%%%%%%%%%%%%%%%
\subsection{Pretrained CLIP}
\label{appx:reproducability_clip}

\ydnote{add lambda and beta}

\paragraph{Data}
In addition to the pretrained CLIP, we trained the entropy bottleneck.
As we do not have access to the dataset from CLIP, we could not train the entropy bottleneck on the initial data.
Instead we had to use a different dataset.
We used MSCOCO \cite{lin_microsoft_2015} for image captioning, as we initially thought that we would need access to pairs of images and sentences to finetune CLIP.\footnote{
At the end we did not use the sentences as freezing CLIP worked very well.}
Note that in our experiments the choice of dataset for training the entropy bottleneck (\eg CIFAR10 \cite{krizhevsky_learning_2009}) had very little impact on the quality of the final compressor.

To evaluate our compressor in the most realistic setting possible, we selected 10 different datasets.
The datasets were chosen so that
\begin{inlinelist}
\item the source images are of very different shapes and content;
\item they are easily be accessible online;
\item images are already compressed by JPEG;
\item neither the entropy bottleneck nor CLIP should have been trained on the selected datasets;
\item the task are very different classification tasks.
\end{inlinelist}
To ensure that CLIP was (nearly) not pretrained on the selected datasets we selected a subset of the datasets that CLIP was evaluated on and which did not show significant data overlap (see \citepos{radford_learning_2021} section 5 for a discussion about data overlap).
\Cref{table:clip_data} shows the details about the 10 datasets that we use for evaluating our model.
When there is no prespecified validation split, we randomly sampled $10\%$ of the training data for validation.

\begin{table}[h]
\caption{
Datasets used to evaluate our zero-shot compressor. -1 for the shape means variable. 
}
\begin{center}
\begin{tabular}{lrrrrrr}
\toprule
Dataset & Classes  & Train size & Valid size & Test size & Metric & Shape   \\ 
\midrule 
ImageNet \cite{deng_imagenet_2009} & 1000 &  1,281,167 &    & 50,000 & accuracy & (-1,-1,3) \\
CIFAR10 \cite{krizhevsky_learning_2009} & 10 & 50,000 & & 10,000 & accuracy & (32,32,3) \\
CIFAR100 \cite{krizhevsky_learning_2009} & 100 & 50,000 & & 10,000 & accuracy & (32,32,3) \\
Cars196 \cite{krause_3d_2013} & 196 & 8,144 &  & 8,041 & accuracy & (-1,-1,3) \\
Pets37 \cite{parkhi_cats_2012} & 37 & 3,680 &  &  3,669 & balanced acc. & (-1,-1,3) \\
Caltech101 \cite{li_learning_2007} & 102 & 3,060 & & 6,085 & balanced acc. & (-1,-1,3) \\
Food101 \cite{nilsback_automated_2008} & 101 & 75,750 &   & 25,250 & accuracy. & (-1,-1,3) \\
STL10 \cite{coates_analysis_2011} & 10 &  5000 & &  8000 & accuracy & (96,96,3) \\
PCam \cite{veeling_rotation_2018,ehteshami_bejnordi_diagnostic_2017} & 2 &  262,144 & 32,768 & 32,768 & accuracy & (96,96,3)  \\
\bottomrule
\end{tabular}
\end{center}
\label{table:clip_data}
\end{table}

\paragraph{Reproducing the results}
For clarity and reproducebility we also provide a self contained script to train a very similar version to our compressor in \cref{code:array_compressor} and \cref{code:clip_code}.
The main changes being that we change the training data (using CIFAR10), the entropy bottleneck (to the simpler factorized prior from \cite{balle_variational_2018}), and use a simpler evaluation pipeline (only use STL10 with a simplified MLP).
The entire script (including evaluation and actual compression of a dataset) takes less than ten minutes to run on a single GPU and provides a general zero-shot compressor.
The full code that we used is accessible at \codeurl{}.

\paragraph{Training the zero-shot compressor}
To train our compressor we first download the official pretrained CLIP model\footnote{https://github.com/openai/CLIP}.
Specifically the vision transformer \cite{dosovitskiy_image_2020} that they refer to as \texttt{"ViT-B/32"}.
We then freeze it, and add an entropy bottleneck with \citepos{balle_variational_2018} hyperprior.
We then train the entropy bottleneck on the MSCOCO dataset.

To train the entropy bottleneck we need a distortion measure.
In theory, to get our BINCE objective, we should use the distortion that CLIP was trained with, \ie, we should compress the representation in such a way that CLIP can still distinguish examples from the same equivalence class. Minimizing such distortion can lead to catastrophic forgetting as the representations only ensure that CLIP can distinguish equivalent example from our very small dataset. \footnote{We tried many ways of finetuning CLIP with very small learning rates and frozen components, but although the rate gains were large (around $2$ to $3\times$) this lead to significant decrease in performance, most probably due to catastrophic forgetting.}
We instead use a very simple MSE distortion in the representation space.
Specifically, we trained the entropy bottleneck to minimize $\lambda \| z - \hat{z} \| - \log(q_\theta(z))$, where $\hat{z}$ denotes the reconstructed (quantized) representation.
This can be seen in line 22 of \cref{code:array_compressor}.

One important point to notice is that in standard neural compressors the quantization is  a component wise rounding to the closest integer.
This typically does not constrain the compressor, as the compressor is trained in an end-to-end fashion so that the encoder can increase or decrease some components of $z$ to effectively increase or decrease the quantization.
As our encoder is frozen, it cannot learn to adaptively change the scale of $z$ so we needed to learn the size of the quantization interval instead, \ie, the rounding precision.
A simple (and equivalent) way of doing that consists in passing the representation through a (learned) component wise linear transformation (\ie 2 parameters per component) then through the entropy bottleneck (quantization) and finally we reverse the linear transformation.
This can be seen in line 12 and 14 of \cref{code:array_compressor}.

Generally we found that training the entropy model was very robust to hyperparameter changes. 
We used the following: $50$ epochs, a 512 dimensional $z$ (given by CLIP), a batch size of $64$, a learning rate of $1\sci{3}$ with 3, a  scheduler that decreases the learning rate by $10\times$ every 12 epochs (\ie uniformly 3 times during training), Adam with decoupled weight decay (AdamW; \cite{loshchilov_decoupled_2019}) as an optimizer, weight decay of $3\sci{8}$ and a $32$ dimensional side information for the hyperprior.
For the our main CLIP compressor (CLIP+EB) we use an RD hyperparameter $\lambda = 5\sci{2}$ which is linearly annealed from $1\sci{7}$ to $1\sci{2}$ in the first $5$ epochs of training (although annealing did not seem important).
For our CLIP+EB\textsuperscript{$-$} and CLIP+EB\textsuperscript{$+$} (see \cref{table:clip_all}) we respectively use $\lambda=1\sci{2}$ and $\lambda=1\sci{1}$.
For data augmentations we used similar ones as used by CLIP namely, normalizing the image by the mean and std form their training dataset (\texttt{mean=[0.48145466, 0.4578275, 0.40821073]} and \texttt{mean=[0.26862954, 0.26130258, 0.27577711]}), resizing the smallest side of the image to $224\times224$ pixels with bicubic interpolation, then taking a random $224\times224$ crop.
During the evaluation the random cropping is replaced by a center cropping.

\paragraph{Evaluating the zero-shot compressor}
For evaluating the rate obtained by our compressor, we provide the negative log likelihood of our entropy model for the each test dataset.
For evaluating the downstream predictive performance for each dataset we train a 2 hidden layer MLP of dimensions $512-2048-2048-|\mathcal{Y}|$ with batch normalization and ReLU activation.
For each dataset we provide the best model from 25 randomly sampled models, that arise by randomly sampling the following hyperparameters:
\begin{itemize}
\item batch size: $[32,64]$ with logarithmic sampling.
\item optimizer: Adam, SGD, AdamW.
\item weight decay: $[1\sci{7},1\sci{4}]$ with logarithmic sampling.
\item learning rate: $[1\sci{5},1\sci{3}]$ with logarithmic sampling.
\item scheduler: exponential decay (with total decay by 100 or 1000), decreasing learning rate on validation loss plateau, cosine scheduler, decaying learning rate at fixed intervals.
\item dropout \cite{srivastava_dropout_2014}: $[0,0.5]$.
\end{itemize}

We then provide the result of the best model. In \cref{table:clip} we compare those results to the same vision transformer that we use for CLIP, but trained directly on the raw images.
These results were obtained from \citepos{radford_learning_2021} table 10.
For a better comparison to standard SSL models, in \cref{table:clip_all} we also provide the test accuracy of a linear layer (a support vector machine) from the representations. The regularization parameter of the SVM were all selected using 10 values and three fold cross validation.
For the rates we compare to the average JPEG size of images in each datasets (all the selected datasets are compressed by default in JPEG).
For the rates of the raw CLIP model we losslessly compress the representation using numpy's \texttt{savez} function (zip) \cite{harris_array_2020}.

%%%%%%%%%%%%%%%%%%%%%%%%%%%%%%%%%%%%%%%%%%%%%%%%%%%%%%%%%%
\subsection{Minimal code to train the CLIP compressor in < 5 min.}
\label{appx:code_clip}

In this section we provide minimal code to train our zero-shot compressor and to use our compressor to entropy code an entire dataset.
Note that the model is simplified (\eg using factorized prior instead of a hyperprior, and training on CIFAR10) so the bit-rates is slightly increased but it still achieves orders of magnitude gains compared to JPEG. 
We use CIFAR10 for training the entropy coder and STL10 for downstream evaluation (as both are downloadable through torchvision).
To evaluate the model, we use a linear support vector machine from our representation $\rv Z$.

For this minimal code,
training takes around $3$ minutes on a single GPU.
The theoretical bit-rate that we achieve is around $1400$, while the practical bit-rate achieved by entropy coding is around $1700$.
In comparison the bit-rate of JPEG (with 95 quality) is 4.71e4.
The entropy coder compresses around $200$ images per seconds, and decompresses around around $3$ images per seconds. Decompression is slow as we do not perform it in batch (for simplicity of the code), while encoding is batched processed.
Downstream classification accuracy on STL10 is $98.7 \%$ which is better than the uncompressed representations from CLIP, from which linear probe achieves $98.6 \%$ accuracy.

To run the code you need first need to install the following libraries:
\begin{minted}{bash}
pip install git+https://github.com/openai/CLIP.git
pip install scikit-learn==0.24.2 lightning-bolts==0.3.3 compressai==1.1.4
\end{minted}

The minimal boilerplate code (\cref{code:clip_code}) downloads the training data and the pretrained CLIP (from line 42), trains the compressor (from line 46), entropy code the evaluation data (from line 57), and finally evaluates downstream performance (from line 67).
The actual compressor is defined in \cref{code:array_compressor}.

\clearpage
\begin{listing}
\vspace{-2\baselineskip}
\begin{minted}{python}
import clip, torch, pytorch_lightning, numpy, tqdm, math, time
from torchvision.datasets import STL10,CIFAR10
from compressai.entropy_models import EntropyBottleneck
from torch.optim import Adam,lr_scheduler
from torch.utils.data import DataLoader
from pl_bolts.datamodules import SklearnDataModule
from sklearn.svm import LinearSVC

def clip_featurize_data(dataset, device):
    with torch.no_grad():
        Z,Y = [],[]
        for x, y in tqdm.tqdm(DataLoader(dataset, batch_size=128, num_workers=16)):
            Z += [pretrained.encode_image(x.to(device).half()).cpu().numpy()]
            Y += [y.cpu().numpy()]
    return numpy.concatenate(Z), numpy.concatenate(Y)

def compress_data(trainer, dataset, device,**kwargs):
    start = time.time()
    Z,Y = clip_featurize_data(dataset, device)
    dm = SklearnDataModule(Z,Y,**kwargs)
    out = trainer.predict(dataloaders=dm.train_dataloader())
    Z_bytes = [o[0] for o in out]
    flat_z = [i for batch in Z_bytes for i in batch]
    Y = numpy.concatenate([o[1] for o in out], axis=0)
    coding_rate = sum([len(s) for s in flat_z]) * 8 / len(flat_z)
    sec_per_img = (time.time() - start)/ len(flat_z)
    return Z_bytes, Y, coding_rate, sec_per_img

def decompress_data(compressor,Z_bytes):
    start = time.time()
    with torch.no_grad():
        Z_hat = [compressor.decompress(b).cpu().numpy() for b in Z_bytes]
    sec_per_img = (time.time() - start)/len(Z_hat)
    return numpy.concatenate(Z_hat), sec_per_img

data_dir = "data/"
if torch.cuda.is_available():
    device, precision, gpus = "cuda",16,1
else:
    device, precision, gpus = "cpu",32,0

print("Downloading data and CLIP")
pretrained, preprocess = clip.load("ViT-B/32", device)
cifar = CIFAR10(data_dir,download=True,train=True,transform=preprocess)

print("Training compressor.") 
start = time.time()
Z_cifar,Y_cifar = clip_featurize_data(cifar, device)
data_kwargs = dict(num_workers=16,batch_size=128,pin_memory=True,val_split=0.,test_split=0)
dm_cifar = SklearnDataModule(Z_cifar,Y_cifar,**data_kwargs)
compressor = ArrayCompressor(z_dim=512,lmbda=4e-2,lr=1e-1,lr_step=2)
trainer = pytorch_lightning.Trainer(gpus=gpus,precision=precision,max_epochs=10,logger=False)
trainer.fit(compressor, datamodule=dm_cifar)
print(f"Compressor trained in {(time.time() - start)/60:.0f} minutes.")  # 3 minutes

print("Download evaluation data and entropy code it.")
stl10_train = STL10(data_dir,download=True,split="train",transform=preprocess)
stl10_test = STL10(data_dir,download=True,split="test",transform=preprocess)
Z_b_train,Y_train,*_ = compress_data(trainer,stl10_train,device,**data_kwargs)
Z_b_test,Y_test,rate,enc_time = compress_data(trainer,stl10_test,device,**data_kwargs)
print(f"Bit-rate: {rate:.1f}. \t Compression: {1/enc_time:.1f} img/sec.")  # 1703.6 bits, 230.2 img/sec

print("Decompressing data (no batch processing).")
Z_train,_ = decompress_data(compressor,Z_b_train)
Z_test,dec_time = decompress_data(compressor,Z_b_test)
print(f"Decompression: {1/dec_time:.1f} img/sec.")  # 2.8 img/sec 
print("Downstream evaluation.")
clf = LinearSVC(C=4e-3)
start = time.time()
clf.fit(Z_train, Y_train)
delta_time = time.time() - start
acc = clf.score(Z_test, Y_test)
print(f"Downstream STL10 accuracy: {acc*100:.2f}%.  \t Training time: {delta_time:.1f} ") # 98.65% , 0.5 sec
\end{minted}
\caption{Minimal boilerplate code for training a zero-shot compressor in less than 5 minutes.
For the actual compressor (\texttt{ArrayCompressor}) see \cref{code:array_compressor}.
}
\label{code:clip_code}
\end{listing}

\clearpage
\begin{listing}
\begin{minted}{python}
class ArrayCompressor(pytorch_lightning.LightningModule):
    def __init__(self, *args, **kwargs):
        super().__init__()
        self.save_hyperparameters()
        self.bottleneck = EntropyBottleneck(self.hparams.z_dim)
        self.scaling = torch.nn.Parameter(torch.ones(self.hparams.z_dim))
        self.biasing = torch.nn.Parameter(torch.zeros(self.hparams.z_dim))
        self.is_updated = False
        
    def forward(self, batch):
        z, y = batch
        z = (z + self.biasing) * self.scaling.exp() 
        z_hat, q_z = self.bottleneck(z.unsqueeze(-1).unsqueeze(-1))
        z_hat = z_hat.squeeze() / self.scaling.exp() - self.biasing
        return z_hat, q_z.squeeze(), y.squeeze()
                                          
    def step(self, batch, *args, **kwargs):
        z_hat, q_z, _ = self(batch)
        rate = -torch.log(q_z).sum(-1).mean()
        distortion = torch.norm(batch[0] - z_hat, p=1, dim=-1).mean()
        self.log_dict({"rate":rate / math.log(2),"distortion":distortion}, prog_bar=True)
        return distortion + self.hparams.lmbda * rate   
    
    def training_step(self, batch, _, optimizer_idx=0):
        return self.step(batch) if optimizer_idx == 0 else self.bottleneck.loss()
    
    def predict_step(self, batch, _, __):
        return self.compress(batch[0]), batch[1].cpu().numpy()
    
    def compress(self, z):
        if not self.is_updated:
            self.bottleneck.update(force=True)
            self.is_updated = True
        z = (z + self.biasing) * self.scaling.exp() 
        return self.bottleneck.compress(z.unsqueeze(-1).unsqueeze(-1))
    
    def decompress(self, z_bytes):
        z_hat = self.bottleneck.decompress(z_bytes, [1,1]).squeeze()
        return (z_hat / self.scaling.exp()) - self.biasing
    
    def configure_optimizers(self):
        param = [p for n, p in self.named_parameters() if not n.endswith(".quantiles")]
        quantile_param = [p for n, p in self.named_parameters() if n.endswith(".quantiles")]
        optimizer = Adam(param, lr=self.hparams.lr)
        optimizer_coder = Adam(quantile_param,lr=self.hparams.lr)
        scheduler = lr_scheduler.StepLR(optimizer, self.hparams.lr_step)
        scheduler_coder = lr_scheduler.StepLR(optimizer_coder, self.hparams.lr_step)
        return [optimizer,optimizer_coder], [scheduler,scheduler_coder]
        
\end{minted}
\caption{Minimal code for training an entropy bottleneck to convert a pretrained SSL model into a powerful zero-shot compressor.
For the training and evaluation code see \cref{code:clip_code}.
}
\label{code:array_compressor}
\end{listing}

\clearpage
\newpage

\section{Additional experimental results}
\label{appx:results}
%%%%%%%%%%%%%%%%%%%%%%%%%%%%%%%%%%%%%%%%%%%%%%%%%%%%%%%%%%
\subsection{Banana}
\label{appx:banana}

In \cref{sec:toy_experiments} we compared a classical compressor to our VIC in the case of rotation invariant tasks. 
%The standard compressor achieves a rate $4.86$ bits for an invariance distortion of $\dists{} = 8.19\sci{2} $.
% While our compressor achieves a rate $2.30$ bits for an invariance distortion of $\dists{} = 7.95\sci{2} $.

Here show results for invariances to different equivalences and provide more intuition as to what BINCE and VIC achieve.

\begin{figure}[h]
     \centering
     \begin{subfigure}[h]{0.25\columnwidth}
         \centering
         \includegraphics[width=\textwidth]{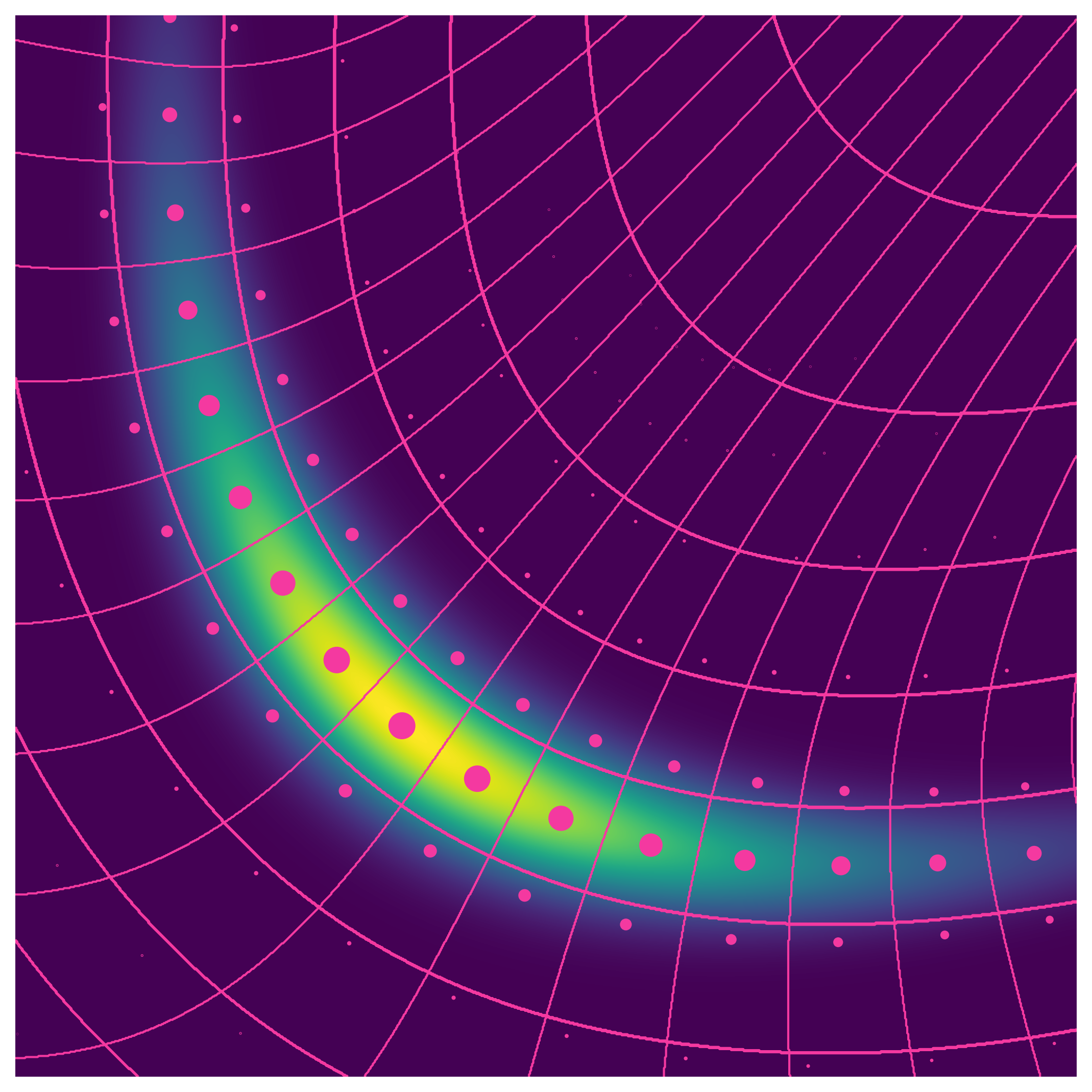}
         \caption{Standard Compression}
         \label{fig:bananas_xtrnslt_vae}
     \end{subfigure}
     \begin{subfigure}[h]{0.25\columnwidth}
         \centering
         \includegraphics[width=\textwidth]{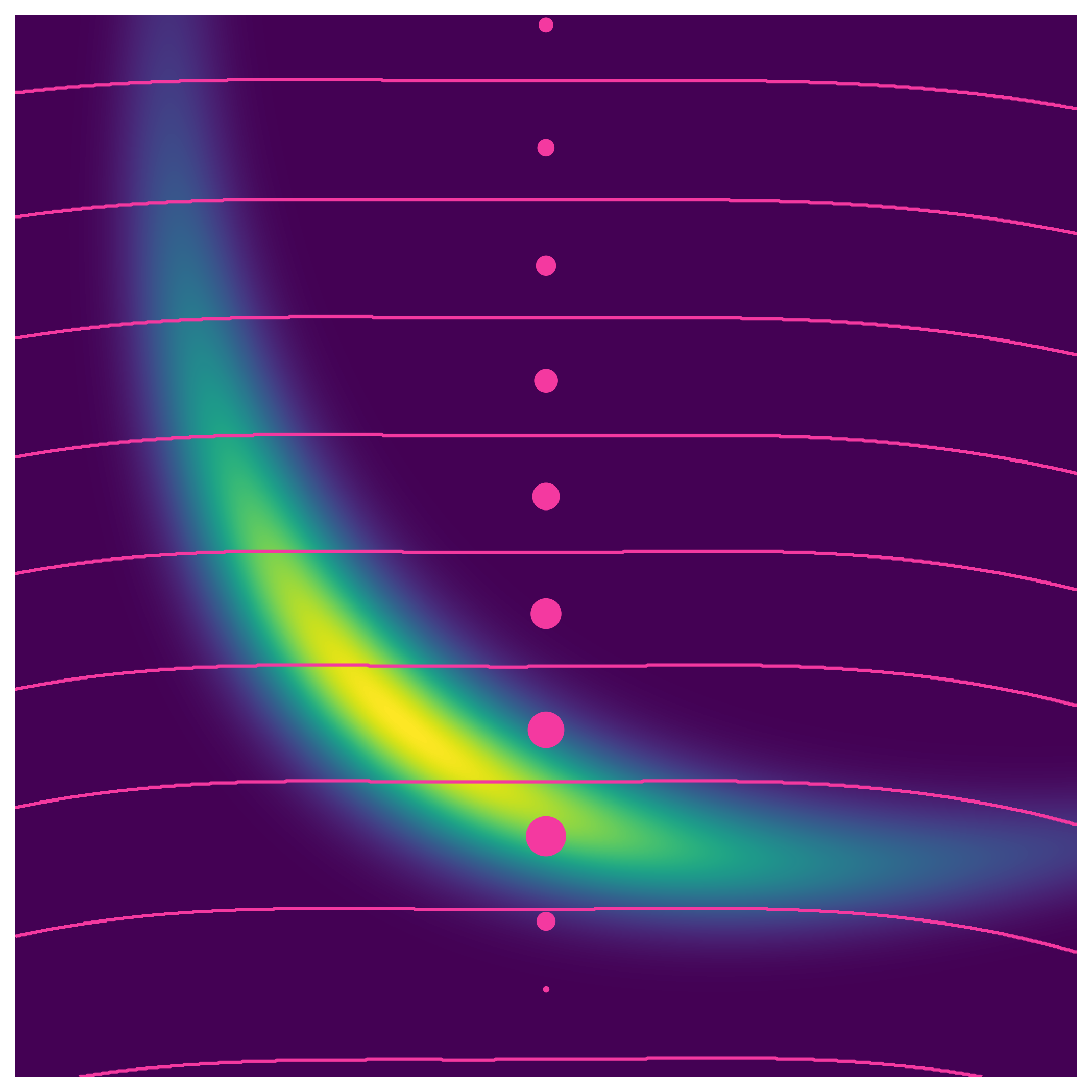}
         \caption{VIC}
         \label{fig:bananas_xtrnslt_ivae}
     \end{subfigure}
     \begin{subfigure}[h]{0.25\columnwidth}
         \centering
         \includegraphics[width=\textwidth]{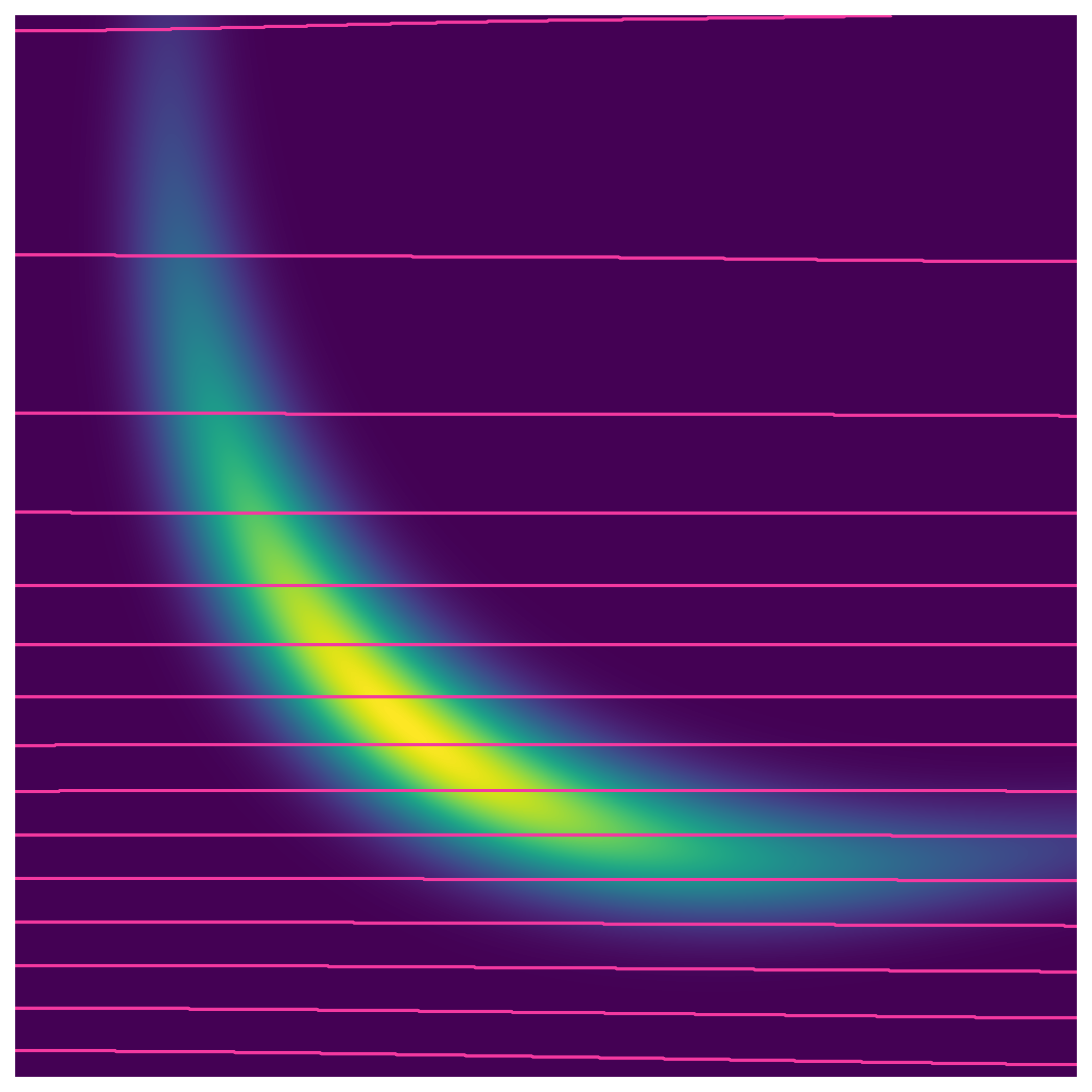}
         \caption{BINCE}
         \label{fig:bananas_xtrnslt_bince}
     \end{subfigure}
\caption{
VIC and BINCE improves compression of Banana distribution when downstream tasks are invariant to translation on the $x$-axis by quantizing the space into horizontal stripes.
%VIC does so by reconstructing maximal invariants (pink dots), BINCE by ensuring that the equivalence holds in the quantized space.
(a) standard compression with a rate of $4.86$ bits and an invariant distortion of $7.51\sci{2}$ ; (b) our VIC with a rate of $2.93$ bits and an invariant distortion of $7.08\sci{2}$.
(c) our BINCE with a rate of $2.93$ bits and an invariant distortion of $7.08\sci{2}$.
}
\label{fig:bananas_xtrnslt}
\vspace{-0.5em}
\end{figure}

\paragraph{$x$-translation and BINCE}
\Cref{fig:bananas_xtrnslt} considers the case where downstream tasks are invariant to $x$-translations.
We used $M : x \mapsto [0,x_2]^T$ as the maximal during training.
We see that our model can essentially perform as well on all downstream tasks for only $60 \%$ of the bit-rate.
Unsurprisingly we see that the codebook is in shape of horizontal stripes as these can cover the entire distribution with a few codes (small bit rate) while incurring a small invariance distortion (which only depends on the $y$ value).

We visualized a BINCE model (\cref{fig:bananas_xtrnslt_bince}) in addition to VIC.
Although the exact partition for both models is quite different (BINCE does not seem to learn equal sized partitions), both models clearly learn to partition the space into horizontal stripes.
Once important difference, is that VIC also provides a codebook (shown with pink dots), as it can reconstruct a quantized version of the input, while BINCE only learns a latent representation and does not provide any reconstructions.

\begin{figure}[h]
     \centering
     \begin{subfigure}[h]{0.24\columnwidth}
         \centering
         \includegraphics[width=\textwidth]{figures/banana/quantization_vae.png}
         \caption{Standard}
         \label{fig:bananas_ytrnslt_vae}
     \end{subfigure}
     \hfill
     \begin{subfigure}[h]{0.24\columnwidth}
         \centering
         \includegraphics[width=\textwidth]{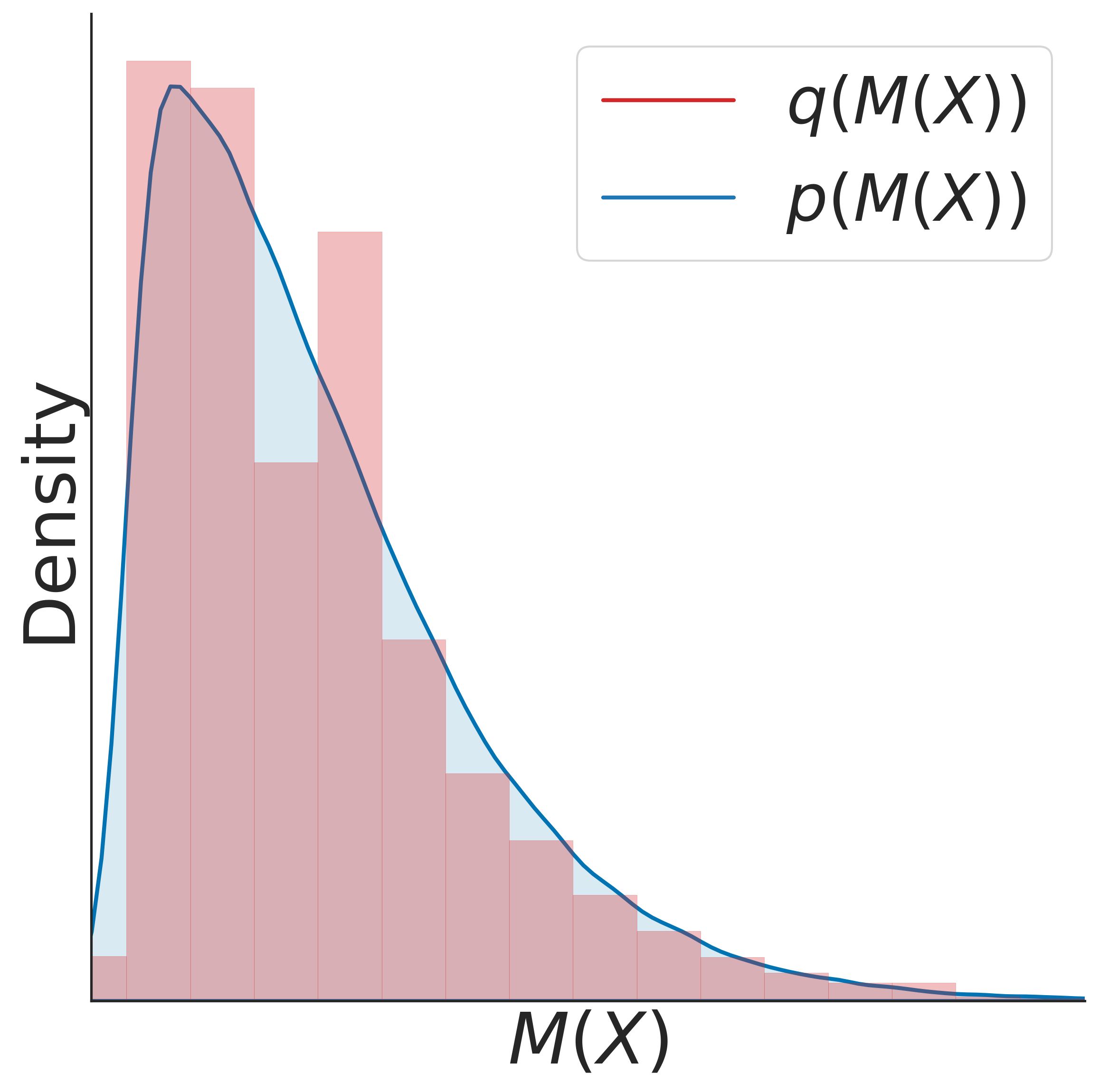}
         \caption{Standard $q(M(X))$}
         \label{fig:bananas_ytrnslt_vae_Mx}
     \end{subfigure}
     \hfill
     \begin{subfigure}[h]{0.24\columnwidth}
         \centering
         \includegraphics[width=\textwidth]{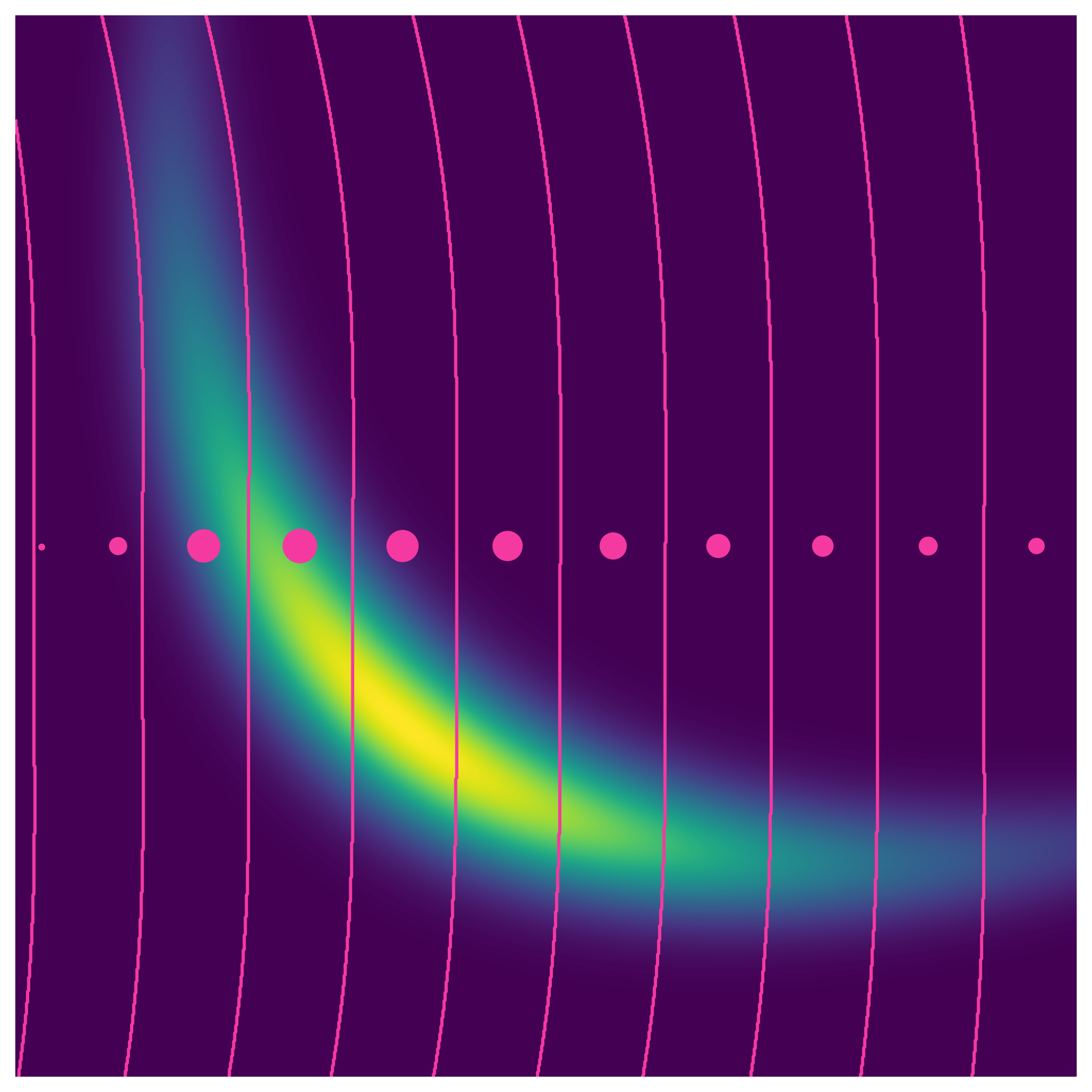}
         \caption{$y$-translation VIC}
         \label{fig:bananas_ytrnslt_ivae}
     \end{subfigure}
     \hfill
     \begin{subfigure}[h]{0.24\columnwidth}
         \centering
         \includegraphics[width=\textwidth]{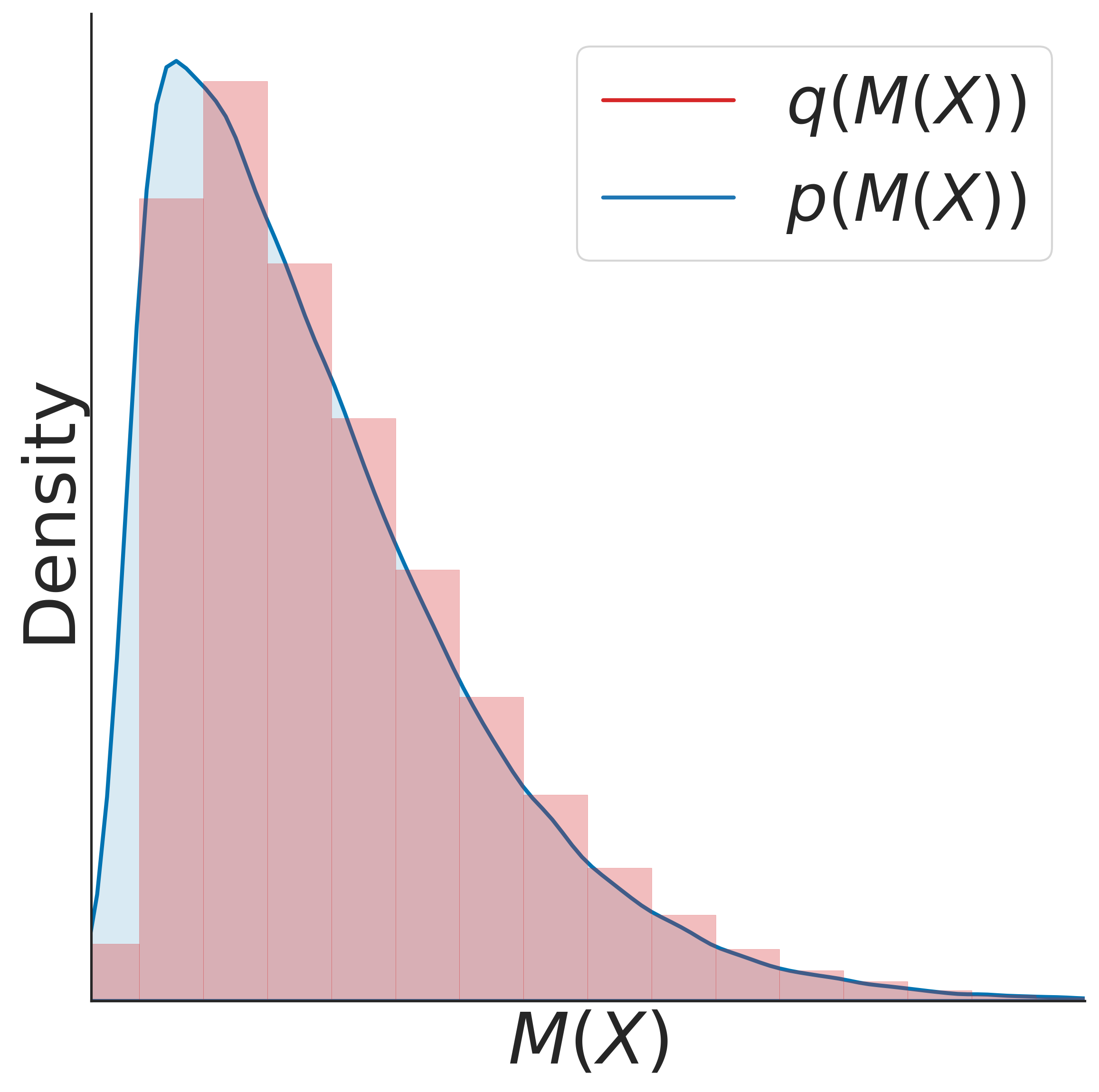}
         \caption{VIC $q(M(X))$}
         \label{fig:bananas_ytrnslt_ivae_Mx}
     \end{subfigure}
\caption{
VIC improves compression of Banana distribution when downstream tasks are invariant to translation on the $y$-axis by (implicitly) estimating the density $p(M(X))$.
(a) standard compression with a rate of $4.86$ bits and an invariant distortion of $7.67\sci{2}$ ; 
(b) the induced marginal distribution $q(M(X))$ of the $x$-value of the reconstructions from the standard neural compressor;
(c) our compression with a rate of $3.24$ bits and an invariant distortion of $7.84\sci{2}$.
(b) the induced marginal distribution $q(M(X))$ of the $x$-value of the reconstructions from the VIC;
}
\label{fig:bananas_ytrnslt}
\vspace{-0.5em}
\end{figure}

\paragraph{$y$-translation and induced distribution}
\Cref{fig:bananas_ytrnslt} considers the case where downstream tasks are invariant to $y$-translations.
In this case the maximal invariant used during training is chosen to be $M : x \mapsto [0,x_2]^T$.
Similarly to the case of $x$-translation and rotations, we see that our model can perform as well as a standard compressor for a fraction of the rate.

To provide a better intuition as to why this is the case we also plot the distribution of the reconstructions when marginalized over the $y$-axis.
In other words we plot the distribution of $M(X)$ when applied to the reconstructions, \ie, the $x$ component of the reconstructions.
We see that although the partition of the source space is very different for a standard compressor (\cref{fig:bananas_ytrnslt_vae}) and for our VIC (\cref{fig:bananas_ytrnslt_ivae}), the induced distribution (and partition) in the marginalized space are actually very similar (\cref{fig:bananas_ytrnslt_vae_Mx} and \cref{fig:bananas_ytrnslt_ivae_Mx}).
This shows where our bit-rate gains come from. 
Indeed from \cref{prop:nicer_dist} we know that in the case of invariant tasks one only needs to model the distribution of $M(X)$ (\eg the distribution of the $x$ component here), and we see that both the standard compressor and VIC does that similarly well.
The main difference being that VIC does so in an optimal way while the standard compressor needs to partition the input space in a finer way to achieve a similar induced partition in the $M(X)$ space.

\paragraphQ{What is the relation between rate and predictions}
\Cref{thm:rate_invariance_distortion} shows that, for log loss, the minimum rate is linearly related to the loss $\delta$ in downstream performance.
Our theory (\cref{appx:theorem_mse}) suggests a logarithmic relationship for MSE.
This is seen for VIC and VC in \cref{fig:bananas_sweeps} of the main text (log scale $x$-axis).

\paragraph{On lossy compression and equivalences}
Efficient lossy compression is about learning a partition (\eg Voronoi diagrams, or  \cref{fig:bananas_xtrnslt} ) of the input space to map many inputs to the same code.
We use the fact that any partition can be constructed from an equivalence relation \cite{schechter_handbook_1996} to learn compressors that are invariant to desired transformations.
The shape of the partitions are then induced by the transformations, which perturb points in their quantization bins (equivalence classes), \eg, rotations in \cref{fig:bananas_sweeps} of the main text.
The size of the partition, \eg, disks width in \cref{fig:bananas_sweeps} of the main text, depend on the desired performance $\delta$.
The pink dots are representatives of the partition, \ie, maximal invariants.
The key is that using our objectives we can learn arbitrary quantization using only desired transformations, which ML practictioners already use for data augmentations.

%%%%%%%%%%%%%%%%%%%%%%%%%%%%%%%%%%%%%%%%%%%%%%%%%%%%%%%%%%
\subsection{MNIST}
\label{appx:mnist}

% \input{figures/augmnist_plus/RD}

% \paragraph{MNIST, reconstructions, and RD curves}
% In the main paper, we have only compared a reconstruction and the achieved bit-rate by our practical VIC loss and a standard neural compressor.
% In \cref{fig:augmnist++_RD} we plot the rate-invariance curve, which shows that for the any downstream performance our VIC model requires (in expectation) only half of the bit-rate.
% In \cref{fig:augmnist++_RD_rec} we provide more reconstructions from our VIC and a standard neural compressor that allows $99\%$ downstream accuracy.
% All together we see that by reconstructing canonical digits our VIC can ensure very good downstream performance for a fraction of the bit-rate cost.

\begin{figure}[ht]
     \centering
     \begin{subfigure}[h]{0.4\columnwidth}
         \centering
         \includegraphics[width=\textwidth]{figures/augmnist_plus/RD_curve_workshop_error.png}
         \caption{Rate-Error}
         \label{fig:augmnist++_err}
     \end{subfigure}
     \hfill{}
     \begin{subfigure}[h]{0.4\columnwidth}
         \centering
         \includegraphics[width=\textwidth]{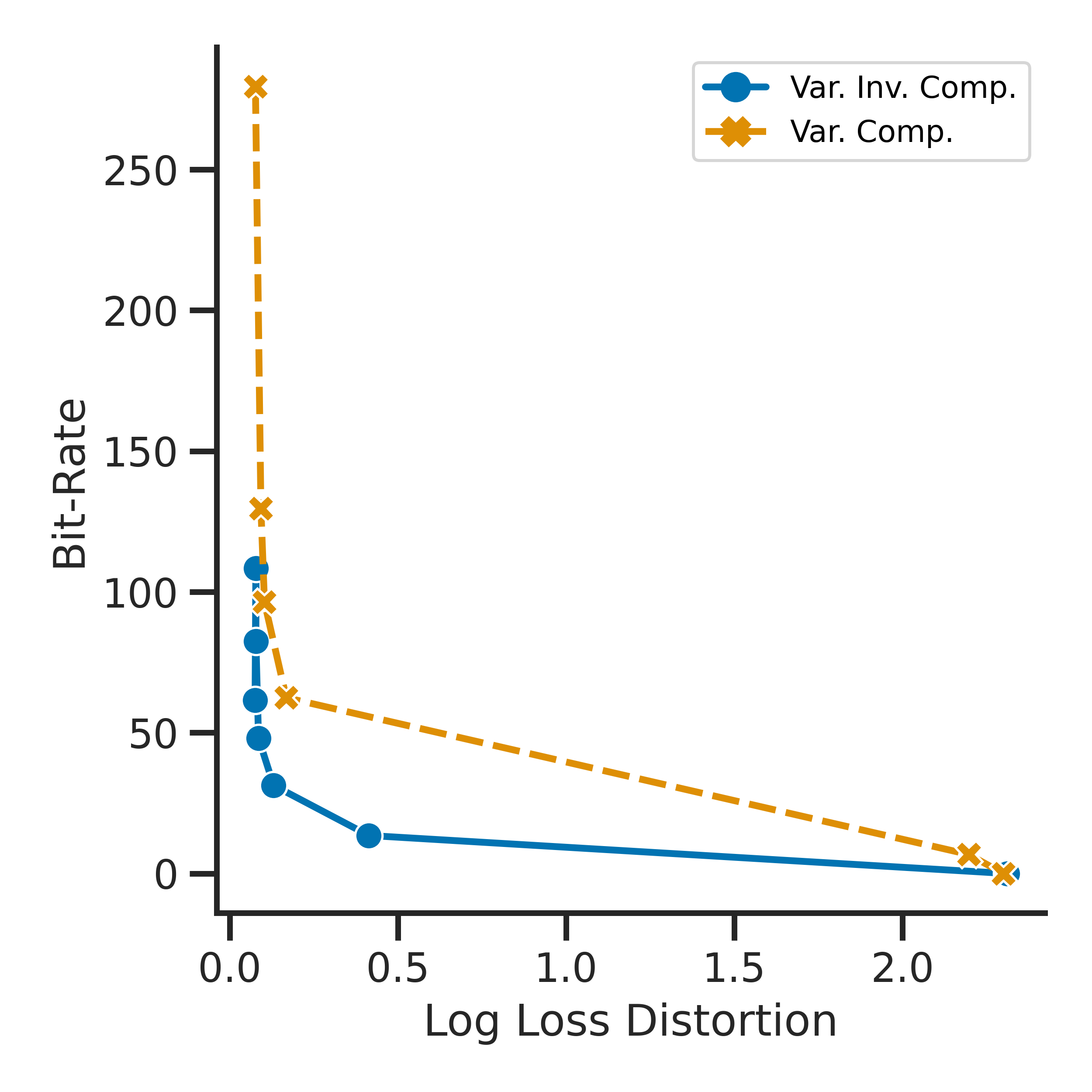}
         \caption{Rate-Log loss}
         \label{fig:augmnist++_log}
     \end{subfigure}
\caption{
Augmented MNIST RD curves for downstream predictions are very similar when using classification error (left) instead of the log loss (right) from our theory.
}
\label{fig:augmnist++_err_vs_log}
\end{figure}

\paragraphQ{How do RD curves change if we use classification error instead of log loss}
Our theory \cref{thm:rate_invariance_distortion} only ensures good downstream log loss risk.
Nevertheless, we used here the classification error ($1-\text{accuracy}$) throughout the main test as it is more commonly used for evaluating classification performance.
\Cref{fig:augmnist++_err_vs_log} shows that RD curves are very similar for when using classification error instead of accuracy.
This is not very surprising as log loss is the standard (differential) proxy of classification error in ML.

\begin{table}[h]
\caption{
Using label-preserving augmentations that remove more information about $\rv X$ decreases the rate without hindering classification performance.
Single run.
}
\small
\center
\begin{tabular}{llrr}
\toprule
&Augmentations& Rate $[\frac{\text{bits}}{\text{img}}]$ & Test Acc. $[\%]$  \\ 
\midrule 
\multirow{2}{*}{\centering  ~VIC } 
 & Small set  &  $185.3$  & $96.2$   \\ 
  & Large set  &  $79.0$  & $97.9$   \\ 
   & Supervised (largest set)  &  $5.7$  & $99.1$   \\ 
\midrule 
\multirow{2}{*}{\centering  ~BINCE  }
 & Small set  &  $256$  & $95.3$   \\ 
  & Large set  &  $131$  & $97.8$   \\ 
   & Supervised (largest set)  &  $5.9$  & $97.8$   \\ 
\bottomrule
\end{tabular}
\label{table:mnist_augmentations}
\end{table}

\paragraphQ{What is the impact of the choice of augmentations}
The rate decreases when $A$ removes more information from $\rv X$.
Indeed, we have $\MI{X}{A(X)} = \MI{X}{M(X)} + \MI{X}{M(X)} = \MI{X}{M(X)} + 0 = \H{M(X)}$, where the second equality comes from \cref{assumption:augmentations} and the last equality from the determinism of $M$.
As a result we can rewrite \cref{thm:rate_invariance_distortion} as $Rate(\delta) = \MI{X}{A(X)} - \delta$.
To illustrate this we trained our VIC and BINCE using three augmentation sets on MNIST, all of which keep the true label invariant but progressively discard more $\rv X$ information:
\begin{inlinelist}
\item standard image augmentations such as random translations, shears, and rescalings;
\item those same standard image augmentations, but drawn from larger ranges of possible translations, shears, scales, etc;
\item supervised ``augmentations'' line in \cite{khosla_supervised_2020} that remove everything except label information, \ie, for every image $x$  let $x^+ = A(x)$ be a random image with the same label.
\end{inlinelist}
\Cref{table:mnist_augmentations} shows that using label-preserving augmentations that remove more information about $X$ greatly decreases the rate without hindering classification performance.
The fact that the supervised augmentations achieve a much better rate, shows that typical SSL compression is still very far from single-task label compression.
SSL compression retains information for at least $2^{79} \approx 10^{23}$ disjoint labels.

\begin{table}[h]
\caption{
End-to-end compression of augmented MNIST works much better than staggered compression for both VIC and BINCE. Single run.
}
\small
\center
\begin{tabular}{llrr}
\toprule
&& Rate $[\frac{\text{bits}}{\text{img}}]$ & Test Acc. $[\%]$  \\ 
\midrule 
\multirow{2}{*}{\centering  ~VIC } 
 & Staggered  &  $477$  & $98.0$   \\ 
  & End-to-end &  $79$  & $97.9$  \\  
\midrule 
\multirow{2}{*}{\centering  ~BINCE  }
 & Staggered &  $358$  &  $97.4$ \\
   & End-to-end & $131$   &  $97.8$  \\ 
\bottomrule
\end{tabular}
\label{table:end2end}
\end{table}

\paragraphQ{How much does end-to-end improve compared to staggered training}
We evaluated end-to-end training for both our losses against a staggered version that consists in first optimizing the distortion and then adding an entropy bottleneck to performing lossy compression of the learned representations (as in \cref{sec:clip_experiments}).
\Cref{table:end2end} shows that end-to-end training can give large gains compared to the staggered method and that our compression gains with CLIP could be even further improved.

%%%%%%%%%%%%%%%%%%%%%%%%%%%%%%%%%%%%%%%%%%%%%%%%%%%%%%%%%%
\subsection{STL10}
\label{appx:stl10}
In the main text we used the STL10 data set to answer some principled questions about our method in controlled experiments.
We provide additional results here.
%What relevance has the choice of distortion metric, entropy bottleneck and prediction method? And what happens when making faulty assumptions?  

\begin{figure}[h]
     \centering
     \begin{subfigure}[h]{0.45\columnwidth}
         \centering
         \includegraphics[width=\textwidth]{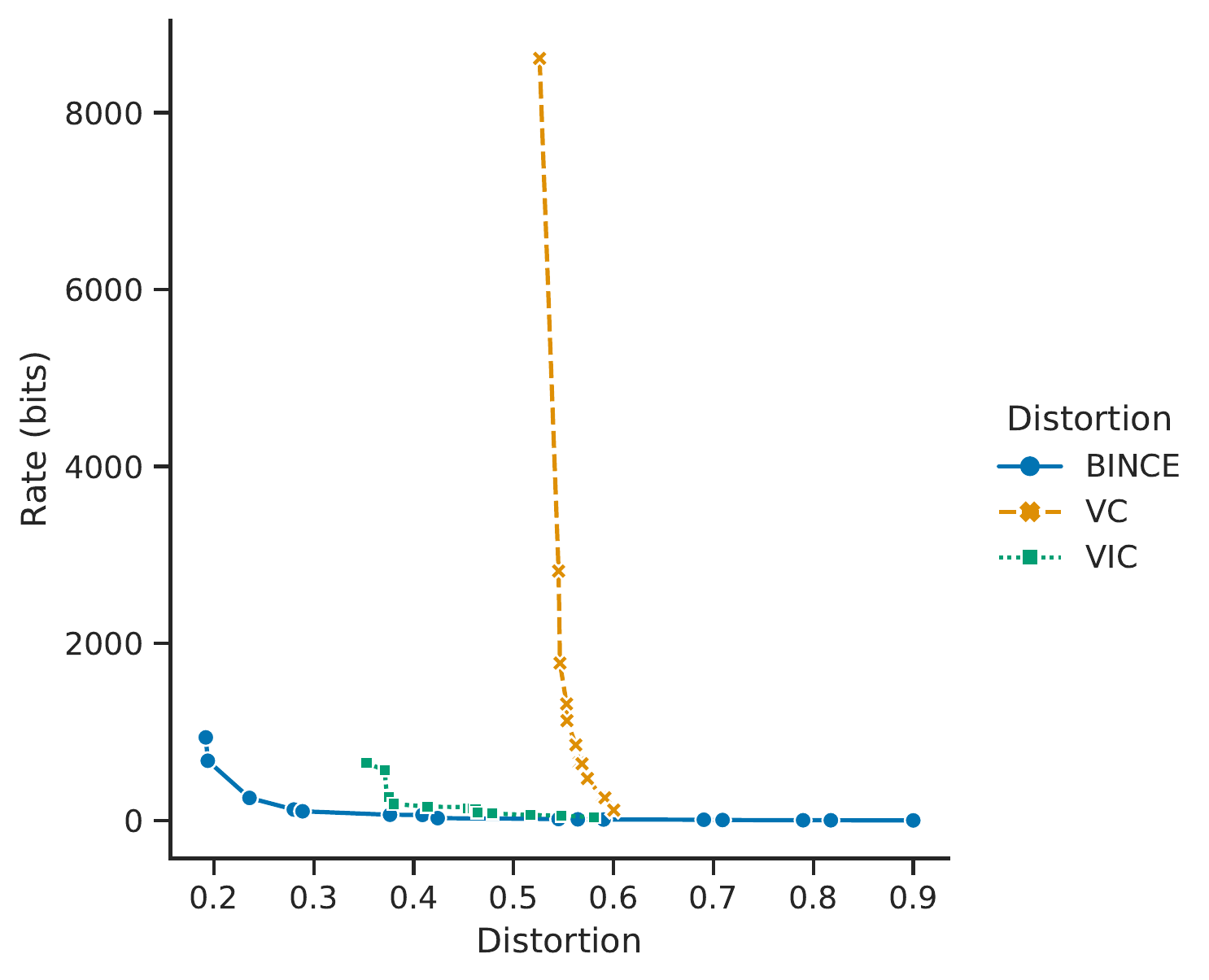}
         \caption{From features $\hrv Z$}
         \label{fig:STL10_dist_Z}
     \end{subfigure}
     \hfill{}
     \begin{subfigure}[h]{0.49\columnwidth}
         \centering
         \includegraphics[width=\textwidth]{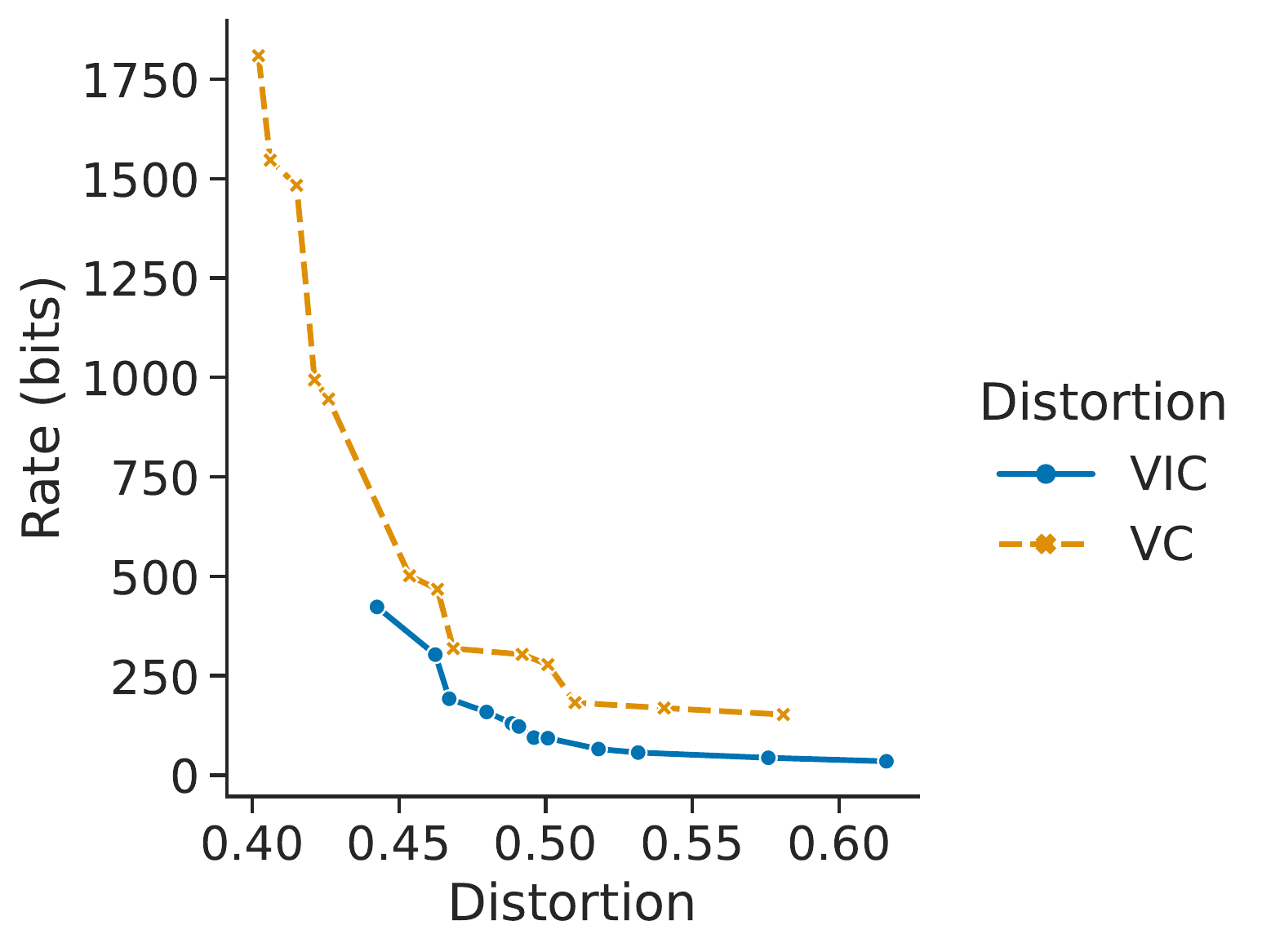}
         \caption{From reconstruction $\hrv X$}
         \label{fig:STL10_dist_X}
     \end{subfigure}
\caption{
BINCE achieves the best RI curves followed by VIC and then VC for STL10 data.
Rate-error curves when predicting downstream tasks from: (left) compressed representations $\hrv Z$, (right) reconstructions $\hrv X$.
}
\label{fig:STL10_dist}
\end{figure}

%\subsubsection{Comparing distortions}
%\label{appx:stl10_distortion}
\ydnote{the fact that the font in \cref{fig:STL10_dist} is different is very ugly. Need to solve for camera ready.}
\paragraphQ{How does the choice of distortion measures or bounds thereof affect RI curves}
Supplementing the results in the main text, we show more extensive results comparing the effect of the distortion measures (invariant or not) or bounds thereof (BINCE and VIC are different bounds on $\Risk{M(X)}{Z}$) on RI curves.
When predicting from compressed representations $\hrv Z$  (\cref{fig:STL10_dist} left), BINCE achieves the best RI curves followed by VIC and VC.
When predicting from reconstructions $\hrv X$ (\cref{fig:STL10_dist} right), VIC still performs a little better than VC although the gap shrinks.
\Cref{table:distortion_variation} shows all quantitative results for best achieved downstream performance (as in \cref{table:distortion_variation} from the main text).

\begin{table}[ht]
\caption{
We compare classical compression formats (PNG, JPEG, webP) to neural (VC) and invariant (BINCE, VIC) ones. 
%Note that, PNG is a lossless compression method, thus predicting from PNG reconstructions is equivalent to regular classification. 
BINCE achieves the same error rate but compresses $ 121.1 \times$ better.
}
\center
\small
\begin{tabular}{llrrr}
\toprule
Distortion  & Predict from & Best error $[\%]$  & Rate [Mb/img] & Compression factor\\ 
\midrule 
PNG \cite{graphics_png_isoiec_2003}    &  reconstructions   &  19.2 & 14.20 & $ 1\times$ \\
JPEG  \cite{group_jpeg_itu-t_1992} &  reconstructions   &  19.9 &  4.60 & $ 3.0\times$ \\
WebP   \cite{webp_google_2018}&  reconstructions   &  20.3 &   1.12 & $ 12.7\times$ \\
VC     &  reconstructions   &  40.2 &   0.23 & $ 62.8\times$ \\
VC     &  features          &  52.6 &   1.08 & $ 13.2\times$ \\
\midrule 
VIC (ours)   &  reconstructions  &  44.3 &   0.05 & $ 268.6\times$ \\
VIC  (ours)  &  features 	     &  35.3 &   0.08 & $ 174.8\times$ \\
BINCE (ours) &  features         &  19.2 &   0.12 & $ 121.1\times$ \\	
\bottomrule
\label{table:distortion_variation_long}
\end{tabular}
\end{table}

Note that VIC and VC achieve much worst downstream performance than BINCE.
Based on preliminary results, we believe that this comes from the fact that, for consistency, in all experiments we used ResNet18 encoders.
Indeed, ResNet18 have an global averaging pooling layer that averages the ``latent image'' over spatial dimensions (width and height).
As a result, the representations $\hrv Z$ does not retain any spatial information, which is often useful for improving reconstructions.
Preliminary results showed that removing this pooling layer improves downstream predictions significantly.
Importantly, this impacts both VIC and VC so although the absolute performance improved by removing this layer the relative error did not seem to.

 For all our models the we estimated (using a naive sample estimate) the mutual information $\MI{Z}{Y}$ and found that $\Hat{\mathrm{I}}\br{Z,X} = \hat{\mathrm{H}}\br{Y}$.  This shows that all information about labels is retained, i.e., no images get compressed to the same Z but have different labels. The difference in test accuracy (which is also similar for training) must thus come from the fact that some information is easier to use/decode from \cite{dubois_learning_2020,xu_theory_2020}.
Indeed, our framework only concerns information retention rather than information usability. 
This is further supported by the fact that BINCE gets very good downstream accuracy, because it has been shown in theory \cite{saunshi_theoretical_2019,tosh_contrastive_2021,lee_predicting_2020} and in practice \cite{chen_simple_2020,oord_representation_2019} that contrastive representations are approximately linearly decodable.

\begin{figure}[!htb]
    \centering
    \begin{minipage}{.44\textwidth}
        \centering
        \includegraphics[width=\textwidth]{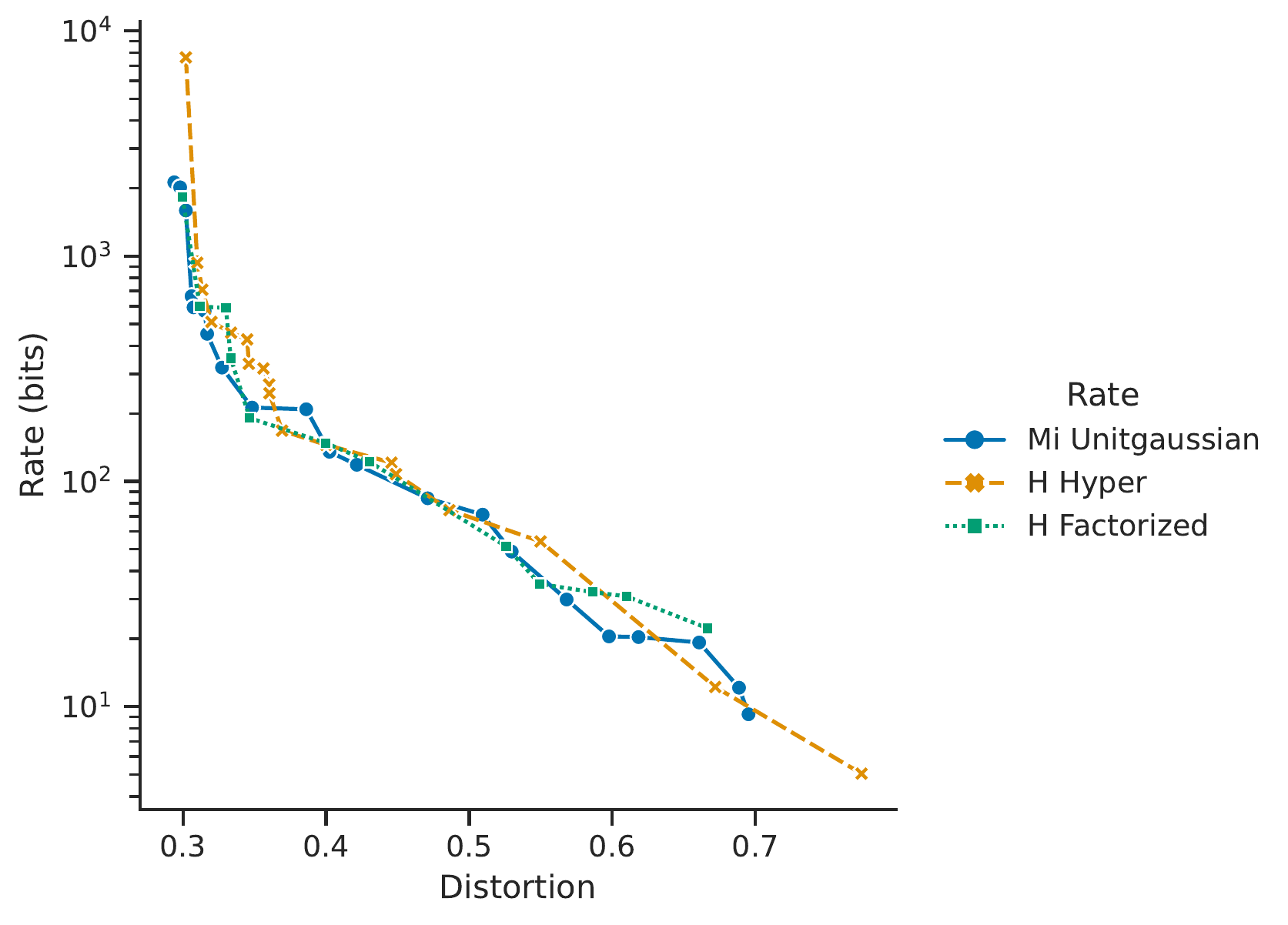}
        \caption{The choice of variational bounds on the rate term $\MI{Z}{X}$ has little effect on RI curves for STL10 data.
        ``MI unitgaussian'' is the upper bound on mutual information used in VIB and VAE;
        ``H factorized'' is \citepos{balle_variational_2018} upper bound on $\H{Z}$ with a factorized entropy model;
        ``H hyper'' is \citepos{balle_variational_2018} upper bound on $\H{Z}$ with a hyperprior entropy model.}
        \label{fig:rate_rd}
    \end{minipage}%
    \hspace{0.1\textwidth}
    \begin{minipage}{0.44\textwidth}
        \centering 
        \includegraphics[width=\textwidth]{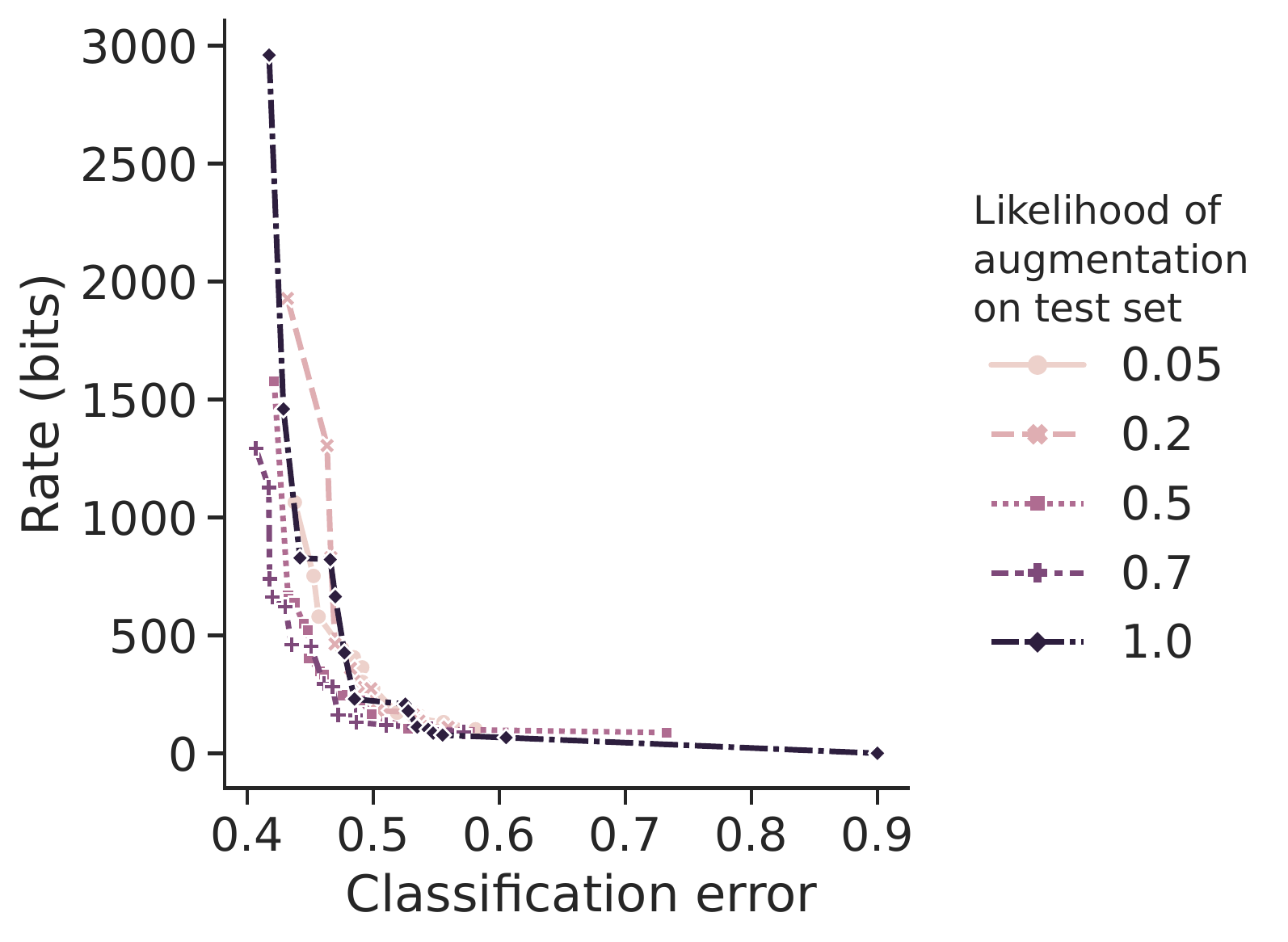}
        \caption{VIC is robust to distribution shifts in the augmentations as it is invariant to the augmentations.
        Specifically, test time shifts in augmentation probability seem to have little effect on the rate-distortion curve for the case of STL10 data.}
        \label{fig:action_dist_rd}
    \end{minipage}
\end{figure}

\begin{table}[h]
\caption{Hierarchical hyperprior works worst (higher rates) for low distortion, when compared to a factorized prior and a mutual information bottleneck on STL10 data.
}
\begin{center}
\small
\begin{tabular}{llrr}
\toprule
Bottleneck & Entropy model & Lossless Rate [bits/img] & Lossless loss\\ 
\midrule 
Entropy $\H{Z}$      &Factorized prior \cite{balle_variational_2018} 		               & 598.5 	& 0.3117 \\
Entropy $\H{Z}$       &Hierarchical prior \cite{balle_variational_2018}  & 934.7  & 0.3100 \\ 
Mutual Information $\MI{Z}{X}$      & Unit Gaussian 	                                                   & 592.5  & 0.3074 \\
\bottomrule
\end{tabular}
\end{center}
\label{table:rate_variation_aurd}
\end{table}

\paragraphQ{How does the choice of bounds on the rate term $\MI{Z}{X}$ impact RI curves}
For the main paper we always used the standard \cite{balle_end--end_2017} neural compressor's upper bound on $\MI{Z}{X}$, namely, the entropy bound $\MI{Z}{X} \leq \H{Z} \leq \E{p(Z,X)}{q_{\theta}(Z)}$.
To understand the effect of using other bounds on $\MI{Z}{X}$ and different entropy models $q_{\theta}(Z)$.
Specifically, in \cref{fig:rate_rd} we compare three different bounds on mutual information:
(MI unitgaussian) the mutual information bound from VAE and VIB $\MI{Z}{X} \leq \KL{p_{\varphi}(Z|X)}{q_{\theta}(Z)}$;
(H Factorized) the entropy bound $\H{Z} \leq \E{p(Z,X)}{q_{\theta}(Z)}$ where $q_{\theta}(Z)$ is \citepos{balle_variational_2018} factorized entropy model;
(H Hyper) the entropy bound where $q_{\theta}(Z)$ is \citepos{balle_variational_2018} hyperprior entropy model.
\ydnote{\karen{} we need to change the legend in \cref{fig:rate_rd} to be more understandable / presentable. I would use latex and the following ``$\MI{Z}{X}$'', ``$\H{Z}$ factorized'', ``$\H{Z}$ hyper'' }
In our experiments, however, we find that neither of these choices influence the RD curves at typical distortion levels as seen \cref{fig:rate_rd}. 
In our experiments we use ``H Hyper'', which does seem to enable very low rates (high distortions) but seems to perform worst at very high rates (see \cref{table:rate_variation_aurd})

%\subsubsection{Approximate invariance}\label{appx:stl10_approximate}
%\input{figures/STL10/STL10_approx_invariance}
\kunote{Todo: correctly implement and run this experiment again}

\paragraphQ{How important is the distribution over augmentations}
%Up to now we have used the same augmentations at train and test time to ensure that our invariance assumption holds.
As discussed in the main text, RD curves of our VIC show negligible difference when the distribution of augmentation shifts from training to test time.
We provide this evidence in \cref{fig:action_dist_rd}. Here we trained a VIC compressor on data with various augmentations, if these were (jointly) applied or not would be decided by a fair coin flip. At test time, we changed the coin to be biased with $p=0.05,0.2,0.5,0.7,1.0$. 
%As our VIC model should learn an invariant representation so the effect of augmentation distributions shifts at test time should not effect our learned representations.
 
\ydimp{markpage}

%%%%%%%%%%%%%%%%%%%%%%%%%%%%%%%%%%%%%%%%%%%%%%%%%%%%%%%%%%
\subsection{Pretrained CLIP}
\label{appx:clip}

\begin{table}[h]
\caption{
Converting a pretrained SSL model into a zero-shot compressor achieves substantial bit-rate gains while allowing test accuracies similar to  predicting from raw images.
CLIP refers to the original CLIP with lossless compression of the representations.
CLIP+EB refers to our CLIP compressor.
CLIP+EB\textsuperscript{$-$} and CLIP+EB\textsuperscript{$+$} are our CLIP compressors trained respectively for a larger and smaller bit-rate.
We provide downstream evaluation using an MLP and a linear (SVM) predictor.
Baselines: JPEG and compression of features from a ImageNet pretrained classifier (Transfer + EB).
}
\scriptsize
\center
\begin{tabular}{lllrrrrrrrrr}
\toprule
 & & & ImageNet  & STL & PCam & Cars & CIFAR10 & CIFAR100 & Food      & Pets & Caltech  \\ 
\midrule 
\multirow{5}{*}{\rotatebox[origin=c]{90}{\centering ~Rate [Bits/img]  }} 
&& JPEG & 1.49e6  & 4.71e4 & 9.60e4 & 1.92e5 & 1.05e4 & 1.05e4 & 1.54e5     & 1.81e5 & 1.69e5  \\ 
 & & Transfer + EB & 3.95e3  & 3.33e3 &3.99e3  & 3.18e3 &  3.92e3  &  & 3.26e3     & 3.70e3 & 3.40e3  \\ 
 && CLIP & 1.52e4  & 1.52e4 & 1.52e4 & 1.52e4 & 1.52e4  & 1.52e4  & 1.52e4     & 1.52e4 & 1.52e4  \\ 
& & \textbf{CLIP+EB}\textsuperscript{$-$} & 2.47e3  & 2.46e3 & 2.61e3 & 2.59e3 & 2.53e3 & 2.54e3 & 2.39e3      & 2.33e3 & 2.46e3  \\ 
 && \textbf{CLIP+EB} & 1.35e3  & 1.34e3 & 1.49e3 & 1.47e3 & 1.41e3 & 1.42e3 & 1.27e3      & 1.21e3 & 1.34e3  \\ 
& & \textbf{CLIP+EB}\textsuperscript{$+$} & 9.63e2  & 9.52e2 & 1.49e3 & 1.52e2 & 1.02e2 & 1.09e3 & 8.89e2      & 8.35e2 & 9.53e2  \\ % to double check imagenet
 \midrule 
\multirow{10}{*}{\rotatebox[origin=c]{90}{\centering ~Test Accuracies $[\%]$ }}
 &\multirow{2}{*}{\rotatebox[origin=c]{90}{\centering ~\cite{radford_learning_2021}  }}
 & JPEG &  76.6   & 99.0 & 82.6 & 49.1 & 96.7 & 86.3 & 81.8       & 90.4 & 94.5   \\ 
 && CLIP \cite{radford_learning_2021} & 76.1  & 98.3 & 83.9 & 81.8 & 95.1 & 80.5 & 88.8      & 90.0 & 93.0  \\ 
  \cmidrule{2-12}
 &\multirow{4}{*}{\rotatebox[origin=c]{90}{\centering ~MLP }}
 & Transfer + EB  & 72.7 & 96.1 & 79.4 & 42.0 &  87.0 & & 66.8    & 91.3 &  89.9  \\ 
 && CLIP  & 76.5  & 98.6 & 84.5 & 80.8 & 95.3 & 80.9 & 88.5      & 89.7 &  93.2  \\  
 && \textbf{CLIP+EB}\textsuperscript{$-$} & 76.6  & 98.7 & 82.7 & 80.4 & 95.3 & 80.9 & 88.5      & 89.6 & 93.5  \\  
 && \textbf{CLIP+EB }& 76.3  & 98.7 & 80.9 & 79.6 & 95.2 & 80.1 & 88.3    & 89.5 & 93.4  \\
 && \textbf{CLIP+EB}\textsuperscript{$+$} & 76.0  & 98.7 & 80.1 & 78.9 & 94.8 & 78.6 & 87.6      & 88.6 & 92.9  \\ % to double check imagenet
 \cmidrule{2-12}
  &\multirow{4}{*}{\rotatebox[origin=c]{90}{\centering Linear  }}
 & CLIP  &   & 98.6 & 83.8 & 80.8 & 95.0 & 79.8 & 85.0      & 89.3 & 93.8  \\ 
  && \textbf{CLIP+EB}\textsuperscript{$-$} &     & 98.7 & 83.2 & 80.8 & 95.0 & 79.7 & 85.0      & 89.2 & 93.6  \\  
 && \textbf{CLIP+EB} &   & 98.7 & 81.1 & 79.9 & 94.8 & 79.0 & 83.6      & 88.3 & 93.7  \\
 && \textbf{CLIP+EB}\textsuperscript{$+$} &   & 98.6 & 80.5 & 78.9 & 94.4 & 80.5 & 82.5      & 87.8 & 93.5  \\
\bottomrule
\end{tabular}
\label{table:clip_all}
\end{table}

In \cref{table:clip_all} we provide all the quantitative results for our zero-shot CLIP experiments from which we derived the tables in the main text.

We note that zero-shot compression can still be analysed using our framework. 
Indeed CLIP was trained on $400M$ sampled from a r.v. $\rv X$ over images on the internet.
As these datasets are on internet, they are samples from the joint $(\rv X, \rv Y)$ for a specific task $\rv Y$.
One can see this as a multi-task setting (each dataset is a distinct task).

\paragraphQ{What is the effect of using a more powerful predictor from the representation}
In our framework we only discuss about information but never whether this information can easily be decoded by the predictors of interest. 
We investigated the effect of using more powerful predictors from our representation to understand how easy it is to decode the information in our representation.
In particular, we evaluated all our CLIP compressors (\ie at different $\beta$), by considering predictions from our compressed representation using a two layer MLP and using a linear classifier (SVM).
\cref{table:clip_all} shows that the advantage of using an MLP compared to a linear model is small, which suggests that our CLIP compressed representation store information in a way that is easily decodable.
This is typical from contrastive self-supervised models \cite{oord_representation_2019,chen_simple_2020}.

\paragraphQ{How does compressing SSL compare to compressing features form transfer learning}
%Our BINCE objective shows that SSL is tightly related to task-centric compression, in that it ensures that all the necessary information is retained.
In previous work, \citet{singh_end--end_2020} had considered the case of compressing features from single-task transfer learning instead of self-supervised method.
In \cref{table:clip_all} we compare compression of both type of features. Specifically ``Transfer + EB'' shows compression of a pretrained ResNet50 on ImageNet.
We see that It generally performs worst than our CLIP compressor both in terms of test accuracies and bit-rate.
One issue with this comparison is that the architectures of both models are not the same is likely that ``transfer+EB'' does not even perform better than CLIP on ImageNet.

%%%%%%%%%%%%%%%%%%%%%%%%%%%%%%%%%%%%%%%%%%%%%%%%%%%%%%%%%%
\subsection{Galaxy Zoo}
\label{appx:galaxy}

\begin{table}[ht]
\small
\caption{Comparisons between pretrained CLIP BINCE, a BINCE trained end-to-end, and SOTA perceptual compressors on Galaxyzoo data. 
CLIP BINCE achieves the smallest bit-rate.
}
\begin{tabular}{llllllll}
\toprule
\textbf{Compressor}           & \textbf{rate} & \textbf{test loss} & \textbf{val. loss mean} & \textbf{median}  &  \textbf{max} & \textbf{min}                   &  \textbf{std }                                                 \\
   &  $[\text{Mb/img}]$    & $[\text{ }]$ & $[\text{ }]$ &  $[10^{-3}]$ &  $[ \text{ }]$&  $[10^{-7}]$                  &  $[10^{-2}]$                                                            \\
\midrule
PNG              & 53.73 & 0.007   & 0.008       & 0.86        & 0.07                     & 1.04                     & 1.62   \\      
 JPEG              & 1.68 & 0.012   & 0.013        & 1.25         & 0.11                     & 1.23                     & 2.61    \\ 
 WebP              & 0.48  & 0.010  & 0.011        & 1.20         & 0.10                   & 1.16                     & 2.29       \\
\midrule
BINCE (CLIP) & 0.33  & 0.011    & 0.011        & 1.13         & 0.10                     & 7.59                     & 2.29   \\
BINCE (end to end)        & 1.77 & 0.012   & 0.012        & 1.43         & 0.11                     & 1.31                     & 2.50 \\
\bottomrule
\end{tabular}
\label{table:galaxy}
\end{table}

Humanity observes earth and sky at high temporal and spatial resolution, this can easily fill entire data centers. What is more, multiple copies of these series often exist over the world. 
At recording time it is usually not clear what kind of queries need to be answered about the recordings in the future; What was the weather like 10 years ago? Did a glacier resolve here? ect.  
To investigate our method in such real world scenario we compressed the GalaxyZoo telescope dataset (GZ2) and its 37 classification tasks.
In \cref{table:galaxy}, we compare a classical lossless and lossy method, to our BINCE at same distortions.

\paragraphQ{How well does CLIP pretraining work compared to in domain training}
In the main paper we have seen that CLIP pretraining gives rise to a compressor that generalizes very well across different datasets.
We evaluated how well the CLIP compressor generalizes compare to training BINCE directly end-to-end on GZ2's training set.
\Cref{table:galaxy} shows that the CLIP compressor works much better ($4\times$ rate gains) than the end-to-end BINCE.
This suggests that pretraining can really be beneficial for training invariant compressors, and that our CLIP compressor can generalize very well across datasets.
\ydnote{We should compare BINCE end to end resnet50 with CLIP resnet 50 instead of the visual transformer CLIP. i.e. use the same encoder. }

\paragraphQ{How does our CLIP compressor generalize to very different images compared to SOTA compressors}
In the main text we have seen that our CLIP compressor generalizes very well across different datasets compared to high quality JPEG.
To better understand the limits of the generalization capacity of our CLIP compressor, we compared it to a SOTA classical compressor (WebP) on images that are completely different than the ones CLIP was trained on, namely Galaxy images (not typical images on internet).
We see in \cref{table:galaxy} that in this challenging setting our CLIP compressor only achieves relatively small gains ($30\%$).

\kunote{Make this great! in da rebuttal :P}

\end{document}